\pdfoutput=1 
\documentclass[preprint,12pt]{imsart}
\usepackage{graphicx}
\usepackage{amssymb,amsmath,bm}
\usepackage{amsthm}
\usepackage{subfigure}
\usepackage{enumerate}

\usepackage{natbib}
\def\cite{\citet}

\usepackage[dvips]{geometry}
\startlocaldefs
\newlength{\figwidth}
\newlength{\figwidthsmall}
\newcommand{\myfig}[1]{\includegraphics[width=\figwidth,clip]{#1}} %% with figures
\newcommand{\myfigc}[1]{\includegraphics[width=\figwidthsmall,clip]{#1}} %% with figures
\newcommand{\myfigbig}[1]{\includegraphics[width=2\figwidth,clip]{#1}} %% with figures

\newtheorem{theo}{Theorem}
\newtheorem{lemm}{Lemma}

\newcommand{\trace}{\mathop{\rm tr}\nolimits}
\newcommand{\diag}{\mathop{\rm diag}\nolimits}
\newcommand{\Diag}{\mathop{\rm Diag}\nolimits}
\newcommand{\bmh}[1]{\bm{\hat {#1}}}
\newcommand{\bmb}[1]{\bm{\bar {#1}}}
\newcommand{\bmt}[1]{\bm{\tilde {#1}}}

\newcommand{\inv}{^{-1}}
\newcommand{\oh}{^{-1/2}}
\newcommand{\ohh}{^{-1}}
\newcommand{\ohhh}{^{-3/2}}
\newcommand{\ohhhh}{^{-2}}
\newcommand{\ga}{\gamma}
\newcommand{\hlambda}{\hat \lambda}
\newcommand{\blambda}{\bar \lambda}
\newcommand{\stilde}{\tilde s}
\endlocaldefs

\begin{document}

%%% title and author
\begin{frontmatter}
\title{Cross-validation of matching correlation analysis
by resampling matching weights}
\runtitle{cross-validation of matching correlation analysis}
\author{\fnms{Hidetoshi}
\snm{Shimodaira}\corref{}\ead[label=e1]{shimo@sigmath.es.osaka-u.ac.jp}\thanksref{t1}}
\thankstext{t1}{
Supported in part by JSPS KAKENHI Grant (24300106, 26120523).}
\address{
Division of Mathematical Science\\Graduate School of
 Engineering Science\\
Osaka University\\
 1-3 Machikaneyama-cho\\Toyonaka, Osaka, Japan\\
 \printead{e1}\\
 \today}
%\affiliation{Osaka University}

\runauthor{H.~SHIMODAIRA}

\begin{abstract}

The strength of association between a pair of data vectors is represented by a nonnegative real
number, called matching weight. For dimensionality reduction, we consider a linear transformation of
data vectors, and define a \emph{matching error} as the weighted sum of squared distances between
transformed vectors with respect to the matching weights. Given data vectors and matching weights,
the optimal linear transformation minimizing the matching error is solved by the spectral graph
embedding of \cite{yan2007graph}. This method is a generalization of the canonical correlation
analysis, and will be called as matching correlation analysis (MCA). In this paper, we consider  a
novel sampling scheme where the observed matching weights are randomly sampled from underlying true
matching weights with small probability, whereas the data vectors are treated as constants. We then
investigate a cross-validation by resampling the matching weights. Our asymptotic theory shows that
the cross-validation, if rescaled properly, computes an unbiased estimate of the matching error with
respect to the true matching weights. Existing ideas of cross-validation for resampling data
vectors, instead of resampling matching weights, are not applicable here. MCA can be used for data
vectors from multiple domains with different dimensions via an embarrassingly simple idea of coding
the data vectors. This method will be called as cross-domain matching correlation analysis (CDMCA),
and an interesting connection to the classical associative memory model of neural networks is also
discussed.

\end{abstract}

\begin{keyword}
 \kwd{asymptotic theory}
 \kwd{cross-validation}
 \kwd{resampling}
 \kwd{matching weight}
 \kwd{multiple domains}
 \kwd{cross-domain matching}
 \kwd{canonical correlation analysis}
 \kwd{multivariate analysis}
 \kwd{spectral graph embedding}
 \kwd{associative memory}
 \kwd{sparse coding}
\end{keyword}

\end{frontmatter}

\section{Introduction} \label{sec:intro}

We have $N$ data vectors of $P$ dimensions. Let $\bm{x}_1,\ldots,\bm{x}_N \in \mathbb{R}^P$ be the
data vectors, and $\bm{X}=(\bm{x}_1,\ldots,\bm{x}_N)^T \in \mathbb{R}^{N\times P}$ be the data
matrix. We also have \emph{matching weights} between the data vectors. Let  $w_{ij}=w_{ji}\ge0$,
$i,j=1,\ldots,N$, be the matching weights, and  $\bm{W}=(w_{ij}) \in \mathbb{R}^{N \times N}$ be the
matching weight matrix. The matching weight $w_{ij}$ represents the strength of association between
$\bm{x}_i$ and $\bm{x}_j$.  For dimensionality reduction, we will consider a linear transformation
from $\mathbb{R}^P$ to $\mathbb{R}^K$ for some $K\le P$ as \[ \bm{y}_i = \bm{A}^T \bm{x}_i,\quad
i=1,\ldots,N, \]  or $\bm{Y}=\bm{X} \bm{A}$, where $\bm{A} \in \mathbb{R}^{P \times K}$ is the
linear transformation matrix, $\bm{y}_1,\ldots,\bm{y}_N \in \mathbb{R}^K$ are the transformed
vectors, and $\bm{Y}=(\bm{y}_1,\ldots,\bm{y}_N )^T \in \mathbb{R}^{N\times K} $ is the transformed
matrix. Observing $\bm{X}$ and $\bm{W}$, we would like to find $\bm{A}$ that minimizes the
\emph{matching error}   \[  \phi = \frac{1}{2} \sum_{i=1}^N \sum_{j=1}^N w_{ij} \|\bm{y}_i -
\bm{y}_j \|^2 \] under some constraints. We expect that the distance between $\bm{y}_i$ and
$\bm{y}_j$ will be small when $w_{ij}$ is large, so that the locations of transformed vectors
represent both the locations of the data vectors and the associations between data vectors.   The
optimization problem for finding $\bm{A}$ is solved by the spectral graph embedding for
dimensionality reduction of \cite{yan2007graph}. Similarly to principal component analysis (PCA),
the optimal solution is obtained as the eigenvectors of the largest $K$ eigenvalues of some matrix
computed from $\bm{X}$ and $\bm{W}$. In Section~\ref{sec:mca}, this method will be formulated by
specifying the constraints on the transformed vectors and also regularization terms for numerical
stability.  We will call the method as \emph{matching correlation analysis} (MCA), since it is a
generalization of the classical canonical correlation analysis (CCA) of
\cite{hotelling1936relations}.  The matching error will be represented by  \emph{matching
correlations} of transformed vectors, which correspond to the canonical correlations of CCA.

MCA will be called as \emph{cross-domain matching correlation analysis} (CDMCA) when we have data
vectors from multiple domains with different sample sizes and different dimensions. Let $D$ be the
number of domains, and $d=1,\ldots,D$ denote each domain. For example, domain $d=1$ may be for image
feature vectors, and domain $d=2$ may be for word vectors computed by word2vec
\citep{mikolov2013distributed} from texts, where the matching weights between the two domains may
represent tags of images in a large dataset, such as Flickr. From domain $d$, we get data vectors
$\bm{x}^{(d)}_i \in \mathbb{R}^{p_d}$, $i=1,\ldots,n_d$, where $n_d$ is the number of data vectors,
and $p_d$ is the dimension of the data vector. Typically, $p_d$ is hundreds, and $n_d$ is thousands
to millions. We would like to retrieve relevant words from an image query, and alternatively
retrieve images from a word query. Given matching weights across/within domains, we attempt to find
linear transformations of data vectors from multiple domains to a ``common space'' of lower
dimensionality so that the distances between transformed vectors well represent the matching
weights. This problem is solved by an embarrassingly simple idea of coding the data vectors, which
is similar to that of \cite{daume2007frustratingly}. Each data vector from domain $d$ is represented
by an augmented data vector $\bm{x}_i$ of dimension $P=\sum_{d=1}^D p_d$, where only $p_d$
dimensions are for the original data vector and the rest of $P-p_d$ dimensions are padded by zeros.
In the case of $D=2$ with $p_1=2$, $p_2=3$, say, a data vector $(1,2)^T$ of domain 1 is represented
by $(1,2,0,0,0)^T$, and $(3,4,5)^T$ of domain 2 is represented by $(0,0,3,4,5)^T$. The number of
total augmented data vectors is $N=\sum_{d=1}^D n_d$. Note that the above mentioned ``embarrassingly
simple coding'' is not actually implemented by padding zeros in computer software; only the nonzero
elements are stored in memory, and CDMCA is in fact implemented very efficiently for sparse
$\bm{W}$. CDMCA is illustrated in  a numerical example of Section~\ref{sec:example}. CDMCA is
further explained in Appendix~\ref{sec:coding}, and an interesting connection to the classical
associative memory model of neural networks \citep{kohonen1972correlation,nakano1972associatron} is
also discussed in Appendix~\ref{sec:associative-memory}.

CDMCA is solved by applying the single-domain version of MCA described in Section~\ref{sec:mca} to
the augmented data vectors, and thus we only discuss the single-domain version in this paper. This
formulation of CDMCA includes a wide class of problems of multivariate analysis, and similar
approaches are very popular recently in pattern recognition and vision \citep{correa2010multi,
yuan2011novel, kan2012multi,shi2013transfer, wang2013learning,gong2014multi,yuan2014graph}.
CDMCA is equivalent to the method of \cite{nori2012multinomial} for multinomial relation prediction
if the matching weights are defined by cross-products of the binary matrices representing relations
between objects and instances.
CDMCA is also found in \cite{huang2013cross} for the case of $D=2$. CDMCA reduces to the multi-set
canonical correlation analysis (MCCA) \citep{kettenring1971canonical, takane2008regularized,
tenenhaus2011regularized} when $n_1=\cdots=n_D$ with cross-domain matching weight matrices being
proportional to the identity matrix. It becomes the classical CCA by further letting $D=2$, or it
becomes PCA  by letting $p_1=p_2=\cdots=p_D=1$.

In this paper, we discuss a cross-validation method for computing the matching error of MCA. In
Section~\ref{sec:matchingerrors}, we will define two types of matching errors, i.e., fitting error
and true error, and introduce cross-validation (cv) error for estimating the true error. In order to
argue distributional properties of MCA, we consider the following sampling scheme. First, the data
vectors are treated as constants. Similarly to the explanatory variables in regression analysis, we
perform conditional inference given data matrix $\bm{X}$, although we do not avoid assuming that
$\bm{x}_i$'s are sampled from some probability distribution. Second, the matching weights  $w_{ij}$
are randomly sampled from underlying true matching weights $\bar w_{ij}$ with small probability
$\epsilon>0$. The value of $\epsilon$ is unknown and it should not be used in our inference. Let
$z_{ij}=z_{ji}\in \{0,1\}$,  $i,j=1,\ldots,N$, be samples from Bernoulli trial with success
probability $\epsilon$, where the number of independent elements is $N(N+1)/2$ due to the symmetry.
Then the observed matching weights are defined as 
\begin{equation} \label{eq:zij} 
  w_{ij} = z_{ij} \bar w_{ij},\quad P(z_{ij}=1)=\epsilon. 
\end{equation}
The true matching weight matrix
$\bmb{W}=(\bar w_{ij})\in \mathbb{R}^{N\times N}$ is treated as an unknown constant matrix with
elements $\bar w_{ij}=\bar w_{ji} \ge 0$. This setting will be appropriate for a large-scale data,
such as those obtained automatically from the web,  where only a small portion $\bm{W}$ of the true
association $\bmb{W}$ may be obtained as our knowledge.  

In Section~\ref{sec:cverror}, we will consider a resampling scheme corresponding to (\ref{eq:zij}).
For the cross-validation, we resample $\bm{W}^*$ from $\bm{W}$ with small probability $\kappa>0$,
whereas $\bm{X}$ is left untouched. Our sampling/resampling scheme is very unique in the sense that
the source of randomness is $\bm{W}$ instead of $\bm{X}$, and existing results of cross-validation
for resampling from $\bm{X}$ such as \cite{stone1977asymptotic} and \cite{golub1979generalized} are
not applicable here. The traditional method of resampling data vectors is discussed in
Section~\ref{sec:node-sampling}.

The true error is defined with respect to the unknown $\bmb{W}$, and the fitting error is defined
with respect to the observed $\bm{W}$.  We would like to look at the true error for finding
appropriate values of the regularization terms (regularization parameters are generally denoted as
$\ga$ throughout) and the dimension $K$ of the transformed vectors. However, the true error is
unavailable, and the fitting error is biased for estimating the true error. The main thrust of this
paper is to show asymptotically that the cv error, if rescaled properly, is an unbiased estimator of
the true error. The value of $\epsilon$ is unnecessary for computing the cv error, but $\bm{W}$
should be a sparse matrix. The unbiasedness of the cv error is illustrated by a simulation study in
Section~\ref{sec:simulation}, and it is shown theoretically by the asymptotic theory of $N\to\infty$
in Section~\ref{sec:asymptotic}.

\section{Illustrative example} \label{sec:example}

Let us see an example of CDMCA applied to the MNIST database of handwritten digits (see
Appendix~\ref{sec:mnist} for the experimental details). The number of domains is $D=3$ with the
number of vectors $n_1=60,000$, $n_2=10$, $n_3=3$, and dimensions $p_1=2784$, $p_2=100$, $p_3=50$.
The handwritten digit images are stored in domain $d=1$, while domain $d=2$ is for the digit labels
``zero'', ``one'', ... , ``nine'', and domain $d=3$ is for attribute labels ``even'', ``odd'',
``prime''. This CDMCA is also interpreted as MCA with $N=60,013$ and $P=2934$.

The elements of $\bmb{W}$ are simply the indicator variables (called dummy variables in statistics)
of image labels. Instead of working on $\bmb{W}$, here we made $\bm{W}$ by sampling 20\% of the
elements from $\bmb{W}$ for illustrating how CDMCA works. The optimal $\bm{A}$ is computed from
$\bm{W}$ using the method described in Section~\ref{sec:regularization} with regularization
parameter $\gamma_M=0.1$. The data matrix $\bm{X}$ is centered, and the transformed matrix $\bm{Y}$
is rescaled. The first and second elements of $\bm{y}_i$, namely, $(y_{i1}, y_{i2})$, $i=1,\ldots,
N$, are shown in  Fig.~\ref{fig:mnist-map12}. For the computation
of $\bm{A}$, we do not have to specify the value of $K$ in advance. Similar to PCA, we first solve
the optimal $\bm{A} = (\bm{a}^1, \ldots, \bm{a}^P) \in \mathbb{R}^{P\times P}$ for the case of
$K=P$, then take the first $K$ columns to get the optimal $\bm{A} = (\bm{a}^1, \ldots, \bm{a}^K) \in
\mathbb{R}^{P\times K}$ for any $K \le P$. We observe that images and labels are placed in the
common space so that they represent both $\bm{X}$ and $\bm{W}$. Given a digit image, we may find the
nearest digit label or attribute label to tell what the image represents.

The optimal $\bm{A}$ of $K=9$ is then computed for several $\gamma_M$ values. For each $\bm{A}$, the
10000 images of test dataset are projected to the common space and the digit labels and attribute
labels are predicted.  We observe in Fig.~\ref{fig:mnist-classification-error} that the
classification errors become small when the regularization parameter is around $\gamma_M=0.1$. Since
$\bm{x}_i$ does not contribute to $\bm{A}$ if $\sum_{j=1}^N w_{ij}=0$, these error rates are
computed using only 20\% of $\bm{X}$;  they improve to 0.0359 ($d=2$) and 0.0218 ($d=3$) if
$\bmb{W}$ is used for the computation of $\bm{A}$ with $K=11$ and $\gamma_M=0$.

It is important to choose an appropriate value of $\gamma_M$ for minimizing the classification
error. We observe in Fig.~\ref{fig:mnist-matching-error} that the curve of the true matching error
of the test dataset is similar to the curves of the classification errors. However, the fitting
error wrongly suggests that a smaller $\gamma_M$ value would be better. Here, the fitting error is
the matching error computed from the training dataset, and it underestimates the true matching
error. On the other hand, the matching error computed by the cross-validation method
of Section~\ref{sec:cverror} correctly suggests that $\gamma_M=0.1$ is a good choice.

\section{Matching correlation analysis} \label{sec:mca}

\subsection{Matching error and matching correlation} \label{sec:matching}

Let $\bm{M} \in \mathbb{R}^{N\times N}$ be the diagonal matrix of row (column) sums of $\bm{W}$. \[
\bm{M}=\diag(m_1,\ldots,m_N),\quad m_i = \sum_{j=1}^N w_{ij}. \] 
This is also expressed as $\bm{M}=\diag(\bm{W}\bm{1}_N)\in\mathbb{R}^{N\times N}$ using $\bm{1}_N\in
\mathbb{R}^N$, the vector with all elements one. $\bm{M}-\bm{W}$ is sometimes called as the graph
Laplacian. This notation will be applied to other weight matrices, say,  $\bmb{M}$ for $\bmb{W}$.
Key notations are shown in Table~\ref{tab:notations}.

Column vectors of matrices will be denoted by superscripts. For example, the $k$-th component of
$\bm{Y}=(y_{ik};\, i=1,\ldots,N,\, k=1,\ldots,K)$ is  $\bm{y}^k=(y_{1k},\ldots,y_{Nk})^T \in
\mathbb{R}^N$ for $k=1,\ldots,K$, and we write $\bm{Y}=(\bm{y}^1,\ldots,\bm{y}^K)$. Similarly,
$\bm{X}=(\bm{x}^1,\ldots,\bm{x}^P)$ with $\bm{x}^k\in \mathbb{R}^N$, and
$\bm{A}=(\bm{a}^1,\ldots,\bm{a}^K)$ with $\bm{a}^k\in \mathbb{R}^P$. The linear transformation is
now written as $\bm{y}^k = \bm{X} \bm{a}^k$, $k=1,\ldots,K$.

The matching error of the $k$-th component $\bm{y}^k$ is defined by \[     \phi_k =  \frac{1}{2}
\sum_{i=1}^N \sum_{j=1}^N w_{ij} (y_{ik}- y_{jk})^2,  \] and the matching error of all the
components is $\phi=\sum_{k=1}^K \phi_k$. By noticing $\bm{W}=\bm{W}^T$, the matching error is
rewritten as \[ \begin{split}  \phi_k &= \frac{1}{2} \sum_{i=1}^N  m_i y_{ik}^2 +
\frac{1}{2}\sum_{j=1}^N m_j y_{jk}^2   - \sum_{i=1}^N \sum_{j=1}^N w_{ij} y_{ik}  y_{jk}\\   &=
\bm{y}^{k\,T} (\bm{M} - \bm{W}) \bm{y}^k. \end{split} \] Let us specify constraints on $\bm{Y}$ as
\begin{equation} \label{eq:ymy1} \bm{y}^{k\,T} \bm{M} \bm{y}^k=\sum_{i=1}^N m_i y_{ik}^2 = 1,\quad
k=1,\ldots,K. \end{equation} In other words, the weighted variance of $y_{1k},\ldots,y_{Nk}$ with
respect to the weights $m_1,\ldots, m_N$ is fixed as a constant. Note that we say ``variance'' or
``correlation'' although variables are not centered explicitly throughout. The matching error is now
written as \[ \phi_k = 1 -  \bm{y}^{k\,T} \bm{W} \bm{y}^k. \] We call $\bm{y}^{k\,T} \bm{W}
\bm{y}^k$ as the \emph{matching (auto) correlation} of $\bm{y}^k$.

More generally, the matching error between the $k$-th component $\bm{y}^k$ and $l$-th component
$\bm{y}^l$ for $k,l=1,\ldots,K$, is defined by \[     \frac{1}{2} \sum_{i=1}^N \sum_{j=1}^N w_{ij}
(y_{ik}- y_{jl})^2 = 1 -  \bm{y}^{k\,T} \bm{W} \bm{y}^l, \] and the \emph{matching (cross)
correlation} between $\bm{y}^k$ and $\bm{y}^l$ is defined by $\bm{y}^{k\,T} \bm{W} \bm{y}^l$.   This
is analogous to the weighted correlation $\bm{y}^{k\,T} \bm{M} \bm{y}^l$ with respect to the weights
$m_1,\ldots,m_N$, but a different measure of association between $\bm{y}^k$  and $\bm{y}^l$. It is
easily verified that $|\bm{y}^{k\,T} \bm{W} \bm{y}^l | \le 1$ as well as $|\bm{y}^{k\,T} \bm{M}
\bm{y}^l | \le 1$. The matching errors reduce to zero when the corresponding matching correlations
approach 1.

A matching error not smaller than one, i.e., $\phi_k\ge 1$, may indicate the component $\bm{y}^k$ is
not appropriate for representing $\bm{W}$. In other words, the matching correlation should be
positive: $\bm{y}^{k\,T} \bm{W} \bm{y}^k > 0$. For justifying the argument, let us consider the
elements $y_{ik}$, $i=1, \ldots, N$, are independent random variables with mean zero. Then $E(
\bm{y}^{k\,T} \bm{W} \bm{y}^k ) = \sum_{i=1}^N w_{ii} V(y_{ik}) = 0$ if $w_{ii}=0$. 
Therefore random components, if centered properly, give the matching error $\phi_k \approx 1$.

\subsection{The spectral graph embedding for dimensionality reduction} \label{sec:spectral}

We would like to find the linear transformation matrix $\bmh{A}$ that minimizes $\phi = \sum_{k=1}^K
\phi_k$. Here symbols are denoted with hat like $\bmh{Y}=\bm{X}\bmh{A}$ to make a distinction from
those defined in Section~\ref{sec:regularization}. Define $P\times P$ symmetric matrices \[ \bmh{G}
= \bm{X}^T \bm{M} \bm{X},\quad     \bmh{H} = \bm{X}^T \bm{W} \bm{X}. \] Then, we consider the
optimization problem: \begin{gather} \mbox{Maximize}\quad \trace( \bmh{A}^T \bmh{H} \bmh{A}
)\quad\mbox{with  respect to}\quad\bmh{A}\in\mathbb{R}^{P\times K} \label{eq:ahahat}\\ \mbox{subject
to}\quad \bmh{A}^T \bmh{G} \bmh{A} = \bm{I}_K \label{eq:agahat}. \end{gather} The objective function
$\trace(\bmh{A}^T \bmh{H} \bmh{A})=\sum_{k=1}^K \bmh{y}^{k\,T} \bm{W} \bmh{y}^k$ is the sum of
matching correlations of $\bmh{y}^k$, $k=1,\ldots, K$, and thus (\ref{eq:ahahat}) is equivalent to
the minimization of $\phi$ as we wished. The constraints in (\ref{eq:agahat}) are $\bmh{y}^{k\,T}
\bm{M} \bmh{y}^l = \delta_{kl}$, $k,l=1,\ldots, K$. In addition to (\ref{eq:ymy1}), we assumed that
$\bmh{y}^k$, $k=1,\ldots, K$, are uncorrelated each other to prevent $\bmh{a}^1, \ldots, \bmh{a}^K $
degenerating to the same vector.

The optimization problem mentioned above is the same formulation as the spectral graph embedding for
dimensionality reduction of \cite{yan2007graph}. A difference is that $\bm{W}$ is specified by
external knowledge in our setting, while $\bm{W}$ is often specified from $\bm{X}$ in the
graph embedding literature. Similar optimization problems are found in the spectral graph theory \citep{chung1997spectral}, the
normalized graph Laplacian \citep{von2007tutorial}, or the spectral embedding
\citep{belkin2003laplacian} for the case of $\bm{X}=\bm{I}_N$.

\subsection{Regularization and rescaling} \label{sec:regularization}

We introduce regularization terms $\Delta \bm{G}$ and $\Delta \bm{H}$   for numerical stability.
They are $P\times P$ symmetric matrices, and added to   $\bmh{G}$ and $\bmh{H}$. We will  replace
$\bmh{G}$ and $\bmh{H}$ in the optimization problem with
 \[  \bm{G} =
\bmh{G} + \Delta \bm{G},\quad  \bm{H} = \bmh{H} + \Delta \bm{H}. \] The same regularization terms
are considered in \cite{takane2008regularized} for MCCA. 
 We may write $\Delta \bm{G}=\ga_M \bm{L}_M$ and $\Delta
\bm{H}=\ga_W \bm{L}_W$ with prespecified matrices, say,  $\bm{L}_M=\bm{L}_W=\bm{I}_P$, and attempt
to choose appropriate values of the regularization parameters $\ga_M, \ga_W \in \mathbb{R}$.

We then work on the optimization problem:  \begin{gather} \mbox{Maximize}\quad \trace( \bm{A}^T
\bm{H} \bm{A} )\quad\mbox{with  respect to}\quad\bm{A}\in\mathbb{R}^{P\times K} \label{eq:aha}\\
\mbox{subject to}\quad \bm{A}^T \bm{G} \bm{A} = \bm{I}_K \label{eq:aga}. \end{gather} For the
solution of the optimization problem, we denote $\bm{G}^{1/2}\in \mathbb{R}^{P\times P}$ be one of
the matrices satisfying $(\bm{G}^{1/2})^T \bm{G}^{1/2} = \bm{G}$. The inverse matrix is denoted by
$\bm{G}^{-1/2} = (\bm{G}^{1/2} )\inv$. These are easily computed by, say, Cholesky decomposition or
spectral decomposition of symmetric matrix.  The eigenvalues of $(\bm{G}^{-1/2})^T \bm{H}
\bm{G}^{-1/2}$ are $\lambda_1\ge\lambda_2\ge \cdots  \ge \lambda_P$, and the corresponding
normalized eigenvectors are $\bm{u}_1,\bm{u}_2,\ldots,\bm{u}_P \in \mathbb{R}^P$. The solution of
our optimization problem is \begin{equation} \label{eq:agu}  \bm{A} =
\bm{G}^{-1/2}(\bm{u}_1,\ldots,\bm{u}_K). \end{equation}
The solution (\ref{eq:agu}) can also be characterized by (\ref{eq:aga}) and
\begin{equation}\label{eq:ahalambda}     \bm{A}^T \bm{H} \bm{A}  = \bm{\Lambda}, \end{equation}
where $\bm{\Lambda}=\diag(\lambda_1,\ldots,\lambda_K)$. Obviously, we do not have to solve the
optimization problem several times when changing the value of $K$. We may compute $\bm{a}^k=
\bm{G}^{-1/2} \bm{u}_k$, $k = 1, \ldots, P$, and take the first $K$ vectors to get $\bm{A} =
(\bm{a}^1,\ldots, \bm{a}^K)$ for any $K\le P$. This is the same property of PCA mentioned in Section
14.5 of \cite{hastie2009elements}.

Suppose $\Delta \bm{G}=\Delta \bm{H}=\bm{0}$. Then the problem becomes that of
Section~\ref{sec:spectral}, and (\ref{eq:ymy1}) holds.  From the diagonal part of
(\ref{eq:ahalambda}), we have $\bm{y}^{k\,T} \bm{W} \bm{y}^k = \lambda_k$, $k=1,\ldots,K$, meaning
that the eigenvalues are the matching correlations of $\bm{y}^k$'s. From the off-diagonal parts of
(\ref{eq:aga}) and (\ref{eq:ahalambda}), we also have $\bm{y}^{k\,T} \bm{M} \bm{y}^l = \bm{y}^{k\,T}
\bm{W} \bm{y}^l = 0$ for $k\neq l$, meaning that the weighted correlations and the matching
correlations between the components are all zero. These two types of correlations defined in
Section~\ref{sec:matching} explain the structure of the solution of our optimization problem. Since
the matching correlations should be positive for representing $\bm{W}$, we will confine the
components to those with $\lambda_k > 0$. Let $K^+$ be the number of positive eigenvalues. Then we
will choose $K$ not larger than $K^+$.

In general, $\Delta \bm{G}\neq\bm{0}$, and (\ref{eq:ymy1}) does not hold. We thus rescale each
component as $\bm{y}^k = b_k \bm{X} \bm{a}^k$ with factor $b_k>0$, $k=1,\ldots,K$. We may set
$b_k=(\bm{a}^{k\,T} \bm{X}^T \bm{M} \bm{X} \bm{a}^k)^{-1/2}$ so that (\ref{eq:ymy1}) holds. In the
matrix notation, $\bm{Y} = \bm{X} \bm{A} \bm{B}$ with $\bm{B}=\diag(b_1,\ldots,b_K)$. Another choice
of rescaling factor is to set $b_k=(\bm{a}^{k\,T} \bm{X}^T \bm{X} \bm{a}^k)^{-1/2}$, so that
\begin{equation} \label{eq:yy1}     \bm{y}^{k\,T} \bm{y}^k =\sum_{i=1}^N y_{ik}^2 =1,\quad
k=1,\ldots,K      \end{equation} holds.  In other words, the unweighted variance of
$y_{1k},\ldots,y_{Nk}$ is fixed as a constant. Both rescaling factors defined by (\ref{eq:ymy1}) and
(\ref{eq:yy1}) are considered  in the simulation study of Section~\ref{sec:simulation}, but only
(\ref{eq:ymy1}) is considered for the asymptotic theory of Section~\ref{sec:asymptotic}.

In the numerical computation of Section~\ref{sec:example} and Section~\ref{sec:simulation}, the data
matrix $\bm{X}$ is centered as $\sum_{i=1}^N m_i \bm{x}_i = \bm{0}$ for
(\ref{eq:ymy1}) and $\sum_{i=1}^N \bm{x}_i = \bm{0}$ for (\ref{eq:yy1}). Thus the transformed
vectors  $\bm{y}_i$ are also centered in the same way. The rescaling factors $b_k$, $k=1, \ldots,
K$, are actually computed by multiplying $(\sum_{i=1}^N m_i)^{1/2}$ for the weighted variance and
$N^{1/2}$ for the unweighted variance.  In other words, (\ref{eq:ymy1}) is replaced by $\sum_{i=1}^N
m_i y_{ik}^2 = \sum_{i=1}^N m_i$, and (\ref{eq:yy1}) is replaced by $\sum_{i=1}^N  y_{ik}^2 = N$.
This makes the magnitude of the components $y_{ik} = O(1)$, so that the interpretation becomes
easier. The matching error is then computed as $  \phi_k/(\sum_{i=1}^N m_i)$.

\section{Three types of matching errors} \label{sec:matchingerrors}

\subsection{Fitting error and true error}  \label{sec:fiterror}

$\bm{A}$ and $\bm{Y}$ are computed from $\bm{W}$ by the method of Section~\ref{sec:regularization}.
The matching error of the $k$-th component $\bm{y}^k$ is defined with respect to an arbitrary weight
matrix $\bmt{W}$   as  \[ \phi_k(\bm{W},\bmt{W}) = \frac{1}{2} \sum_{i=1}^N \sum_{j=1}^N \tilde
w_{ij} (y_{ik}- y_{jk})^2. \] We will omit $\bm{X}$ from the notation, since it is fixed throughout.
We also omit $\Delta \bm{G}$ and $\Delta \bm{H}$ from the notation above, although the matching
error actually depends on them. We define the fitting error as  \[  \phi_k^{\rm fit}(\bm{W}) =
\phi_k(\bm{W}, \bm{W})     \]     by letting $\bmt{W}=\bm{W}$. This is the $\phi_k$ in
Section~\ref{sec:matching}. On the other hand, we define the true error as  \[ \phi_k^{\rm
true}(\bm{W},\bmb{W})=\phi_k(\bm{W},\epsilon\bmb{W}) \] by letting $\bmt{W}=
\epsilon\bmb{W}$. Since $\phi_k(\bm{W},\bmt{W}) $ is proportional to $\bmt{W}$, we have used
$\epsilon\bmb{W}$ instead of $\bmb{W}$ so that $\phi_k^{\rm fit}$ and $\phi_k^{\rm true}$ are
directly comparable with each other.   Let $E(\cdot)$ denote the expectation with respect to
(\ref{eq:zij}). Then  $E(\bm{W}) = \epsilon \bmb{W}$ because $E(w_{ij}) = E(z_{ij} )\bar w_{ij} =
\epsilon \bar w_{ij}$. Therefore, $\bmt{W}= \epsilon\bmb{W}$ is comparable with $\bmt{W}=\bm{W}$.

\subsection{Resampling matching weights for cross-validation error} \label{sec:cverror}

The bias of the fitting error for estimating the true error is $O(N\ohh P)$ as shown in
Section~\ref{sec:experror}.  We adjust this bias by cross-validation as follows. The observed weight
$\bm{W}$ is randomly split into $\bm{W}-\bm{W^*}$ for learning and $\bm{W}^*$ for testing, and the
matching error $\phi_k(\bm{W}-\bm{W^*}, \bm{W}^*)$ is computed. By repeating it several times for
taking the average of the matching error, we will get a cross-validation (cv) error. More formal
definition of the cv error is explained below.

The matching weights  $w_{ij}^*$ are randomly resampled from the observed matching weights $w_{ij}$
with small probability $\kappa>0$.  Let $z_{ij}^*=z_{ji}^*\in \{0,1\}$,  $i,j=1,\ldots,N$, be
samples from Bernoulli trial with success probability $\kappa$, where the number of independent
elements is $N(N+1)/2$ due to the symmetry. Then the resampled matching weights are defined as
\begin{equation} \label{eq:zijstar}     w_{ij}^* = z_{ij}^* w_{ij},\quad     P(z_{ij}^*=1 |
\bm{W})=\kappa. \end{equation} Let $E^*(\cdot |\bm{W})$, or $E^*(\cdot)$ by omitting $\bm{W}$,
denote the conditional expectation given $\bm{W}$. Then  $E^*(\bm{W}^*) = \kappa \bm{W}$ because
$E^*(w_{ij}^*) = E^*(z_{ij}^* ) w_{ij} = \kappa w_{ij}$. By noticing $E^*((1-
\kappa)\inv(\bm{W}-\bm{W^*}))=\bm{W}$ and $E^*(\kappa\inv\bm{W}^*)=\bm{W}$,  we use $(1-
\kappa)\inv(\bm{W}-\bm{W^*})$ for learning and $\kappa\inv\bm{W}^*$ for testing so that the cv error
is comparable with the fitting error. Thus we define the cv error as     \[         \phi_k^{\rm
cv}(\bm{W}) = E^* \Bigl\{ \phi_k ((1- \kappa)\inv(\bm{W}-\bm{W^*}), \kappa\inv\bm{W}^* ) \mid \bm{W}
\Bigr\}.     \] 

The conditional expectation $E^*(\cdot|\bm{W})$ is actually computed as the average over several
$\bm{W}^*$'s. In the numerical computation of Section~\ref{sec:simulation}, we resample $\bm{W}^*$
 from $\bm{W}$ with $\kappa=0.1$ for 30 times. On the other hand,
we resampled $\bm{W}^*$ from $\bm{W}$ only once with  $\kappa=0.1$ in Section~\ref{sec:example}, 
because the average may not be necessary for large $N$. For each $\bm{W}^*$,  we
compute $\phi_k((1-
\kappa)\inv(\bm{W}-\bm{W^*}), \kappa\inv\bm{W}^* )$ by the method of
Section~\ref{sec:regularization}. $\bm{A}^*$ is computed as the solution of the optimization problem
by replacing $\bm{W}$ with $(1- \kappa)\inv(\bm{W}-\bm{W^*})$. Then $\bm{Y}^*$ is computed with
rescaling factor $b_k^*=((1- \kappa)\inv\bm{a}^{*k\,T} \bm{X}^T (\bm{M}-\bm{M^*}) \bm{X}
\bm{a}^{*k})^{-1/2}$.

\subsection{Link sampling vs. node sampling} \label{sec:node-sampling}

The matching weight matrix is interpreted as the adjacency matrix of a weighted graph (or network)
with nodes of the data vectors. The sampling scheme (\ref{eq:zij}) as well as the resampling scheme
(\ref{eq:zijstar}) is interpreted as link sampling/resampling.  Here we describe another
sampling/resampling scheme. Let $z_i  \in \{0, 1\}$, $i=1, \ldots, N$, be samples from Bernoulli
trial with success probability $\xi>0$. This is interpreted as node sampling, or equivalently
sampling of data vectors, by taking $z_i$ as the indicator variable of sampling node $i$. Then the
observed matching weights may be defined as
\begin{equation} \label{eq:zij-node}
  w_{ij} = z_i z_j \bar w_{ij}, \quad P(z_i=1) = \xi,
\end{equation}
meaning $w_{ij}$ is sampled if both node $i$ and node $j$ are sampled together. For computing
$\phi_k^{\rm true}$, we set $\epsilon=\xi^2$. The resampling scheme of $\bm{W}-\bm{W^*}$ should
simulate the sampling scheme of $\bm{W}$. Therefore, $z_i^* \in \{0, 1\}$, $i=1, \ldots, N$, are
samples from Bernoulli trial with success probability $1-\nu>0$, and resampling scheme is defined as
\begin{equation} \label{eq:zijstar-node}
  w_{ij}^* = (1-z_i^* z_j^*) w_{ij}, \quad P(z_i^*=1) = 1-\nu.
\end{equation} 
For computing $\phi_k^{\rm cv}$, we set $\kappa = 1- (1-\nu)^2 \approx 2 \nu$. In the numerical
computation of Section~\ref{sec:simulation}, we resample $\bm{W}^*$ with $\nu=0.05$ for 30 times.

The vector $\bm{x}_i$ does not contribute to the optimization problem if $z_i=0$ (then $w_{ij}=0$
for all $j$) or $z_i^*=0$ (then $w_{ij} - w_{ij}^*=0$ for all $j$). Thus the node
sampling/resampling may be thought of as the ordinary sampling/resampling of data vectors, while the
link sampling/resampling is a new approach. These two methods will be compared in the simulation
study of Section~\ref{sec:simulation}. Note that the two methods become identical if the number of
nonzero elements in each row (or column) of $\bmb{W}$ is not more than one, or equivalently  the
numbers of links are zero or one for all vectors. CCA is a typical example: there is a one-to-one
correspondence between the two domains. We expect that the difference of the two methods becomes
small for extremely sparse $\bm{W}$.

We can further generalize the sampling/resampling scheme. Let us introduce correlations ${\rm
corr}(z_{ij}, z_{kl})$ and ${\rm corr}(z^*_{ij}, z^*_{kl})$ in (\ref{eq:zij}) and (\ref{eq:zijstar})
instead of independent Bernoulli trials. The link sampling/resampling corresponds to ${\rm
corr}(z_{ij}, z_{kl})={\rm corr}(z^*_{ij}, z^*_{kl})=\delta_{ik}\delta_{jl}$ if indices are confined
to $i \ge j$ and $k \ge l$. The node sampling has additional nonzero correlations if a node is
shared by the two links:  ${\rm corr}(z_{ij}, z_{ik})=\xi/(1+\xi) \approx \xi - \xi^2$. Similarly the node resampling
has nonzero correlations ${\rm corr}(z_{ij}^*, z_{ik}^*)=(1-\nu)/(2-\nu) \approx \frac{1}{2} -
\frac{\nu}{4}$. In the theoretical argument, we only consider the simple link sampling/resampling.
It is a future work to incorporate the structural correlations into the equations such as
(\ref{eq:covdhatwlm}) and  (\ref{eq:expvardhwlmstar}) in Appendix~\ref{sec:technical} for
generalizing the theoretical results.

Although we have multiplied $z_{ij}^*$ or $z_i^*$ to all $w_{ij}$ elements in the mathematical
notations, we actually look at only nonzero elements of $w_{ij}$ in computer software. The
resampling algorithms are implemented very efficiently for sparse $\bm{W}$.

\section{Simulation study} \label{sec:simulation}

We have generated twelve datasets for CDMCA of $D=3$ domains with $P=140$ and $N=875$ as shown in
Table~\ref{tab:simulation-parameters}. The details of the data generation is given in
Appendix~\ref{sec:simulation-data}. Considered are two sampling schemes (link sampling and node
sampling), two true matching weights ($\bmb{W}_A$ and $\bmb{W}_B$), and three sampling probabilities
(0.02, 0.04, 0.08). The matching weights are shown in Fig.~\ref{fig:weight-a} and
Fig.~\ref{fig:weight-b}. The observed $\bm{W}$ are very sparse. For example, the number of nonzero
elements in the upper triangular part of $\bm{W}$ is 162 for Experiment~1, and it is 258 for
Experiment~5.

$\bm{X}$ is generated by projecting underlying $5 \times 5$ grid points in $\mathbb{R}^2$ to the
higher dimensional spaces for domains with small additive noise. Scatter plots of $\bm{y}^k$,
$k=1,2$, are shown for Experiments 1 and 5, respectively, in Fig.~\ref{fig:exp1} and
Fig.~\ref{fig:exp5}. The underlying structure of $5\times 5$ grid points is well observed in the
scatter plots for $\ga_M=0.1$, while the structure is obscured in the plots for $\ga_M=0$. For
recovering the hidden structure in $\bm{X}$ and $\bm{W}$, the value $\ga_M=0.1$ looks better than
$\ga_M=0$.

In Fig.~\ref{fig:exp1-fitcv} and Fig.~\ref{fig:exp5-fitcv}, the curves of matching errors are shown
for the components $\bm{y}^k$, $k=1,\ldots, 20$. They are plotted at $\gamma_M=0, 0.001, 0.01, 0.1,
1$. The smallest two curves ($k=1,2$) are clearly separated from the other 18 curves
($k=3,\ldots,20$), suggesting correctly $K=2$. Looking at the true matching error $\phi_k^{\rm
true}(\bm{W},\bmb{W})$, we observe that the true error is minimized at $\gamma_M=0.1$ for $k=1,2$.
The fitting error $\phi_k^{\rm fit}(\bm{W})$, however, underestimates the true error, and wrongly
suggests $\gamma_M=0$.

For computing the the cv error $\phi_k^{\rm cv}(\bm{W})$, we used both the link resampling and the
node resampling. These two cv errors accurately estimates the true error in Experiment~1, where each
node has very few links in $\bm{W}$. In Experiment~5, some nodes have more links, and only the link
resampling estimates the true error very well.

In each of the twelve experiments, we generated 160 datasets of $\bm{W}$. We computed the expected
values of the matching errors by taking the simulation average of them. We look at the bias of a
matching error divided by its true value. $ ({E(\phi_k^{\rm fit})-E(\phi_k^{\rm
true})})/{E(\phi_k^{\rm true})}$ or $ ({E(\phi_k^{\rm cv})-E(\phi_k^{\rm true})})/{E(\phi_k^{\rm
true})}$ is computed for $k=1,\ldots, 10$, with $\gamma_M=0.001, 0.01, 0.1, 1$. These $10\times 4=40$ values are
used for each boxplot in Fig.~\ref{fig:boxplots-w} and Fig.~\ref{fig:boxplots-u}. The two figures
correspond to the two types of rescaling factor $b_k$, respectively, in
Section~\ref{sec:regularization}. We observe that the fitting error underestimates the true error.
The cv error of link resampling is almost unbiased in Experiments 1 to 6, where $\bm{W}$ is
generated by link sampling. This verifies our theory of Section~\ref{sec:asymptotic} to claim that
the cv error is asymptotically unbiased for estimating the true error. However, it behaves poorly in
Experiments 7 to 12, where $\bm{W}$ is generated by node sampling. On the other hand, the cv error
of node resampling performs better than link resampling in Experiments 7 to 12, suggesting
appropriate choice of resampling scheme may be important.

The two rescaling factors behave very similarly overall. Comparing Fig.~\ref{fig:boxplots-w} and
Fig.~\ref{fig:boxplots-u}, we observe that the unweighted variance may lead to more stable results
than the weighted variance. This may happen because only a limited number of $y_{1k},\ldots,y_{Nk}$
is used in the weighted variance when $\bm{W}$ is very sparse.

%%%%%%%%%%%%%%%%%%%%%%%%%%%%%%%%%%%%%%%%%%%%%%%%%%%%%%%%%%%%%%%%%%%%%%%%%%%%%%%%%%%%%%%%%%%%%%%%

\section{Asymptotic theory of the matching errors} \label{sec:asymptotic}

\subsection{Main results} \label{sec:main}

We investigate the three types of matching errors defined in Section~\ref{sec:matchingerrors}. We
work on the asymptotic theory for sufficiently large $N$ under the assumptions given below. Some
implications of these assumptions are mentioned in Section~\ref{sec:technotes}.

\begin{enumerate}[({A}1)]

\item \label{asm:order}  We consider the limit of $N\to\infty$ for asymptotic expansions. $P$ is a
constant or increasing as $N\to\infty$, but not too large as $P = o(N^{1/2})$. The sampling scheme
is (\ref{eq:zij}), and the resampling scheme is (\ref{eq:zijstar}). The sampling probability of
$\bm{W}$ from $\bmb{W}$ is $\epsilon=o(1)$ but not too small as $\epsilon\inv=o(N)$. The resampling
probability of $\bm{W}^*$ from $\bm{W}$ is proportional to $N\inv$; $\kappa=O(N\inv)$ and
$\kappa\inv=O(N)$.

\item \label{asm:w}  The true matching weights are $\bar w_{ij}=O(1)$. In general, the asymptotic
order of a matrix or a vector is defined as the maximum order of the elements in this paper, so we
write $\bmb{W}=O(1)$. The number of nonzero elements of each row (or each column) is $ \#\{\bar
w_{ij}\neq 0, j=1,\ldots,N \} = O(\epsilon\inv)$ for $i=1,\ldots, N$.

\item \label{asm:x}  The elements of $\bm{X}$ are $x_{ik}=O(N\oh)$. This is only a technical
assumption for the asymptotic argument. In practice, we may assume $x_{ik}=O(1)$ for MCA
computation, and redefine $\bm{x}^k := \bm{x}^k/\|\bm{x}^k\|$ for the theory.

\item \label{asm:gamma} Let $\ga$ be a generic order parameter for representing the magnitude of
regularization terms as $\Delta \bm{G} = O(\ga)$ and  $\Delta \bm{H} = O(\ga)$.  For example, we put
$\bm{L}_M=O(1)$ and $\ga_M=\ga$. We assume $\ga=O(N\oh)$.

\item \label{asm:eigen}  All the $P$ eigenvalues are distinct from the others; $\lambda_i \neq
\lambda_j$ for $i\neq j$. $\bm{A}$ is of full rank, and assume $\bm{A}=O(P\inv)$.  We evaluate
$\phi_k$ only for $k=1,\ldots,J$, for some $J\le P$. We assume that $J$ is bounded and  $(\lambda_i -
\lambda_j)\inv=O(1)$ for $i\neq j$ with $i\le P$, $j\le J$.  These assumptions apply to all the
cases under consideration such as  $\Delta \bm{G}=\Delta \bm{H} = \bm{0}$ or $\bm{W}$ being replaced
by $\epsilon \bmb{W}$.

\end{enumerate}

\begin{theo} \label{theorem:main}
Under the assumptions mentioned above, the following equation holds.
\begin{equation} \label{eq:main}
  E(\phi_k^{\rm fit}(\bm{W}) -  \phi_k^{\rm true}(\bm{W},\bmb{W}) ) = 
(1-\epsilon)  E(\phi_k^{\rm fit}(\bm{W}) -  \phi_k^{\rm cv}(\bm{W})) 
+O(N\ohhh P^2 + \ga^3 P^3).
\end{equation}
This implies 
\begin{equation} \label{eq:cvunbiased}
  E( \phi_k^{\rm cv}(\bm{W}) ) = E( \phi_k^{\rm true}(\bm{W},\bmb{W}) ) + O(\epsilon N\ohh P + N\ohhh P^2 + \ga^3 P^3).
\end{equation}
Therefore, the cross-validation error is an unbiased estimator of the true error by ignoring the 
higher-order term of $O(\epsilon N\ohh P + N\ohhh P^2 + \ga^3 P^3)=o(1)$, which is smaller than
$E( \phi_k^{\rm cv}(\bm{W}) )=O(1)$ for sufficiently large $N$.
\end{theo}

\begin{proof} By comparing (\ref{eq:expphikfitbias}) of Lemma~\ref{lemma:exptrue} and
(\ref{eq:expfitcv}) of Lemma~\ref{lemma:expcv}, we obtain (\ref{eq:main}) immediately. Then
(\ref{eq:cvunbiased}) follows, because  $E(\phi_k^{\rm fit}(\bm{W}) -  \phi_k^{\rm true}(\bm{W}))=
O(N\ohh P)$ as mentioned in Lemma~\ref{lemma:exptrue}.
\end{proof}

We use $N, P, \ga, \epsilon$ in expressions of asymptotic orders. Higher order terms can be
simplified by substituting $P =o(N^{1/2})$ and $\ga=O(N\oh)$. For example, $O(N\ohhh P^2 + \ga^3
P^3)=o(N\oh + 1)=o(1)$ in (\ref{eq:main}). However, we attempt to leave the terms with $P$ and $\ga$
for finer evaluation.

The theorem justifies the link resampling for estimating the true matching error under the link
sampling scheme. Now we have theoretically confirmed our observation that the cross-validation is
nearly unbiased in the numerical examples of Sections~\ref{sec:example} and \ref{sec:simulation}.
Although the fitting error underestimates the true error in the numerical examples, it has not been
clear that the bias is negative in some sense from the expression of the bias given in
Lemma~\ref{lemma:exptrue}.

In the following subsections, we will discuss lemmas used for the proof of
Theorem~\ref{theorem:main}.  In Section~\ref{sec:technotes}, we look at the assumptions of the
theorem. In Section~\ref{sec:lambdachange},  the solution of the optimization problem and the
matching error are expressed in terms of small perturbation of the regularization matrices.
Lemma~\ref{lemma:lambdachange} gives the asymptotic expansions of the eigenvalues $\lambda_1,\ldots,
\lambda_P$ and the linear transformation matrix $\bm{A}$ in terms of  $\Delta \bm{G}$ and $\Delta
\bm{H}$. Lemma~\ref{lemma:errorchange} gives the asymptotic expansion of the matching error
$\phi_k(\bm{W}, \bmt{W})$ in terms of $\Delta \bm{G}$ and $\Delta \bm{H}$. Using these results, the
bias for estimating the true error is discussed in Section~\ref{sec:experror}.
Lemma~\ref{lemma:exptrue} gives the asymptotic expansion of the bias of the fitting error, and
Lemma~\ref{lemma:expcv} shows that the cross-validation adjusts the bias. All the proofs of lemmas
are given in Appendix~\ref{sec:technical}.

\subsection{Some technical notes} \label{sec:technotes}

We put $K=P$ for the definitions of matrices in Section~\ref{sec:asymptotic} without losing
generality. For characterizing the solution $\bm{A} \in \mathbb{R}^{P \times P} $ of the
optimization problem in Section~\ref{sec:regularization}, (\ref{eq:aga}) and (\ref{eq:ahalambda})
are now, with $\bm{\Lambda}=\diag(\lambda_1,\ldots,\lambda_P)$,
\begin{equation} \label{eq:agaaha} 
  \bm{A}^T \bm{G} \bm{A} = \bm{I}_P,\quad
  \bm{A}^T \bm{H} \bm{A} = \bm{\Lambda}. 
\end{equation}
We can do so, because the solution $\bm{a}^k$ and the eigenvalue $\lambda_k$ for the $k$-th
component  does not change for each $1\le k \le K$ when the value of $K$ changes. This property also
holds for the matching errors of the $k$-th component. Therefore a result shown for some $1 \le k
\le P$ with $K=P$ holds true for the same $k$ with any  $K \le P$, meaning that we can put $K=P$ in
Section~\ref{sec:asymptotic}. In our asymptotic theory, however, we would like to confine $k$ to a
finite value. So we restrict our attention to $1 \le k
\le J$ in the assumption (A\ref{asm:eigen}).

The assumption of $N$ and $P$ in (A\ref{asm:order}) covers many applications in practice. The theory
may work for the case that $N$ is hundreds and $P$ is dozens, or for a more recent case that $N$ is
millions and $P$ is hundreds.

Asymptotic properties of $\bm{W}$ follow from (A\ref{asm:order}) and (A\ref{asm:w}). Since the
elements of $\bm{W}$ are sampled from $\bmb{W}$, we have $\bm{W}=O(1)$ and $\bm{M}=O(1)$. The number
of nonzero elements for each row (or each column) is 
\begin{equation} \label{eq:wsparse}
 \#\{w_{ij}\neq 0, j=1,\ldots,N \} = O(1),\quad i=1,\ldots,N,
\end{equation}
and the total number of nonzero elements is $\#\{w_{ij}\neq0 ,i,j=1,\ldots,N\} = O(N)$. Thus
$\bm{W}$ is assumed to be a very sparse matrix. Examples of such sparse matrices are image-tag links
in Flickr, or more typically friend links in Facebook. Although the label domains of MNIST dataset
do not satisfy (\ref{eq:wsparse}) with many links to images, our method still worked very well.

The assumption of $\kappa$ in (A\ref{asm:order}) implies that the number of nonzero elements in
$\bm{W}^*$ is $O(1)$. Similarly to the leave-one-out cross-validation of data vectors, we resample
very few links in our cross-validation.

From (A\ref{asm:x}), $\sum_{i=1}^N m_i x_{ik} x_{jl} =O(N (N\oh)^2)=O(1)$, and thus $\bm{X}^T \bm{M}
\bm{X} = O(1)$, and $\bm{G}=O(1)$. Also, $\sum_{i=1}^N \sum_{j=1}^N w_{ij} x_{ik} x_{jl} =
\sum_{(i,j) : w_{ij}\neq 0}   w_{ij} x_{ik} x_{jl} = O(N (N\oh)^2) = O(1)$, and thus $\bm{X}^T
\bm{W} \bm{X} = O(1)$, and $\bm{H}=O(1)$. From (A\ref{asm:eigen}), $\sum_{i=1}^P \sum_{j=1}^P
(\bm{G})_{ij} a_{ik} a_{jk} = O(P^2 (P\inv)^2) = O(1)$, which is  necessary for $\bm{A}^T \bm{G}
\bm{A}=\bm{I}_P$. Then $\bm{Y}=\bm{X} \bm{A} = O(P N\oh P\inv) = O(N\oh)$.

The assumptions on the eigenvalues described in (A\ref{asm:eigen}) may be difficult to hold in
practice. In fact, there are many zero eigenvalues in the examples of  Section~\ref{sec:simulation};
60 zeros, 40 positives ($K^+=40$), and 40 negatives in the $P=140$ eigenvalues. Looking at Experiment~1,
however, we observed that $\phi_k^{\rm cv} \approx \phi_k^{\rm true}$ holds well and $E(\phi_k^{\rm
cv} ) = E(\phi_k^{\rm true})$ holds very accurately for $k=1,\ldots, 40$ ($\lambda_k>0$) when $\ga_M
> 0$. The eigenvalues for $\ga_M=0.1$ are $\lambda_1=0.988, \lambda_2=0.978,
\lambda_3=0.562, \lambda_4=0.509, \lambda_5=0.502,\ldots$, where $\lambda_1-\lambda_2$ looks very
small, but it did not cause any problem. On the other hand, the eigenvalues for $\ga_M=0$ are
$\lambda_1=0.999, \lambda_2=0.997, \lambda_3=0.973, \lambda_4=0.971, \lambda_5=0.960,\ldots$, where
some $\lambda_k$ are very close to each other. This might be a cause for the deviation of
$\phi_k^{\rm cv}$ from $\phi_k^{\rm true}$ when $\ga_M=0$.

%> load("test6-print3-dataset.rda")
% > round(ans9$f01$cor[1:6],3)
% [1] 0.988 0.978 0.562 0.509 0.502 0.485
% > round(ans9$f00$cor[1:6],3)
% [1] 0.999 0.997 0.973 0.971 0.960 0.955

%%%%%%%%%%%%%%%%%%%%%%%%%% Lemmas %%%%%%%%%%%%%%%%%%%%

\subsection{Small change in $\bm{A}$, $\bm{\Lambda}$, and the matching error} \label{sec:lambdachange}

Here we show how $\bm{A}$, $\bm{\Lambda}$ and $\phi_k(\bm{W},\bmt{W})$ depend on $\Delta
\bm{G}$  and $\Delta \bm{H}$. Recall that terms with hat are for $\Delta \bm{G} = \Delta \bm{H} =
\bm{0}$ as defined in Section~\ref{sec:spectral}; $\bmh{G} = \bm{X}^T \bm{M} \bm{X}$, $\bmh{H} =
\bm{X}^T \bm{W} \bm{X}$, $\bmh{Y} = \bm{X} \bmh{A}$. The optimization problem is characterized as
$\bmh{A}^T \bmh{G} \bmh{A} = \bm{I}_P$, $\bmh{A}^T \bmh{H} \bmh{A} = \bmh{\Lambda}$, where
$\bmh{\Lambda}=\diag(\hlambda_1,\ldots,\hlambda_P)$ is the diagonal matrix of the eigenvalues
$\hlambda_1,\ldots, \hlambda_P$. Then  $\bm{A}$ and $\bm{\Lambda}$ are defined in
Section~\ref{sec:regularization} for $\bm{G}=\bmh{G}+ \Delta \bm{G}$ and $\bm{H}=\bmh{H}+ \Delta
\bm{H}$. They satisfy (\ref{eq:agaaha}). The asymptotic expansions for 
 $\bm{A}$ and $\bm{\Lambda}$  will be given in Lemma~\ref{lemma:lambdachange} using
 \begin{equation*} 
\bm{g} =\bmh{A}^T \Delta \bm{G} \bmh{A},\quad
\bm{h} = \bmh{A}^T \Delta \bm{H} \bmh{A} \in \mathbb{R}^{P \times P}.
\end{equation*}
For proving the lemma in Section~\ref{sec:proof-lambdachange},
we will solve (\ref{eq:agaaha}) under the small perturbation of $\bm{g}$ and $\bm{h}$.

\begin{lemm} \label{lemma:lambdachange} Let $\Delta \bm{\Lambda} = \bm{\Lambda} - \bmh{\Lambda}$
with elements $\Delta \lambda_i = \lambda_i - \hlambda_i$, $i=1,\ldots,P$. Define $\bm{C} \in
\mathbb{R}^{P \times P}$ as $\bm{C} = \bmh{A}\inv \bm{A} - \bm{I}_P$ so that $\bm{A} =
\bmh{A}(\bm{I}_P + \bm{C})$.  Here $\bmh{A}\inv$ exists, since we assumed that $\bmh{A}$ is of full
rank in (A\ref{asm:eigen}).   We assume $\bm{g}=O(\ga)$ and $\bm{h}=O(\ga)$.

Then the elements of $\Delta \bm{\Lambda}=O(\ga)$ are, for $i=1,\ldots,P$,
\begin{equation}\label{eq:Dlambdai}
	\Delta \lambda_i =  - (g_{ii} \hlambda_i - h_{ii}) + \delta \lambda_i
\end{equation}
with $\delta \lambda_i=O(\ga^2 P)$ defined for $i\le J$ as
\begin{equation}\label{eq:dlambdai}
	\delta \lambda_i = g_{ii} (g_{ii} \hlambda_i -  h_{ii}) - 
	\sum_{j \neq i}(\hlambda_j - \hlambda_i)\inv (g_{ij} \hlambda_i - h_{ij})^2 + O(\ga^3 P^2),
\end{equation}
where $\sum_{j \neq i}$ is the summation over $j=1,\ldots,P$, except for $j = i$.

The elements of the diagonal part of $\bm{C}=O(\ga)$ are, for $i=1,\ldots,P$,
\begin{equation}\label{eq:cii}
	c_{ii} = -\frac{1}{2} g_{ii} + \delta c_{ii}
\end{equation}
with $\delta c_{ii}=O(\ga^2 P)$ defined for $i\le J$ as
\begin{equation}\label{eq:dcii}
	\delta c_{ii} = \frac{3}{8} (g_{ii})^2
	-\sum_{j\neq i} (\hlambda_j - \hlambda_i)^{-2} (g_{ij} \hlambda_i - h_{ij})
\Bigl\{
	g_{ij} \Bigl(\hlambda_j - \frac{1}{2} \hlambda_i \Bigr)  - \frac{1}{2}  h_{ij}
	\Bigr\} + O(\ga^3 P^2),
\end{equation}
and the elements of the off-diagonal ($i \neq j$) part of $\bm{C}$ are, 
for either $i \le J, j \le P$ or $i \le P, j\le J$, i.e., one of $i$ and $j$ is not greater than $J$,
\begin{equation}\label{eq:cij}
	c_{ij} = (\hlambda_i - \hlambda_j)\inv (g_{ij} \hlambda_j - h_{ij})
	+ \delta c_{ij}
\end{equation}
with $\delta c_{ij} =O(\ga^2 P)$ defined for $i\le P, j \le J$ as
\begin{equation}\label{eq:dcij}
\begin{split}
	\delta c_{ij} =&
	-\frac{1}{2} (\hlambda_i - \hlambda_j)\inv(g_{ij} \hlambda_j - h_{ij})g_{jj}
	-(\hlambda_i - \hlambda_j)^{-2}(g_{ij} \hlambda_i - h_{ij}) ( g_{jj} \hlambda_j - h_{jj})\\
	&+ (\hlambda_i - \hlambda_j)\inv
	\sum_{k \neq j} (\hlambda_k - \hlambda_j)\inv
	(g_{ki} \hlambda_j - h_{ki})
	(g_{kj} \hlambda_j - h_{kj}).
\end{split}
\end{equation}
\end{lemm}

The asymptotic expansion of $\phi_k(\bm{W},\bmt{W})$ is given in Lemma~\ref{lemma:errorchange}.
For proving the lemma in Section~\ref{sec:proof-errorchange},
the matching error is first expressed by $\bm{C}$, and then the result of Lemma~\ref{lemma:lambdachange}
is used for simplifying the expression.

\begin{lemm} \label{lemma:errorchange}  Let $\bmt{S} = \bmh{A}^T \bm{X}^T (\bmt{M} - \bmt{W}) \bm{X}
\bmh{A} = \bmh{Y}^T (\bmt{M} - \bmt{W}) \bmh{Y} \in \mathbb{R}^{P\times P}$. We assume
$\bmt{S}=O(1)$, which holds for, say, $\bmt{W} = \bm{W}$ and $\bmt{W} = \epsilon\bmb{W}$.  We assume
$\bm{g}=O(\ga)$ and  $\bm{h}=O(\ga)$. Then the matching error of Section~\ref{sec:matching} is
expressed asymptotically as
\begin{equation} \label{eq:phikwtw}
\begin{split}
\phi_k(\bm{W}, & \bmt{W}) =\stilde_{kk} + \sum_{i\neq k} 2  (\hlambda_i - \hlambda_k)\inv \stilde_{ik}
(g_{ik} \hlambda_k - h_{ik})\\
& - \sum_{i\neq k} (\hlambda_i - \hlambda_k)^{-2} \Bigl\{
\stilde_{kk} (g_{ik} \hlambda_k - h_{ik})^2
+\stilde_{ik} (g_{ki}\hlambda_i - h_{ki})
(g_{kk} \hlambda_k - h_{kk})\Bigr\}\\
&+\sum_{i\neq k} \sum_{j\neq k} (\hlambda_i - \hlambda_k)\inv(\hlambda_j - \hlambda_k)\inv \Bigr\{
 2\stilde_{ik} (g_{ij} \hlambda_k - h_{ij})(g_{jk} \hlambda_k -h_{jk}) \\
& \hspace{10em} +\stilde_{ij} (g_{ik}\hlambda_k - h_{ik})(g_{jk} \hlambda_k - h_{jk})\Bigr\} + O(\ga^3 P^3).
\end{split}
\end{equation}
Let us further assume $\bmt{W}=\bm{W}$. Then $\bmt{S}=\bm{I}_P - \bmh{\Lambda}$ with elements
$\stilde_{ij}=\delta_{ij}(1-\hlambda_i)$. Substituting it into (\ref{eq:phikwtw}), we get an
expression of $\phi_k^{\rm fit}(\bm{W})=\phi_k(\bm{W},\bm{W})$ as \begin{equation}\label{eq:phikfit}
\phi_k^{\rm fit}(\bm{W})=1-\hlambda_k - \sum_{i\neq k}(\hlambda_i - \hlambda_k)\inv(g_{ik}\hlambda_k
- h_{ik})^2     +O(\ga^3 P^3). \end{equation}
\end{lemm}

It follows from  (A\ref{asm:gamma}) and (A\ref{asm:eigen}) that the magnitude of
the regularization terms are expressed as $\bm{g}=O(\ga)$ and $\bm{h}=O(\ga)$.
This is mentioned as an assumption
in Lemma~\ref{lemma:lambdachange} and Lemma~\ref{lemma:errorchange} for the sake of clarity.
Although these two lemmas are shown
for the regularization terms, they hold generally for any small perturbation
other than the regularization terms. Later, in the proofs of Lemma~\ref{lemma:exptrue}
and Lemma~\ref{lemma:expcv}, we will apply 
 Lemma~\ref{lemma:lambdachange} and Lemma~\ref{lemma:errorchange} to perturbation of other types
 with $\gamma=O(N^{-1/2})$ or $\gamma=O(N^{-1})$.

\subsection{Bias of the fitting error} \label{sec:experror}

Let us consider the optimization problem with respect to $\epsilon \bmb{W}$. We define $\bmb{A}$ and
$\bmb{\Lambda}$ as the solution of $\bmb{A}^T \bmb{G} \bmb{A} = \bm{I}_P$ and $\bmb{A}^T \bmb{H}
\bmb{A} = \bmb{\Lambda}$ with $\bmb{G} = \epsilon \bm{X}^T \bmb{M} \bm{X}$ and  $\bmb{H} = \epsilon
\bm{X}^T \bmb{W} \bm{X}$. The eigenvalues are $\blambda_1,\ldots,\blambda_P$ and the matrix is
$\bmb{\Lambda}=\diag(\blambda_1,\ldots,\blambda_P)$. We also define $\bmb{Y} = \bm{X} \bmb{A}$.
These quantities correspond to those with hat, but $\bm{W}$ is replaced by $\epsilon \bmb{W}$. We
then define matrices representing change from $\epsilon\bmb{W}$ to $ \bm{W}$: $\Delta
\bmh{W}   = \bm{W} - \epsilon\bmb{W}$, $\Delta \bmh{M}   = \bm{M} - \epsilon\bmb{M}$,
$\Delta \bmh{G} = \bm{X}^T \Delta \bmh{M} \bm{X}$ and $\Delta \bmh{H} = \bm{X}^T \Delta \bmh{W}
\bm{X}$.

In Section~\ref{sec:lambdachange}, $\bm{g}$ and $\bm{h}$ are used for describing change with respect
to the regularization terms $\Delta\bm{G}$ and $\Delta\bm{H}$.  Quite similarly,  \begin{equation*}
\bmh{g} = \bmb{A}^T \Delta \bmh{G} \bmb{A} = \bmb{Y}^T \Delta \bmh{M} \bmb{Y},\quad \bmh{h} =
\bmb{A}^T \Delta \bmh{H} \bmb{A} = \bmb{Y}^T \Delta \bmh{W} \bmb{Y} \end{equation*} will be used for
describing change  with respect to $\Delta \bmh{G} $ and $\Delta
\bmh{H} $, namely, change from $\epsilon\bmb{W}$ to $ \bm{W}$. The elements of $\bmh{g}=(\hat
g_{ij})$ and $\bmh{h}=(\hat h_{ij})$ are
\begin{equation*}
\begin{split}
 \hat g_{ij} &= (\bmb{y}^i)^T \Delta\bmh{M} \bmb{y}^j = \sum_{l>m} (\bar y_{li} \bar y_{lj} + \bar y_{mi} \bar y_{mj}) \Delta \hat w_{lm} + \sum_{l=1}^N y_{li} y_{lj} \Delta \hat w_{ll},\\
 \hat h_{ij} &= (\bmb{y}^i)^T \Delta\bmh{W} \bmb{y}^j = \sum_{l>m} (\bar y_{li} \bar y_{mj} + \bar y_{mi} \bar y_{lj}) \Delta \hat w_{lm} + \sum_{l=1}^N y_{li} y_{lj} \Delta \hat w_{ll},
\end{split}
\end{equation*}
and they will be denoted as
\begin{equation}\label{eq:gyhy}
 \hat g_{ij} = \sum_{l \ge m} {\cal G}[\bmb{Y}]_{lm}^{ij} \Delta \hat w_{lm},\quad
  \hat h_{ij} = \sum_{l \ge m} {\cal H}[\bmb{Y}]_{lm}^{ij} \Delta \hat w_{lm},
\end{equation}
where $\sum_{l>m}=\sum_{l=2}^N \sum_{m=1}^{l-1}$ and $\sum_{l\ge m}=\sum_{l=1}^N \sum_{m=1}^{l}$.

We are now ready to consider the bias of the fitting error. The difference of the fitting error from
the true error is
\begin{equation} \label{eq:phikwdw}
	\phi_k^{\rm fit}(\bm{W}) -  \phi_k^{\rm true}(\bm{W},\bmb{W})
	 = \phi_k(\bm{W},\Delta\bmh{W}).
\end{equation}	
The asymptotic expansion of (\ref{eq:phikwdw}) is given by Lemma~\ref{lemma:errorchange} with
$\bmt{W}=\Delta\bmh{W}$. This expression will be rewritten by $\bmh{g}$ and $\bmh{h}$ in
Section~\ref{sec:proof-exptrue}  for proving the following lemma. For rewriting $\bmh{A}$ and
$\bmh{\Lambda}$ in terms of $\bmb{A}$ and $\bmb{\Lambda}$, Lemma~\ref{lemma:lambdachange} will be
used there.

\begin{lemm} \label{lemma:exptrue}
Bias of the fitting error for estimating the true error is expressed asymptotically as
\begin{equation} \label{eq:expphikfitbias}
  E(\phi_k^{\rm fit}(\bm{W}) -  \phi_k^{\rm true}(\bm{W},\bmb{W}) ) = 
{\rm bias}_k 
+O(N\ohhh P^2 + \ga^3 P^3),
\end{equation}
where ${\rm bias}_k = O(N\ohh P)$ is defined as
\[
  {\rm bias}_k = E\Bigl[
  -(\hat g_{kk} - \hat h_{kk})\hat g_{kk} 
  +\sum_{j\neq k} 2(\blambda_j - \blambda_k)\inv (\hat g_{jk} - \hat h_{jk}) (\hat g_{jk}\blambda_k - \hat h_{jk})
  \Bigr]
\]
using the elements of $\bmh{g}$, $\bmh{h}$ and $\bmb{\Lambda}$ mentioned above. ${\rm bias}_k$ can be expressed as
\begin{equation} \label{eq:biask}
	\begin{split} 
   {\rm bias}_k =
   \epsilon(1- \epsilon) \sum_{l\ge m}
   \bar w_{lm}^2 \Bigl[
   &- ( {\cal G}[\bmb{Y}]_{lm}^{kk}  - {\cal H}[\bmb{Y}]_{lm}^{kk}) {\cal G}[\bmb{Y}]_{lm}^{kk} \\
  &+\sum_{j\neq k} 2(\blambda_j - \blambda_k)\inv
  ( {\cal G}[\bmb{Y}]_{lm}^{jk}  - {\cal H}[\bmb{Y}]_{lm}^{jk})
 ( {\cal G}[\bmb{Y}]_{lm}^{jk}\blambda_k  - {\cal H}[\bmb{Y}]_{lm}^{jk})
   \Bigr].
\end{split}
\end{equation}
We also have $\bmh{g}=O(N\oh)$ and $\bmh{h}=O(N\oh)$.
\end{lemm}

In the following lemma, we will show that the bias of the fitting error is adjusted by the
cross-validation. For proving the lemma in Section~\ref{sec:proof-expcv}, we will first give the
asymptotic expansion of  $\phi_k((1- \kappa)\inv(\bm{W}-\bm{W^*}), \kappa\inv\bm{W}^* ) $ using
Lemma~\ref{lemma:errorchange}.  The expression will be rewritten using the change from $\bmh{A}$ and
$\bmh{\Lambda}$ to those with respect to $(1- \kappa)\inv(\bm{W}-\bm{W^*})$. The asymptotic
expansion of $\phi_k^{\rm cv}$ will be obtained by taking the expectation with respect to
(\ref{eq:zijstar}). Finally, we will take the expectation with respect to (\ref{eq:zij}).

\begin{lemm} \label{lemma:expcv}
The difference of the fitting error from the cross-validation error is expressed asymptotically as
\begin{equation} \label{eq:fitcvexpand}
\begin{split}
	\phi_k^{\rm fit}(\bm{W}) - & \phi_k^{\rm cv}(\bm{W}) =
 \sum_{l\ge m}
   w_{lm}^2 \Bigl[
 -  ( {\cal G}[\bmh{Y}]_{lm}^{kk}  - {\cal H}[\bmh{Y}]_{lm}^{kk}) {\cal G}[\bmh{Y}]_{lm}^{kk} \\
 &+\sum_{j\neq k} 2(\hlambda_j - \hlambda_k)\inv 
 ( {\cal G}[\bmh{Y}]_{lm}^{jk}  - {\cal H}[\bmh{Y}]_{lm}^{jk})
 ( {\cal G}[\bmh{Y}]_{lm}^{jk} \hlambda_k  - {\cal H}[\bmh{Y}]_{lm}^{jk})
 \Bigr]\\
 &\qquad +O(N\ohhhh P^2 +  N\ohh \ga P^2 + \ga^3 P^3),
\end{split}
\end{equation}
and its expected value is
\begin{equation} \label{eq:expfitcv}
	E(\phi_k^{\rm fit}(\bm{W}) -  \phi_k^{\rm cv}(\bm{W})) 
	= (1- \epsilon)\inv {\rm bias}_k +O(N\ohhh P^2 + \ga^3 P^3).
\end{equation}
\end{lemm}

%%%%%%%%%%%%%%%%%%%%%%%%%%%%%%%%%%%%%%%%%%%%%%%%%%%%%%%%%%%%%%%%%%%%%%%%%%%%%%%%%%%%%%%%%%%%%%%%
\clearpage

\appendix

\section{Cross-domain matching correlation analysis} \label{sec:cdmca}

\subsection{A simple coding for cross-domain matching} \label{sec:coding}

Here we explain how CDMCA is converted to MCA. Let $\bm{x}^{(d)}_i \in \mathbb{R}^{p_d}$,
$i=1,\ldots, n_d$, denote the data vectors  of domain $d$. Each $\bm{x}^{(d)}_i$  is coded as an
augmented vector $\bm{\tilde x}^{(d)}_i\in \mathbb{R}^P$ defined as
\begin{equation} \label{eq:coding}
  ( \bm{\tilde x}^{(d)}_i )^T = \Bigl( (\bm{0}_{p_1})^T,\ldots, (\bm{0}_{p_{d-1}})^T,
   (\bm{x}^{(d)}_i )^T, (\bm{0}_{p_{d+1}})^T,\ldots, (\bm{0}_{p_D})^T \Bigr).
\end{equation}
Here, $\bm{0}_p\in \mathbb{R}^p$ is the vector with zero elements.  This is a sparse coding
\citep{olshausen2004sparse} in the sense that nonzero elements for domains do not overlap each
other.  All the $N$ vectors of $D$ domains are now represented as points in the same $\mathbb{R}^P$.
We take these $N$ vectors as $\bm{x}_i\in \mathbb{R}^P$, $i=1, \ldots, N$. 
Then CDMCA reduces to MCA.  The data matrix is expressed as
\[
 \bm{X}^T = (\bm{\tilde x}^{(1)}_1,\ldots,\bm{\tilde x}^{(1)}_{n_1},\ldots,
\bm{\tilde x}^{(D)}_1,\ldots,\bm{\tilde x}^{(D)}_{n_D}).
\]
Let $\bm{X}^{(d)} \in \mathbb{R}^{n_d \times p_d}$ be the data matrix of domain $d$ defined as $(
\bm{X}^{(d)} ) ^T = (\bm{x}^{(d)}_1,\ldots,\bm{x}^{(d)}_{n_d})$. The data matrix of the augmented
vectors is now expressed as $\bm{X}=\Diag(\bm{X}^{(1)},\ldots,\bm{X}^{(D)})$, where $\Diag()$
indicates a block diagonal matrix.

Let us consider partitions of $\bm{A}$ and $\bm{Y}$ as $ \bm{A}^T = ((\bm{A}^{(1)})^T,\ldots,
(\bm{A}^{(D)})^T)$ and $ \bm{Y}^T = ((\bm{Y}^{(1)})^T, \ldots, (\bm{Y}^{(D)})^T)$ with $\bm{A}^{(d)} \in
\mathbb{R}^{p_d \times K}$ and $\bm{Y}^{(d)} \in \mathbb{R}^{n_d \times K}$.
Then the linear transformation of MCA, $\bm{Y} = \bm{X} \bm{A}$, is expressed as
\[
	\bm{Y}^{(d)} = \bm{X}^{(d)} \bm{A}^{(d)},\quad d=1,\ldots,D,
\]
which are the linear transformations of CDMCA. The matching weight matrix between domains $d$ and
$e$ is $\bm{W}^{(de)} = (w^{(de)}_{ij}) \in \mathbb{R}^{n_d \times n_e}$ for $d,e=1,\ldots,D$. They
are placed in a array to define $\bm{W} = (\bm{W}^{(de)}) \in \mathbb{R}^{ N \times N}$. Then the
matching error of MCA is expressed as
\[
	\phi = \frac{1}{2}\sum_{d=1}^D \sum_{e=1}^D \sum_{i=1}^{n_d} \sum_{j=1}^{n_e} w^{(de)}_{ij}
	\|\bm{y}^{(d)}_i - \bm{y}^{(e)}_j\|^2,
\]
which is the matching error of CDMCA.

Notice $\bm{M}=\Diag(\bm{M}^{(1)},\ldots,\bm{M}^{(D)})$ with $\bm{M}^{(d)} = \diag(
(\bm{W}^{(d1)},\ldots,\bm{W}^{(dD)})\bm{1}_N)$, and so $\bm{X}^T \bm{M} \bm{X} =
\Diag((\bm{X}^{(1)})^T \bm{M}^{(1)} \bm{X}^{(1)},\ldots, (\bm{X}^{(D)})^T
\bm{M}^{(D)} \bm{X}^{(D)} )$ is computed efficiently by looking at only the block diagonal parts.

Let us consider a simple case with $n_1=\cdots=n_D=n$, $N=nD$, $\bm{W}^{(de)}=c_{de} \bm{I}_n$ using
a coefficient $c_{de}\ge0$ for all $d,e=1,\ldots,D$. Then CDMCA reduces to a version of MCCA, where
associations between sets of variables are specified by the coefficients $c_{de}$
\citep{tenenhaus2011regularized}. Another version of MCCA with
all $c_{de}=1$ for $d\neq e$ is discussed extensively in \cite{takane2008regularized}. For the
simplest case of $D=2$ with $c_{12}=1$, $c_{11}=c_{22}=0$,  CDMCA reduces to CCA with  $\bm{W} =
\bigl(\begin{smallmatrix} \bm{0} & \bm{I}_n \\  \bm{I}_n & \bm{0} \end{smallmatrix} \bigr) $.

\subsection{auto-associative correlation matrix memory} \label{sec:associative-memory}

Let us consider $\bm{\Omega} \in \mathbb{R}^{P \times P}$ defined by
\[
	\bm{\Omega} = \frac{1}{2} \sum_{i=1}^N \sum_{j=1}^N w_{ij} (\bm{x}_i + \bm{x}_j)
	(\bm{x}_i + \bm{x}_j)^T.
\]
This is the correlation matrix of $\bm{x}_i + \bm{x}_j$ weighted by $w_{ij}$. Since $\bm{\Omega} =
\bm{X}^T \bm{M} \bm{X} + \bm{X}^T \bm{W} \bm{X} $, we have $\bm{\Omega} + \Delta \bm{G} + \Delta
\bm{H} = \bm{G} + \bm{H}$, and then MCA of Section~\ref{sec:regularization} is equivalent to
maximizing $\trace( \bm{A}^T (\bm{\Omega} + \Delta \bm{G} + \Delta \bm{H})  \bm{A} )$ with  respect
to $\bm{A}\in\mathbb{R}^{P\times K}$ subject to (\ref{eq:aga}). The role of $\bm{H}$ is now replaced
by  $\bm{\Omega} + \Delta \bm{G} + \Delta \bm{H} $. Thus MCA is interpreted as dimensionality
reduction of the correlation matrix $\bm{\Omega}$ with regularization term $\Delta
\bm{G} + \Delta \bm{H}$.

For CDMCA, the correlation matrix becomes
\[
	\bm{\Omega} = \frac{1}{2} \sum_{d=1}^D \sum_{e=1}^D \sum_{i=1}^{n_d} \sum_{j=1}^{n_e}
	 w_{ij}^{(de)} (\bm{\tilde x}_i^{(d)} + \bm{\tilde x}_j^{(e)} )
	 (\bm{\tilde x}_i^{(d)} + \bm{\tilde x}_j^{(e)} )^T.
\]
This is the correlation matrix of the input pattern of a pair of data vectors
\begin{equation} \label{eq:coding-associative-memory}
  ( \bm{\tilde x}^{(d)}_i + \bm{\tilde x}^{(e)}_j  )^T =
   \Bigl( 0,\ldots,0,  (\bm{x}^{(d)}_i )^T, 0,\ldots,0, (\bm{x}^{(e)}_j )^T ,0,\ldots,0 \Bigr)
\end{equation}
weighted by $ w_{ij}^{(de)}$. Interestingly, the same correlation matrix is found in one of the
classical neural network models. Any part of the memorized vector can be used as a key for recalling
the whole vector in the auto-associative correlation matrix memory
\citep{kohonen1972correlation}, also known as {\it Associatron}
\citep{nakano1972associatron}. This associative memory may recall $\bm{\tilde
x}^{(d)}_i + \bm{\tilde x}^{(e)}_j  $ for input key either $\bm{\tilde x}^{(d)}_i $ or $\bm{\tilde
x}^{(e)}_j $ if $w_{ij}^{(de)}>0$. In particular, the representation
(\ref{eq:coding-associative-memory}) of a pair of data vectors is equivalent to eq.~(14) of
\cite{nakano1972associatron}. Thus CDMCA is interpreted as dimensionality reduction of the
auto-associative correlation matrix memory for pairs of data vectors.

\section{Experimental details} \label{sec:experiment}

\subsection{MNIST handwritten digits} \label{sec:mnist}

The MNIST database of handwritten digits \citep{lecun1998gradient} has a training set of 60,000
images, and a test set of 10,000 images. Each image has $28 \times 28 = 784$ pixels of 256 gray
levels. We prepared a dataset of cross-domain matching with three domains for illustration purpose.

The data matrix $\bm{X}$ is specified as follows. The first domain ($d=1$) is for the handwritten
digit images of $n_1=60,000$. Each image is coded as a vector of $p_1=2784$ dimensions by
concatenating an extra 2000 dimensional vector. We have chosen randomly 2000 pairs of pixels  {\tt
x[i,j]}, {\tt x[k,l]} with $|\text{\tt i}-\text{\tt k}|\le 5$, $|\text{\tt j}-\text{\tt l}|\le 5$ in
advance, and compute the product {\tt x[i,j] * x[k,l]} for each image. The second domain ($d=2$) is
for digit labels of $n_2=10$. They are zero, one, two, ..., nine. Each label is coded as a random
vector of $p_2=100$ dimensions with each element generated independently from the standard normal
distribution $N(0,1)$. The third domain ($d=3$) is for attribute labels even, odd, and prime
($n_3=3$). Each label is coded as a random vector of $p_3=50$ dimensions with each element generated
independently from $N(0,1)$. The total number of vectors is $N=n_1+n_2+n_3 = 60013$, and the
dimension of the augmented vector is $P=p_1+p_2+p_3=2934$.

The true matching weight matrix $\bmb{W}$ is specified as follows. The cross-domain matching weight
$\bmb{W}^{(12)}\in \mathbb{R}^{n_1 \times n_2}$ between domain-1 and domain-2 is the 1-of-$K$ coding;
$\bar w^{(12)}_{ij}=1$ if $i$-th image has $j$-th label, and $\bar w^{(12)}_{ij}=0$ otherwise.
$\bmb{W}^{(13)}\in \mathbb{R}^{n_1 \times n_3}$ is defined similarly, but images may have two
attribute labels such as an image of 3 has labels odd and prime. We set all elements of
$\bmb{W}^{(23)}$ as zeros, pretending ignorance about the number properties.
 We then prepared $\bm{W}$ by randomly sampling elements from $\bmb{W}$ with
$\epsilon=0.2$. Only the upper triangular parts of the weight matrices are stored in memory, so that
the symmetry $\bm{W}=\bm{W}^T$ is automatically hold. The number of nonzero elements in the upper
triangular part of the matrix is 143775 for  $\bmb{W}$, and it becomes 28779 for $\bm{W}$. In
particular, the number of nonzero elements in $\bm{W}^{(12)}$ is 12057.

The optimal $\bm{A}$ is computed by the method of Section~\ref{sec:regularization}. The
regularization matrix is block diagonal $\bm{L}_M = \Diag( \bm{L}_M^{(1)}, \bm{L}_M^{(2)},
\bm{L}_M^{(3)} )$ with $\bm{L}_M^{(d)} = \alpha_d \bm{I}_{p_d}$ and $\alpha_d = \trace(
(\bm{X}^{(d)} )^T \bm{M}^{(d)} \bm{X}^{(d)} )/p_d$. The regularization parameters are $\gamma_M>0$
and $\gamma_W=0$. The computation with $\gamma_M=0$ actually uses a small value $\gamma_M=10^{-6}$.
The number of positive eigenvalues is $K^+=11$. The appropriate value $K=9$ was chosen by looking at
the distribution of eigenvalues $\lambda_k$ and $\phi_k^{\rm cv}$.

The three types of matching errors of Section~\ref{sec:matchingerrors} are computed as follows. The
plotted values are not the component-wise matching errors $\phi_k$, but the sum $\phi = \sum_{k=1}^K
\phi_k$. The true matching error $\phi_k^{\rm true}(\bm{W}, \bmb{W})$ is computed with $\bmb{W}$ of
the test dataset here, while $\bm{A}$ is computed from the training dataset. This is different from
the definition in Section~\ref{sec:fiterror} but more appropriate if test datasets are available.
The fitting error $\phi_k^{\rm fit}(\bm{W})$ and the cross-validation error $\phi_k^{\rm
cv}(\bm{W})$ are computed from the training dataset. In particular, $\phi_k^{\rm cv}(\bm{W})$ is
computed by resampling elements from $\bm{W}$ with $\kappa=0.1$ so that the number of nonzero
elements in the upper triangular part of $\bm{W}^*$ is about 3000.

\subsection{Simulation datasets} \label{sec:simulation-data}

We generated simulation datasets of cross-domain matching with $D=3$, $p_1=10, p_2=30, p_3=100$,
$n_1=125, n_2=250, n_3=500$. The two true matching weights $\bmb{W}_A$ and $\bmb{W}_B$ were created
at first, and they were unchanged during the experiments. In each of the twelve experiments,
$\bm{X}$ and $\bm{W}$ are generated from either of $\bmb{W}_A$ and $\bmb{W}_B$, and then $\bm{W}$ is
generated independently 160 times, while $\bm{X}$ is fixed, for taking the simulation average.
Computation of the optimal $\bm{A}$ is the same as Appendix~\ref{sec:mnist} using the same
regularization term specified there.

The data matrix $\bm{X}$ is specified as follows. First, 25 points on $5 \times 5$ grid in
$\mathbb{R}^2$ are placed as $(1,1), (1,2), (1,3),  (1,4), (1,5), (2,1),\ldots,(5,5)$.  They are
$(\bm{x}^{(0)}_1)^T, \ldots, (\bm{x}^{(0)}_{25})^T$, where $d=0$ is treated as a special domain for
data generation. Matrices $\bm{B}^{(d)} \in \mathbb{R}^{p_d \times 2}$, $d=1,2,3$, are prepared with
all elements distributed as $N(0,1)$ independently. Let $n_{d,i}$ be the number of vectors in
domain-$d$ generated from $i$-th grid point for $d=1,2,3$, $i=1,\ldots,25$, which will be specified
later with constraints $\sum_{i=1}^{25} n_{d,i}=n_d$. Then data vectors $\{
\bm{x}_i^{(d)}, i=1,\ldots, n_d \}$ are generated as $\bm{x}_{i,j}^{(d)}=\bm{B}^{(d)} \bm{x}^{(0)}_i
+ \bm{\epsilon}^{(d)}_{i,j}$, $i=1,\ldots,25$, $j=1,\ldots, n_{d,i}$, with elements of
  $\bm{\epsilon}^{(d)}_{i,j}$ distributed as $N(0,0.5^2)$ independently. Each column of
  $\bm{X}^{(d)}$ is then standardized to mean zero and variance one.

The true matching weight matrix $\bmb{W}$ is specified as follows. Two data vectors in different
domains $d$ and $e$ are linked to each other as $\bar w_{ij}^{(de)}=1$ if they are generated from
the same grid point. All other elements in $\bmb{W}$ are zero. Two types of $\bmb{W}$, denoted as
$\bmb{W}_A$ and $\bmb{W}_B$, are considered. For $\bmb{W}_A$, the numbers of data vectors are the
same for all grid points; $n_{1,i}=5$, $n_{2,i} = 10$, $n_{3,i}=20$, $i=1,\ldots, 25$. For
$\bmb{W}_B$, the numbers of data vectors are randomly generated from the power-law with probability
proportional to $n^{-3}$. The largest numbers are $n_{1,17}=26$, $n_{2,9}=49$, $n_{3,23}=349$. The
number of nonzero elements in the upper triangular part of the matrix is 8750 for $\bmb{W}_A$ (1250,
2500, 5000, respectively, for  $\bmb{W}^{(12)}_A$, $\bmb{W}^{(13)}_A$, $\bmb{W}^{(23)}_A$  ), and it
is 6659 for $\bmb{W}_B$  (906, 1096, 4657, respectively, for  $\bmb{W}^{(12)}_B$,
$\bmb{W}^{(13)}_B$, $\bmb{W}^{(23)}_B$  ).

For generating $\bm{W}$ from $\bmb{W}$, two sampling schemes, namely, the link sampling and the node
sampling, are considered with three parameter settings for each. For the link sampling, $w_{ij}$ are
sampled independently with probability $\epsilon=0.02, 0.04, 0.08$. For the node sampling, 
vectors are sampled independently with probability $\xi=\sqrt{0.02}\approx 0.14, \sqrt{0.04}=0.2,
 \sqrt{0.08} \approx 0.28 $. Then $w_{ij}$ is sampled when both vectors $\bm{x}_i$ and $\bm{x}_j$
 are sampled simultaneously.

\section{Technical details} \label{sec:technical}

\subsection{Proof of Lemma~\ref{lemma:lambdachange}} \label{sec:proof-lambdachange}

The following argument on small change in eigenvalues and eigenvectors is an adaptation of
\cite{van2007computation} to our setting. The two equations in (\ref{eq:agaaha}) are $(\bm{I}_P +
\bm{C})^T \bmh{A}^T (\bmh{G} + \Delta\bm{G}) \bmh{A} (\bm{I}_P + \bm{C}) - \bm{I}_P = \bm{0}$ and
$(\bm{I}_P + \bm{C})^T \bmh{A}^T (\bmh{H} + \Delta\bm{H}) \bmh{A} (\bm{I}_P + \bm{C}) -
\bmh{\Lambda} - \Delta \bm{\Lambda} = \bm{0}$. They are expanded as \begin{gather} (\bm{C}^T +
\bm{C} + \bm{g}) +  (\bm{C}^T \bm{C} + \bm{C}^T \bm{g} + \bm{g} \bm{C}) = O(\ga^3 P^2),
\label{eq:agaic} \\ (\bm{C}^T \bmh{\Lambda} + \bmh{\Lambda} \bm{C} + \bm{h}  - \Delta \bm{\Lambda} )
+  (\bm{C}^T \bmh{\Lambda} \bm{C} + \bm{C}^T \bm{h} + \bm{h} \bm{C})  = O(\ga^3 P^2),
\label{eq:ahaic}  \end{gather} where the first part is $O(\ga)$ and the second part is $O(\ga^2 P)$
on the left hand side of each equation.

First, we solve the $O(\ga)$ parts of (\ref{eq:agaic}) and (\ref{eq:ahaic}) by ignoring $O(\ga^2 P)$
terms. $O(\ga)$ part in (\ref{eq:agaic}) is $\bm{C}^T + \bm{C} + \bm{g} = O(\ga^2 P)$, and we get
(\ref{eq:cii}) by looking at the $(i,i)$ elements $c_{ii} + c_{ii} + g_{ii} = O(\ga^2 P)$. Then
substituting $\bm{C}^T = - \bm{C} - \bm{g} + O(\ga^2 P)$ into (\ref{eq:ahaic}), we have
\begin{equation} \label{eq:ahaggp}
\bmh{\Lambda} \bm{C} - \bm{C} \bmh{\Lambda} -\bm{g} \bmh{\Lambda}
 + \bm{h} - \Delta \bm{\Lambda} = O(\ga^2 P). 
\end{equation}
Looking at $(i,j)$ elements $\hlambda_i c_{ij} - c_{ij} \hlambda_j - g_{ij} \hlambda_j + h_{ij} =
O(\ga^2 P)$ in (\ref{eq:ahaggp}) for $i \neq j$, and noticing $(\hlambda_i - \hlambda_j)\inv=O(1)$ from (A\ref{asm:eigen}), we
get (\ref{eq:cij}). We also have (\ref{eq:Dlambdai}) by looking at $(i,i)$ elements $\hlambda_i
c_{ii} - c_{ii} \hlambda_i - g_{ii} \hlambda_i + h_{ii} - \Delta \lambda_i = O(\ga^2 P)$ in
(\ref{eq:ahaggp}).

Next, we solve $O(\ga^2 P)$ parts of (\ref{eq:agaic}) and (\ref{eq:ahaic})  by ignoring $O(\ga^3
P^2)$ terms. For extracting $O(\ga^2 P)$ parts from the equations, we simply replace $\bm{C}$ with
$\delta\bm{C}$, $\Delta \bm{\Lambda}$ with $\delta \bm{\Lambda}$,  $\bm{g}$ with $\bm{0}$, and
$\bm{h}$ with $\bm{0}$ in the $O(\ga)$ parts. By substituting $\bm{C}^T + \bm{C} + \bm{g} = \delta \bm{C}^T +  \delta \bm{C} + O(\ga^3 P^2)$
into (\ref{eq:agaic}), we get
\begin{equation}\label{eq:agacc} 
	\delta \bm{C}^T +  \delta \bm{C} + 
\bm{C}^T \bm{C} + \bm{C}^T \bm{g} + \bm{g} \bm{C} = O(\ga^3 P^2),
\end{equation}
and the $(i,i)$ elements give
\begin{equation} \label{eq:dcii2}
	\delta c_{ii} = -\frac{1}{2} (\bm{C}^T \bm{C})_{ii} - (\bm{g} \bm{C})_{ii} + O(\ga^3 P^2).
\end{equation}
By substituting $\bm{C}^T
\bmh{\Lambda} + \bmh{\Lambda} \bm{C} + \bm{h} - \Delta \bm{\Lambda}= \delta \bm{C}^T
\bmh{\Lambda} +  \bmh{\Lambda} \delta \bm{C}  - \delta \bm{\Lambda} + O(\ga^3 P^2)$ into
 (\ref{eq:ahaic}), we get
\begin{equation}\label{eq:ahacc} 
	\delta \bm{C}^T \bmh{\Lambda} +  \bmh{\Lambda} \delta \bm{C} - \delta \bm{\Lambda} + 
\bm{C}^T \bmh{\Lambda}  \bm{C} + \bm{C}^T \bm{h} + \bm{h} \bm{C} 
 = O(\ga^3 P^2).
\end{equation}
Rewriting (\ref{eq:agacc}) as $\delta \bm{C}^T = -(  \delta \bm{C} +  \bm{C}^T \bm{C} + \bm{C}^T
\bm{g} + \bm{g} \bm{C}) = O(\ga^3 P^2)$, we substitute it into
 (\ref{eq:ahacc}) to have
\begin{equation} \label{eq:ahagggp}
	-\delta \bm{C} \bmh{\Lambda} + \bmh{\Lambda} \delta \bm{C}
	-\bm{g} \bm{C} \bmh{\Lambda} + \bm{h} \bm{C} 
	-\delta \bm{\Lambda} + \bm{C}^T \Delta \bm{\Lambda} = O(\ga^3 P^2),
\end{equation}
where (\ref{eq:ahaggp}) is used for simplifying the expression.
Then we get
\begin{equation} \label{eq:dcij2}
\delta c_{ij}  = (\hlambda_i-\hlambda_j )\inv
\Bigl( 
 (\bm{g} \bm{C})_{ij} \hlambda_j - (\bm{h} \bm{C})_{ij} - c_{ji} \Delta \lambda_j
\Bigr) +  O(\ga^3 P^2)
\end{equation}
by looking at $(i,j)$ elements ($i\neq j$) of (\ref{eq:ahagggp}).
Also we get
\begin{equation}\label{eq:dlambdai2}
	\delta \lambda_i = -(\bm{g}\bm{C})_{ii} \hlambda_i + (\bm{h} \bm{C})_{ii} + c_{ii} \Delta \lambda_i + O(\ga^3 P^2)
\end{equation}
by looking at $(i,i)$ elements of (\ref{eq:ahagggp}).

Finally, the remaining terms in (\ref{eq:dcii2}), (\ref{eq:dcij2}), (\ref{eq:dlambdai2}) will be
written by using (\ref{eq:Dlambdai}), (\ref{eq:cii}), (\ref{eq:cij}). For any $(i,j)$ with $i\le P, j\le J$,
$(\bm{g}\bm{C})_{ij} = -\frac{1}{2} g_{ij} g_{jj}  + \sum_{k\neq j} (\hlambda_k - \hlambda_j)\inv
(g_{kj}\hlambda_j - h_{kj}) g_{ki}  + O(\ga^3 P^2)$ and $(\bm{h}\bm{C})_{ij} = -\frac{1}{2} h_{ij}
g_{jj}  + \sum_{k\neq j} (\hlambda_k - \hlambda_j)\inv (g_{kj}\hlambda_j - h_{kj}) h_{ki}  + O(\ga^3
P^2)$. For $(i,j)$ with $i\neq j$ and one of $i$ and $j$ being not greater than $J$,
$c_{ji}\Delta\lambda_j = - (\hlambda_j - \hlambda_i)\inv
(g_{ji}\hlambda_i - h_{ji}) (g_{jj}\hlambda_j - h_{jj}) + O(\ga^3 P^2) $. For $i\le J$, $(\bm{C}^T
\bm{C})_{ii} = \frac{1}{4} (g_{ii})^2 +  \sum_{j \neq i} (\hlambda_j - \hlambda_i)^{-2} ( g_{ji}
\hlambda_i - h_{ji})^2 + O(\ga^3 P^2)$. For $i\le P$, $c_{ii} \Delta\lambda_i= (1/2) g_{ii} (g_{ii}\hlambda_i -
h_{ii})+O(\ga^3 P^2)$. Using these expressions, we get   (\ref{eq:dlambdai}), (\ref{eq:dcii}), and
(\ref{eq:dcij}).

\subsection{Proof of Lemma~\ref{lemma:errorchange}} \label{sec:proof-errorchange}
Let us denote
$\bm{C}=(\bm{c}^1,\ldots,\bm{c}^P)$ and $\bm{I}_P = (\bm{\delta}_1,\ldots,\bm{\delta}_P)$, where the
elements are $\bm{c}^k = (c_{1k},\ldots,c_{Pk})^T$ and $\bm{\delta}_k =
(\delta_{1k},\ldots,\delta_{Pk})^T$. Then  $\bm{a}^k = \bmh{A}(\bm{\delta}_k + \bm{c}^k)$ and
$\bm{y}^k = b_k \bm{X} \bm{a}^k = b_k \bm{X} \bmh{A}(\bm{\delta}_k + \bm{c}^k)$.
Noticing $\phi_k(\bm{W}, \bmt{W}) = (1/2)\sum_{i=1}^N \sum_{j=1}^N \tilde w_{ij} (y_{ik} -
y_{jk})^2=(\bm{y}^k)^T (\bmt{M} - \bmt{W} ) \bm{y}^k$, and substituting $\bm{y}^k =  b_k \bm{X}
\bmh{A}(\bm{\delta}_k + \bm{c}^k)$ into it, we have $\phi_k(\bm{W}, \bmt{W}) =b_k^2 (\bm{\delta}_k +
\bm{c}^k)^T  \bmt{S} (\bm{\delta}_k + \bm{c}^k) =b_k^2 (\stilde_{kk} + 2\sum_i \stilde_{ik} c_{ik} +
\sum_i \sum_j \stilde_{ij} c_{ik} c_{jk}) $. Similarly, we have $b_k^2=((\bm{a}^k)^T \bm{X}^T \bm{M}
\bm{X} \bm{a}^k)\inv = ((\bm{\delta}_k + \bm{c}^k)^T  (\bm{\delta}_k +
\bm{c}^k))\inv=(1+2c_{kk}+(\bm{C}^T\bm{C})_{kk})\inv = 1 - 2c_{kk} - (\bm{C}^T\bm{C})_{kk} +
4(c_{kk})^2 + O(\ga^3 P)$. Substituting it into $\phi_k(\bm{W}, \bmt{W})$, we have
\begin{align*}
\phi_k(\bm{W}, \bmt{W})
= & \stilde_{kk} + 2\sum_i \stilde_{ik} c_{ik} + \sum_i \sum_j \stilde_{ij} c_{ik} c_{jk} \\
  & - 2c_{kk} \stilde_{kk} - 4 c_{kk} \sum_i \stilde_{ik} c_{ik}
-(\bm{C}^T\bm{C})_{kk} \stilde_{kk} + 4(c_{kk})^2 \stilde_{kk} + O(\ga^3 P^2)\\
= & \stilde_{kk} (1 + (c_{kk})^2 - (\bm{C}^T\bm{C})_{kk} )\\
&+ 2\sum_{i\neq k} \stilde_{ik} c_{ik}(1 - c_{kk})
 + \sum_{i\neq k} \sum_{j\neq k} \stilde_{ij} c_{ik} c_{jk} + O(\ga^3 P^2),
\end{align*}
which gives (\ref{eq:phikwtw}) after rearranging the formula using the results of
Lemma~\ref{lemma:lambdachange}. In particular, the last term $\sum_{i\neq k} \sum_{j\neq k}
\stilde_{ij} c_{ik} c_{jk} =  \sum_{i\neq k} \sum_{j\neq k} \stilde_{ij}  (\hlambda_i -
\hlambda_k)\inv(\hlambda_j - \hlambda_k)\inv   (g_{ik} \hlambda_k - h_{ik})(g_{jk} \hlambda_k
-h_{jk}) + O(\ga^3 P^3)$ leads to the asymptotic error $O(\ga^3 P^3)$ of  (\ref{eq:phikwtw}).

\subsection{Proof of Lemma~\ref{lemma:exptrue}} \label{sec:proof-exptrue}

First note that $\Delta \hat w_{lm} = w_{lm} - \epsilon \bar w_{lm} = (z_{lm} - \epsilon) \bar
w_{lm}$ from (\ref{eq:zij}), and so $E(\Delta \hat w_{lm} )=0$ and 
\begin{equation}\label{eq:covdhatwlm}
	E(\Delta \hat w_{lm} \Delta \hat
w_{l'm'})=\delta_{ll'}  \delta_{mm'} \epsilon(1- \epsilon) \bar w_{lm}^2.
\end{equation}
From (\ref{eq:gyhy}) and the definition of ${\rm bias}_k$, we have
\[
\begin{split}
{\rm bias}_k = E\Bigl[
& \sum_{l\ge m} \sum_{l'\ge m'} \Delta \hat w_{lm} \Delta \hat w_{l'm'} \Bigl\{
   - ( {\cal G}[\bmb{Y}]_{lm}^{kk}  - {\cal H}[\bmb{Y}]_{lm}^{kk}) {\cal G}[\bmb{Y}]_{l'm'}^{kk} \\
  &+\sum_{j\neq k} 2(\blambda_j - \blambda_k)\inv
 ( {\cal G}[\bmb{Y}]_{lm}^{jk} - {\cal H}[\bmb{Y}]_{lm}^{jk})
 ( {\cal G}[\bmb{Y}]_{l'm'}^{jk} \blambda_k   - {\cal H}[\bmb{Y}]_{l'm'}^{jk})
 \Bigr\}
  \Bigr],
\end{split}
\]
and thus we get (\ref{eq:biask}) by (\ref{eq:covdhatwlm}).
Both $\hat g_{ij}$ and $\hat h_{ij}$ are of the form $\sum_{l\ge m} f_{lm} \Delta \hat w_{lm}$ with
$f_{lm}=O(N\ohh)$ in  (\ref{eq:gyhy}), where the number of nonzero terms is $O(\epsilon\inv N)$ in
the summation $\sum_{l \ge m}$. It then follows from (\ref{eq:covdhatwlm}) that $V(\sum_{l\ge m}
f_{lm} \Delta \hat w_{lm}) = \sum_{l\ge m} f_{lm}^2 V(\Delta \hat w_{lm}) = \epsilon (1 - \epsilon)
\sum_{l\ge m} f_{lm}^2 \bar w_{lm}^2 = O(\epsilon  N^{-2} \epsilon\inv N) = O(N\inv)$. Therefore
$\sum_{l\ge m} f_{lm} \Delta \hat w_{lm} = O(N\oh)$, showing $\bmh{g}=O(N\oh)$ and
$\bmh{h}=O(N\oh)$.

In order to show (\ref{eq:expphikfitbias}), we prepare  $\bmh{C} = (\hat c_{ij})$ and $\Delta
\bmh{\Lambda} = \diag(\Delta \hlambda_1,\ldots \Delta \hlambda_P)$ with
\[
\bmh{A} = \bmb{A} ( \bm{I}_P + \bmh{C}),\quad
\bmh{\Lambda} = \bmb{\Lambda} + \Delta \bmh{\Lambda}
\]
for describing change from $\bmb{A}$ to $\bmh{A}$.   The elements are given by
Lemma~\ref{lemma:lambdachange} with $\gamma=N\oh$. In particular (\ref{eq:Dlambdai}), (\ref{eq:cii})
and  (\ref{eq:cij}) become
\begin{equation} \label{eq:hatcij}
\begin{split}
\Delta\hlambda_i & = -(\hat g_{ii} \blambda_i  - \hat h_{ii} )  + O(N\ohh P),\\
\hat c_{ii} & = -\frac{1}{2} \hat g_{ii} +  O(N\ohh P),\\
\hat c_{ij} & = (\blambda_i - \blambda_j)\inv (\hat g_{ij} \hlambda_j - \hat h_{ij})+  O(N\ohh P).
\end{split}
\end{equation}
Note that the roles of  $\bm{W}$, $\bm{g}$ and $\bm{h}$ in Lemma~\ref{lemma:lambdachange} are now
played by $\epsilon\bmb{W}$,  $\bmh{g}$ and $\bmh{h}$, respectively, and therefore the expressions
of $\bm{C}$ and $\Delta \bm{\Lambda}$ in Lemma~\ref{lemma:lambdachange}  give those of $\bmh{C}$ and
$\Delta \bmh{\Lambda}$ above. 

Let us define
\[
	\Delta\bmh{S} = \bmh{A}^T \bm{X}^T (\Delta\bmh{M} - \Delta\bmh{W}) \bm{X} \bmh{A}.
\]
Then the difference of the fitting error from the true error, namely (\ref{eq:phikwdw}),
is expressed asymptotically by (\ref{eq:phikwtw}) of Lemma~\ref{lemma:errorchange}
with $\bmt{S}=\Delta\bmh{S}$.
Substituting $\bmh{A} = \bmb{A} ( \bm{I}_P + \bmh{C})$ into $\Delta\bmh{S}$, we get
\[
\begin{split}
\Delta\bmh{S}  &=  ( \bm{I}_P + \bmh{C})^T(\bmh{g} - \bmh{h})  ( \bm{I}_P + \bmh{C})\\
&=\bmh{g} - \bmh{h} + (\bmh{g} - \bmh{h})\bmh{C} + \bmh{C}^T(\bmh{g} - \bmh{h}) +
O(N\ohhh P^2)\\
&=\bmh{g} - \bmh{h} + O(N\ohh P),
\end{split}
\]
where $\Delta\bmh{S} = O(N\oh)$ but $E(\Delta\bmh{S}) = O(N\ohh P)$, since $E(\bmh{g})= E(\bmh{h}) =
\bm{0}$.

We now attempt to rewrite terms in (\ref{eq:phikwtw}) using the relation $\bm{W} = \epsilon\bmb{W} +
\Delta\bmh{W}$. Define $\bmb{g} = O(\gamma)$, $\bmb{h} = O(\gamma)$ by \[ \bmb{g} = \bmb{A}^T
\Delta\bm{G} \bmb{A},\quad      \bmb{h} = \bmb{A}^T \Delta\bm{H} \bmb{A}. \] Then
$\bm{g}=(\bm{I}+\bmh{C})^T \bmb{g} (\bm{I}+\bmh{C}) =  \bmb{g} + O(N\oh \ga P)$ and
$\bm{h}=(\bm{I}+\bmh{C})^T \bmb{h} (\bm{I}+\bmh{C}) =  \bmb{h} + O(N\oh \ga P)$. We also have
$(\hlambda_i - \hlambda_k)\inv =(\blambda_i - \blambda_k)\inv + O(N\oh)$, since $\hlambda_i =
\blambda_i + O(N\oh)$.  We thus have $ g_{ik} \hlambda_k - h_{ik} = \bar g_{ik} \blambda_k - \bar
h_{ik} +  O(N\oh \ga P)$.
Therefore, (\ref{eq:phikwtw}) with $\bmt{S}=\Delta\bmh{S}$ is rewritten as
\begin{equation*}
\begin{split}
	\phi_k(\bm{W}, & \Delta\bmh{W}) = 
	\Delta \hat s_{kk} + \sum_{i\neq k}
	2(\hlambda_i - \hlambda_k)\inv \Delta \hat s_{ik} (g_{ik} \hlambda_k -h_{ik})
	+O(N\oh \ga^2 P^2 + \ga^3 P^3)\\
	&=
	\Delta \hat s_{kk} + \sum_{i\neq k}
	2(\blambda_i - \blambda_k)\inv  (\hat g_{ik} - \hat h_{ik}) (\bar g_{ik} \blambda_k - \bar h_{ik})
	+O(N\ohh \ga P^2+ \ga^3 P^3).
\end{split}
\end{equation*}
By noting $E(\hat g_{ik} - \hat h_{ik})=0$, we get
\begin{equation} \label{eq:expphikwdw}
	E(\phi_k(\bm{W},  \Delta\bmh{W})) = E(\Delta \hat s_{kk} ) + O(N\ohh \ga P^2 + \ga^3 P^3).
\end{equation}
For calculating $E(\Delta \hat s_{kk} ) $, we substitute (\ref{eq:hatcij}) into  $\Delta\hat s_{kk}
= \hat g_{kk} - \hat h_{kk} +2((\bmh{g}-\bmh{h}) \bmh{C})_{kk}  + O(N\ohhh P^2)$. Then we have
\[
\begin{split}
\Delta\hat s_{kk} =\hat g_{kk} - \hat h_{kk}
&-(\hat g_{kk} - \hat h_{hh}) \hat g_{kk} \\
&+ \sum_{j \neq k} (\hat g_{jk} - \hat h_{jk})
(\blambda_j - \blambda_k)\inv (\hat g_{jk} \blambda_k - \hat h_{jk}) 
+O(N\ohhh P^2),
\end{split}
\]
and therefore, by noting $E(\hat g_{kk} - \hat h_{kk})=0$, we obtain
\[
	E(\Delta\hat s_{kk} ) = {\rm bias}_k + O(N\ohhh P^2).
\]
Combining it with (\ref{eq:phikwdw}) and (\ref{eq:expphikwdw}),
and also noting $O(N\ohhh P^2 + N\ohh \ga P^2) = O(N\ohhh P^2 )$, we finally get
(\ref{eq:expphikfitbias}).

\subsection{Proof of Lemma~\ref{lemma:expcv}} \label{sec:proof-expcv}

For deriving an asymptotic expansion of  $\phi_k((1- \kappa)\inv(\bm{W}-\bm{W^*}),
\kappa\inv\bm{W}^* ) $ using Lemma~\ref{lemma:errorchange},  we replace $\bm{W}$ and $\bmt{W}$ in
(\ref{eq:phikwtw}), respectively,  by  $(1- \kappa)\inv(\bm{W}-\bm{W}^*)$ and $\kappa\inv \bm{W}^*$.
We define $\bmh{G}^*$, $\bmh{H}^*$, $\bmh{A}^*$, and $\bmh{\Lambda}^*$; they correspond to those with
hat but $\bm{W}$ is replaced by $(1- \kappa)\inv(\bm{W}-\bm{W}^*)$.  The key equations are $\bmh{G}^*
= (1 - \kappa)\inv \bm{X}^T (\bm{M}- \bm{M}^*) \bm{X}$, $\bmh{H}^* = (1 - \kappa)\inv \bm{X}^T
(\bm{W}- \bm{W}^*) \bm{X}$, $\bmh{A}^{*T}  \bmh{G}^* \bmh{A}^* = \bm{I}_P$, and $\bmh{A}^{*T} \bmh{H}^*
\bmh{A}^* = \bmh{\Lambda}^*$. The regularization terms are now represented by $\bm{g}^* = \bmh{A}^{*T}
\Delta\bm{G} \bmh{A}^*$ and $\bm{h}^* = \bmh{A}^{*T} \Delta\bm{H} \bmh{A}^*$. For $\bmt{W}=\kappa\inv
\bm{W}^*$, we put \[ \bm{S}^* = \kappa\inv\bmh{A}^{*T} \bm{X}^T (\bm{M}^* - \bm{W}^*) \bm{X}
\bmh{A}^*. \]   Then $\bm{g}$, $\bm{h}$, $\bmh{\Lambda}$ and $\bmt{S}$, respectively, in
Lemma~\ref{lemma:errorchange} are replaced by $\bm{g}^*$, $\bm{h}^*$, $\bmh{\Lambda}^*$ and
$\bm{S}^*$. Noticing $\bm{g}^*=O(\ga)$, $\bm{h}^*=O(\ga)$ and $\bm{S}^*=O(1)$, the asymptotic orders
of the terms in Lemma~\ref{lemma:errorchange} remain the same. (\ref{eq:phikwtw}) is now written
like
\begin{equation}\label{eq:phikcv1}
  \phi_k ((1- \kappa)\inv(\bm{W}-\bm{W^*}), \kappa\inv\bm{W}^* )
=s_{kk}^* + \sum_{i\neq k} 2(\hlambda_i^* - \hlambda_k^*)\inv s_{ik}^* (g_{ik}^* \hlambda_k^* - h_{ik}^*)
+\cdots,
\end{equation}
where terms are omitted for saving the space but all the terms in (\ref{eq:phikwtw}) will be calculated below.

In order to take $E^*(\cdot | \bm{W})$ of (\ref{eq:phikcv1}) later, we define
\[
	\Delta\bmh{G}^* = \bmh{G}^* - \bmh{G},\quad
	\Delta\bmh{H}^* = \bmh{H}^* - \bmh{H},\quad
	\bmh{g}^* = \bmh{A}^T \Delta\bmh{G}^*  \bmh{A},\quad
	\bmh{h}^* = \bmh{A}^T \Delta\bmh{H}^*  \bmh{A}
\]
for describing change from $\bmh{A}$ to $\bmh{A}^*$. Also define
\[
  \bmh{S}^* = \kappa\inv\bmh{A}^T \bm{X}^T (\bm{M}^* - \bm{W}^*) \bm{X} \bmh{A}
\]
for $\bmt{W}=\kappa\inv \bm{W}^*$. They are expressed in terms of
\[
\Delta\bmh{W}^* = \bm{W}^* - \kappa \bm{W},\quad
\Delta\bmh{M}^* = \bm{M}^* - \kappa \bm{M}
\]
as $ \Delta\bmh{G}^* = - (1-\kappa)\inv \bm{X}^T \Delta\bmh{M}^* \bm{X}$, $\Delta\bmh{H}^* = -
(1-\kappa)\inv \bm{X}^T \Delta\bmh{W}^* \bm{X}$, $ \bmh{g}^* = - (1-\kappa)\inv \bmh{Y}^T
\Delta\bmh{M}^* \bmh{Y}$, $\bmh{h}^* = - (1-\kappa)\inv \bmh{Y}^T \Delta\bmh{W}^* \bmh{Y}$, and $
\bmh{S}^* = \kappa\inv\bmh{Y}^{T} (\Delta\bmh{M}^* - \Delta\bmh{W}^*) \bmh{Y} + \bm{I}_P -
\bmh{\Lambda}$.
The elements of $ \bmh{g}^*$, $ \bmh{h}^*$ and $\bmh{S}^*$ are expressed using the notation of Section~\ref{sec:experror} as
\begin{equation}\label{eq:gijhatstar}
\begin{split}
\hat g_{ij}^* = & - (1-\kappa)\inv\sum_{l \ge m} {\cal G}[\bmh{Y}]_{lm}^{ij} \Delta \hat w_{lm}^*,\quad
  \hat h_{ij}^* = - (1-\kappa)\inv\sum_{l \ge m} {\cal H}[\bmh{Y}]_{lm}^{ij} \Delta \hat w_{lm}^*\\
  \hat s_{ij}^* = &  \kappa\inv\sum_{l \ge m} ({\cal G}[\bmh{Y}]_{lm}^{ij}- {\cal H}[\bmh{Y}]_{lm}^{ij}) \Delta \hat w_{lm}^*
  +\delta_{ij}(1 - \hlambda_i).
\end{split}
\end{equation}
It follows from the argument below that $\bmh{g}^*=O(N\ohh)$, $\bmh{h}^*=O(N\ohh)$, and $\bmh{S}^* = O(1)$.
Note that $\Delta \hat w_{lm}^* = w_{lm}^* - \kappa  w_{lm} = (z_{lm}^* - \kappa) w_{lm}$ from (\ref{eq:zijstar}), and so
$ E^*(\Delta \hat w_{lm}^* |\bm{W})=0$ and
\begin{equation}\label{eq:expvardhwlmstar}
  E^*(\Delta \hat w_{lm}^* \Delta \hat w_{l'm'}^* |\bm{W})=  \delta_{ll'} \delta_{mm'} \kappa (1 - \kappa)  w_{lm}^2.
\end{equation}
Both $\hat g_{ij}^*$ and $\hat h_{ij}^*$ are of the form $\sum_{l\ge m} f_{lm} \Delta \hat w_{lm}^* $ with
$f_{lm}=O(N\inv)$,  where the number of nonzero terms is $O(N)$ in the summation $\sum_{l \ge m}$.
Thus $V^*(\sum_{l\ge m} f_{lm} \Delta \hat w_{lm}^* |\bm{W}) = \sum_{l\ge m} f_{lm}^2 V^*(\Delta
\hat w_{lm}^*|\bm{W}) = \kappa  (1 - \kappa) \sum_{l\ge m} f_{lm}^2  w_{lm}^2 = O(\kappa  N^{-2}  N)
= O(\kappa N\inv)=O(N\ohhhh)$. Therefore $\sum_{l\ge m} f_{lm} \Delta \hat w_{lm}^* = O(N\ohh)$.

The change from $\bmh{A}$ to $\bmh{A}^*$ is expressed as \[   \bmh{A}^* = \bmh{A}(\bm{I}_P +
\bmh{C}^*),\quad   \bmh{\Lambda}^* = \bmh{\Lambda} + \Delta\bmh{\Lambda}^*. \]  The elements of
$\bmh{C}^* = (\hat c_{ij}^*)$ and $\Delta\bmh{\Lambda}^* =
\diag(\Delta\hlambda_1^*,\ldots,\Delta\hlambda_P^*)$ are given by Lemma~\ref{lemma:lambdachange}
with $\gamma = N\ohh$. In particular (\ref{eq:Dlambdai}), (\ref{eq:cii}) and (\ref{eq:cij}) becomes
\begin{equation} \label{eq:hatstarcij} \begin{split} \Delta\hlambda_i^* & = -(\hat
g_{ii}^* \hlambda_i - \hat h_{ii}^*) + O(N\ohhhh P),\\ \hat c_{ii}^* & = -\frac{1}{2} \hat g_{ii}^* +  O(N\ohhhh
P),\\ \hat c_{ij}^* & = (\hlambda_i - \hlambda_j)\inv (\hat g_{ij}^* \hlambda_j - \hat h_{ij}^*)+
O(N\ohhhh P). \end{split} \end{equation} Note that the roles of $\bm{g}$ and $\bm{h}$ in
Lemma~\ref{lemma:lambdachange} are now played by $\bmh{g}^*$ and $\bmh{h}^*$, and therefore
the expressions of $\bm{C}$ and $\Delta\bm{\Lambda}$ in Lemma~\ref{lemma:lambdachange} give those
of $\bmh{C}^*$ and $\Delta\bmh{\Lambda}^*$ .

Using the above results, $\bmh{\Lambda}^*$, $\bm{g}^*$, $\bm{h}^*$ and $\bm{S}^*$ in
(\ref{eq:phikcv1}) are expressed as follows. $\lambda_k^* = \hlambda_k + \Delta\hlambda_k^* =
\hlambda_k + O(N\ohh)$. $\bm{g}^* = (\bm{I}_P + \bmh{C}^*)^T \bm{g} (\bm{I}_P + \bmh{C}^*) = \bm{g}
+ \bm{g} \bmh{C}^* + \bmh{C}^{*T} \bm{g} + \bmh{C}^{*T} \bm{g} \bmh{C}^* = \bm{g} + O(N\ohh \ga P)$,
and similarly $\bm{h}^*=\bm{h}+O(N\ohh \ga P)$. $\bm{S}^* = (\bm{I}_P + \bmh{C}^*)^T \bmh{S}^*
(\bm{I}_P + \bmh{C}^*) = \bmh{S}^* + \bmh{S}^* \bmh{C}^* + \bmh{C}^{*T} \bmh{S}^* + \bmh{C}^{*T}
\bmh{S}^* \bmh{C}^* = \bmh{S}^* + O(N\ohh P)$. In particular the diagonal elements of $\bm{S}^*$ are
\begin{equation}\label{eq:skkstar}
\begin{split}
 s_{kk}^* &= \hat s_{kk}^* + 2 \sum_{j\neq k} \hat s_{jk}^* \hat c_{jk}^* + 2\hat s_{kk}^* \hat c_{kk}^* + O(N\ohhhh P^2)\\
 &=\hat s_{kk}^* + 2 \sum_{j\neq k} (\hlambda_j - \hlambda_k)\inv \hat s_{jk}^* (\hat g_{jk}^* \hlambda_k - \hat h_{jk}^*)
 -\hat s_{kk}^* \hat g_{kk}^* + O(N\ohhhh P^2).
 \end{split} 
\end{equation} 
Therefore, (\ref{eq:phikcv1}) is now expressed as
\begin{equation*}
\begin{split}
  \phi_k ((1- \kappa)\inv  &  (\bm{W}-\bm{W^*}), \kappa\inv\bm{W}^* ) = s_{kk}^* +
    \sum_{i\neq k} 2(\hlambda_i - \hlambda_k)\inv \hat s_{ik}^* (g_{ik} \hlambda_k - h_{ik})\\
& - \sum_{i\neq k} (\hlambda_i - \hlambda_k)^{-2} \Bigl\{
\hat s_{kk}^* (g_{ik} \hlambda_k - h_{ik})^2
+\hat s_{ik}^* (g_{ki}\hlambda_i - h_{ki})
(g_{kk} \hlambda_k - h_{kk})\Bigr\}\\
&+\sum_{i\neq k} \sum_{j\neq k} (\hlambda_i - \hlambda_k)\inv(\hlambda_j - \hlambda_k)\inv \Bigr\{
 2\hat s_{ik}^* (g_{ij} \hlambda_k - h_{ij})(g_{jk} \hlambda_k -h_{jk}) \\
& \hspace{5em} +\hat s_{ij}^* (g_{ik}\hlambda_k - h_{ik})(g_{jk} \hlambda_k - h_{jk})\Bigr\} + O(N\ohh\ga P^2 + \ga^3 P^3).
\end{split}
\end{equation*}
We take $E^*(\cdot|\bm{W})$ of the above formula. Noting $E^*(\hat s_{ij}^* | \bm{W}) = \delta_{ij} (1-\hlambda_i)$, we have
\begin{equation*}
  \phi_k^{\rm cv}(\bm{W}) = E^*(s_{kk}^* |\bm{W}) 
 - \sum_{i\neq k} (\hlambda_i - \hlambda_k)\inv
 (g_{ik} \hlambda_k - h_{ik})^2
 + O(N\ohh\ga P^2 + \ga^3 P^3).
\end{equation*}
Comparing this with (\ref{eq:phikfit}), and using (\ref{eq:skkstar}), we get
\begin{equation*} 
\begin{split}
\phi_k^{\rm fit}(\bm{W}) -   \phi_k^{\rm cv}(\bm{W}) &= 1 - \hlambda_k - E^*(s_{kk}^*|\bm{W})
 + O(N\ohh\ga P^2 + \ga^3 P^3)\\
&=E^*(\hat s_{kk}^* \hat g_{kk}^*|\bm{W}) 
-2 \sum_{j\neq k} (\hlambda_j - \hlambda_k)\inv E^*(\hat s_{jk}^* (\hat g_{jk}^* \hlambda_k - \hat h_{jk}^*) | \bm{W})\\
&\qquad + O(N\ohhhh P^2 + N\ohh\ga P^2 + \ga^3 P^3).
\end{split}
\end{equation*}
Finally, this gives (\ref{eq:fitcvexpand}) by using (\ref{eq:gijhatstar}) and (\ref{eq:expvardhwlmstar}).

For taking $E(\cdot)$ of (\ref{eq:fitcvexpand}), we now attempt to rewrite the terms. Notice $\bmh{Y}
= \bmb{Y} ( \bm{I}_P + \bmh{C} ) = \bmb{Y} + O(N\ohh P) = O(N\oh)$, we have $\hat y_{ik} \hat y_{jl}
= \bar y_{ik} \bar y_{jl} + O(N\ohhh P) = O(N\ohh)$. Also we have  $\hlambda_i = \blambda_i +
O(N\oh)$. (\ref{eq:fitcvexpand}) is now
 $ \phi_k^{\rm fit}(\bm{W}) -  \phi_k^{\rm cv}(\bm{W}) =
 \sum_{l\ge m} w_{lm}^2 \{
 -  ( {\cal G}[\bmb{Y}]_{lm}^{kk}  - {\cal H}[\bmb{Y}]_{lm}^{kk}) {\cal G}[\bmb{Y}]_{lm}^{kk} 
+\sum_{j\neq k} 2(\blambda_j - \blambda_k)\inv 
 ( {\cal G}[\bmb{Y}]_{lm}^{jk}  - {\cal H}[\bmb{Y}]_{lm}^{jk})
 ( {\cal G}[\bmb{Y}]_{lm}^{jk} \blambda_k  - {\cal H}[\bmb{Y}]_{lm}^{jk})
\} +O(N\ohhh P^2 + \ga^3 P^3) $. Comparing this with (\ref{eq:biask}), we obtain
(\ref{eq:expfitcv}) by using $E(w_{lm}^2) = \epsilon \bar w_{lm}^2$.

%%%%%%%%%%%%%%%%%%%%%%%%%%%%%%%%%%%%%%%%%%%%%%%%%%%%%%%%%%%%%%%%%%%%%%%%%%%%%%%%%%%%%%%%%%%%%%%%

\section*{Acknowledgments}

I would like to thank Akifumi Okuno, Kazuki Fukui and Haruhisa Nagata for helpful discussions. I
appreciate helpful comments of reviewers for improving the manuscript.

\clearpage
\bibliographystyle{imsart-nameyear}
\bibliography{stat2015}

\clearpage
%%%%%%%%%%%%%%%%%%%%%%%%%%%%%%%%%%%%%%%%%%%%%%%%%%%%%%%%%%%%%%%%%%%%%%%%%%%%%
\section*{Tables and Figures}
\begin{itemize}
	\item Table~\ref{tab:notations} : Section~\ref{sec:matching}
	\item Table~\ref{tab:simulation-parameters} : Section~\ref{sec:simulation}
	\item Fig.~\ref{fig:mnist-map12} : Section~\ref{sec:example}
	\item Fig.~\ref{fig:mnist-error} : Section~\ref{sec:example}
	\item Fig.~\ref{fig:weight-a} : Section~\ref{sec:simulation}
	\item Fig.~\ref{fig:weight-b} : Section~\ref{sec:simulation}
	\item Fig.~\ref{fig:exp1} : Section~\ref{sec:simulation}
	\item Fig.~\ref{fig:exp5} : Section~\ref{sec:simulation}
	\item Fig.~\ref{fig:boxplots-w} : Section~\ref{sec:simulation}
	\item Fig.~\ref{fig:boxplots-u} : Section~\ref{sec:simulation}
\end{itemize}

%%%%%%%%%%%%%%%%%%%%%%%%%%%%%%%%%%%%%%%%%%%%%%%%%%%%%%%%%%%%%%%%%%%%%%%%%%%%%

\clearpage
\def\arraystretch{1.1}

\begin{table}
\caption{Notations of mathematical symbols}
\label{tab:notations}
\centering
  \begin{tabular}{llll}
  \hline
  Symbol & Description & Section & Equation\\
  \hline
  $N$ & the number of data vectors     & \ref{sec:intro} \\
  $P$ & the dimensions of data vector  & \ref{sec:intro} \\
  $K$ & the dimensions of transformed vector  & \ref{sec:intro} \\
  $\bm{x}_i$, $\bm{X}$ & data vector, data matrix & \ref{sec:intro} \\
  $\bm{y}_i$, $\bm{Y}$ & transformed vector, transformed matrix & \ref{sec:intro} \\
  $\bm{A}$ & linear transformation matrix & \ref{sec:intro} \\
  $w_{ij}$, $\bm{W}$ & matching weights, matching weight matrix & \ref{sec:intro} \\
  $\bar{w}_{ij}$, $\bmb{W}$ & true matching weights, true matching weight matrix &
   \ref{sec:intro}& (\ref{eq:zij}) \\
  $\epsilon$ & link sampling probability & \ref{sec:intro} & (\ref{eq:zij})\\
  $m_i$, $\bm{M}$ & row sums of matching weights & \ref{sec:matching} \\
  $\bm{y}^k$ & $k$-th component of transformed matrix & \ref{sec:matching} \\
  $\phi_k$ & matching error of $k$-th component & \ref{sec:matching}, \ref{sec:fiterror} \\
  $\bmh{A} $, $\bmh{Y} $, $\bmh{G} $, $\bmh{H} $ & matrices for the optimization 
  without regularization & \ref{sec:spectral}, \ref{sec:lambdachange} \\
  $\Delta\bm{G} $, $\Delta\bm{H} $ & regularization matrices $\Delta \bm{G}=\ga_M \bm{L}_M$, $\Delta
  \bm{H}=\ga_W \bm{L}_W$ & \ref{sec:regularization} \\
  $\bm{G} $, $\bm{H} $ & matrices for the optimization with regularization 
  & \ref{sec:regularization} \\
  $\bm{u}_k $, $\lambda_k$ & eigenvector, eigenvalue & \ref{sec:regularization} \\
  $K^+$ & the number of positive eigenvalues & \ref{sec:regularization} \\
  $b_k$ & rescaling factor of $k$-th component & \ref{sec:regularization} \\
  $w_{ij}^*$, $\bm{W}^*$ & matching weights for cross-validation
  & \ref{sec:cverror} & (\ref{eq:zijstar}) \\
  $\kappa$ & link resampling probability & \ref{sec:cverror} & (\ref{eq:zijstar}) \\
  $\xi$ & node sampling probability & \ref{sec:node-sampling} & (\ref{eq:zij-node}) \\
  $\nu$ & node resampling probability & \ref{sec:node-sampling} & (\ref{eq:zijstar-node}) \\
  $\gamma$ & regularization parameter  & \ref{sec:main} & (A\ref{asm:gamma}) \\
  $J$ & maximum value of $k$ for evaluating $\phi_k$ & \ref{sec:main}, \ref{sec:technotes}
   & (A\ref{asm:eigen}) \\
  $\hlambda_i$, $\bmh{\Lambda} $ & eigenvalues without regularization  &\ref{sec:lambdachange}\\
  $\bm{g}$, $\bm{h}$ & matrices representing regularization &\ref{sec:lambdachange}\\
  $\Delta \lambda_i$, $\Delta \bm{\Lambda} $ & changes in eigenvalues due to regularization 
  &\ref{sec:lambdachange}\\
  $c_{ij}$, $\bm{C} $ & changes in $\bm{A}$ due to regularization &\ref{sec:lambdachange}\\
  $\delta \lambda_i$, $\delta c_{ij}$ & higher order terms of $\Delta \lambda_i$, $c_{ij}$ 
  &\ref{sec:lambdachange}\\
  $\tilde s_{ij}$, $\bmt{S}$ & defined from $\bmt{W}$ and $\bmh{Y}$ &\ref{sec:lambdachange}\\
  $\bmb{A} $, $\bmb{Y} $, $\bmb{G} $, $\bmb{H} $ & matrices for the optimization with respect to
   $\epsilon \bmb{W}$ & \ref{sec:experror} \\
  $\blambda_i$, $\bmb{\Lambda} $ & eigenvalues with respect to $\epsilon \bmb{W}$ 
  & \ref{sec:experror} \\
  $\Delta \hat w_{ij}$, $\Delta \bmh{W} $, $\Delta \bmh{M}$ 
  & $\Delta \bmh{W}   = \bm{W} - \epsilon\bmb{W}$,
  $\Delta \bmh{M}   = \bm{M} - \epsilon\bmb{M}$ & \ref{sec:experror} \\
  $\Delta\bmh{G} $, $\Delta\bmh{H} $, $\bmh{g} $, $\bmh{h} $ 
  & matrices representing  the change $\Delta \bmh{W}$ & \ref{sec:experror} \\
  ${\cal G}[\bmb{Y}]_{lm}^{ij}$, ${\cal H}[\bmb{Y}]_{lm}^{ij}$ & coefficients for representing 
  $\hat g_{ij}$, $\hat h_{ij}$ & \ref{sec:experror} & (\ref{eq:gyhy}) \\
  $\hat c_{ij}$, $\bmh{C}$, $\Delta \hlambda_i$, $\Delta \bmh{\Lambda}$ 
  & for representing change from $\bmb{A}$ to $\bmh{A}$ & \ref{sec:proof-exptrue} \\
  $\Delta \hat s_{ij}$, $\Delta \bmh{S}$
  & defined from $\Delta\bmh{W}$ and $\bmh{Y}$ & \ref{sec:proof-exptrue} \\
  $\bmb{g} $, $\bmb{h} $ 
  & representing regularization with respect to $\epsilon \bmb{W}$  & \ref{sec:proof-exptrue} \\
  $\bmh{G}^*$, $\bmh{H}^*$, $\bmh{A}^*$, $\bmh{\Lambda}^*$, $\bm{g}^*$, $\bm{h}^*$ & 
  $\bmh{G}$, $\bmh{H}$, $\bmh{A}$, $\bmh{\Lambda}$, $\bm{g}$, $\bm{h}$
   for training dataset in cross-validation  & \ref{sec:proof-expcv}\\
  $s_{ij}^*$,  $\bm{S}^*$, $\hat s_{ij}^*$,  $\bmh{S}^*$  &  for test dataset in cross-validation
  & \ref{sec:proof-expcv}\\
  $\Delta \hat w_{ij}^*$, $\Delta \bmh{W}^* $, $\Delta \bmh{M}^*$ 
  & $\Delta\bmh{W}^* = \bm{W}^* - \kappa \bm{W}$,
  $\Delta\bmh{M}^* = \bm{M}^* - \kappa \bm{M}$ & \ref{sec:proof-expcv} \\
  $\Delta\bmh{G}^* $, $\Delta\bmh{H}^* $, $\bmh{g}^* $, $\bmh{h}^* $ 
  & matrices representing  the change $\Delta \bmh{W}^*$ & \ref{sec:proof-expcv} \\
  $\hat c_{ij}^*$, $\bmh{C}^*$, $\Delta \hlambda_i^*$, $\Delta \bmh{\Lambda}^*$ 
  & for representing change from $\bmh{A}$ to $\bmh{A}^*$ & \ref{sec:proof-expcv} 
   & (\ref{eq:hatstarcij})\\
  \hline
  \end{tabular}
\end{table}

\begin{table}
\caption{Parameters of data generation for experiments}
\label{tab:simulation-parameters}
\centering
  \begin{tabular}{ccccccccc}
    \hline
    Exp.  & & Sampling & $\bmb{W} $  & $\epsilon$, $\xi^2$ \\
    \hline
    1  &  & link & $\bmb{W}_A$ &  0.02 \\
    2  &  & link & $\bmb{W}_A$ &  0.04 \\
    3  &  & link & $\bmb{W}_A$ &  0.08 \\
    4  &  & link & $\bmb{W}_B$ &  0.02 \\
    5  &  & link & $\bmb{W}_B$ &  0.04 \\
    6  &  & link & $\bmb{W}_B$ &  0.08 \\
    7  &  & node & $\bmb{W}_A$ &  0.02 \\
    8  &  & node & $\bmb{W}_A$ &  0.04 \\
    9  &  & node & $\bmb{W}_A$ &  0.08 \\
    10  &  & node & $\bmb{W}_B$ &  0.02 \\
    11  &  & node & $\bmb{W}_B$ &  0.04 \\
    12  &  & node & $\bmb{W}_B$ &  0.08 \\
    \hline
  \end{tabular}
\end{table}

%%%%%%%%%%%%%%%%%%%%%%%%%%%%%%%%%%%%%%%%%%%%%%%%%%%%%%%%%%%%%%%%%%%%%%%%%%%%%

\clearpage

\begin{figure}
 \begin{center}
  \myfigbig{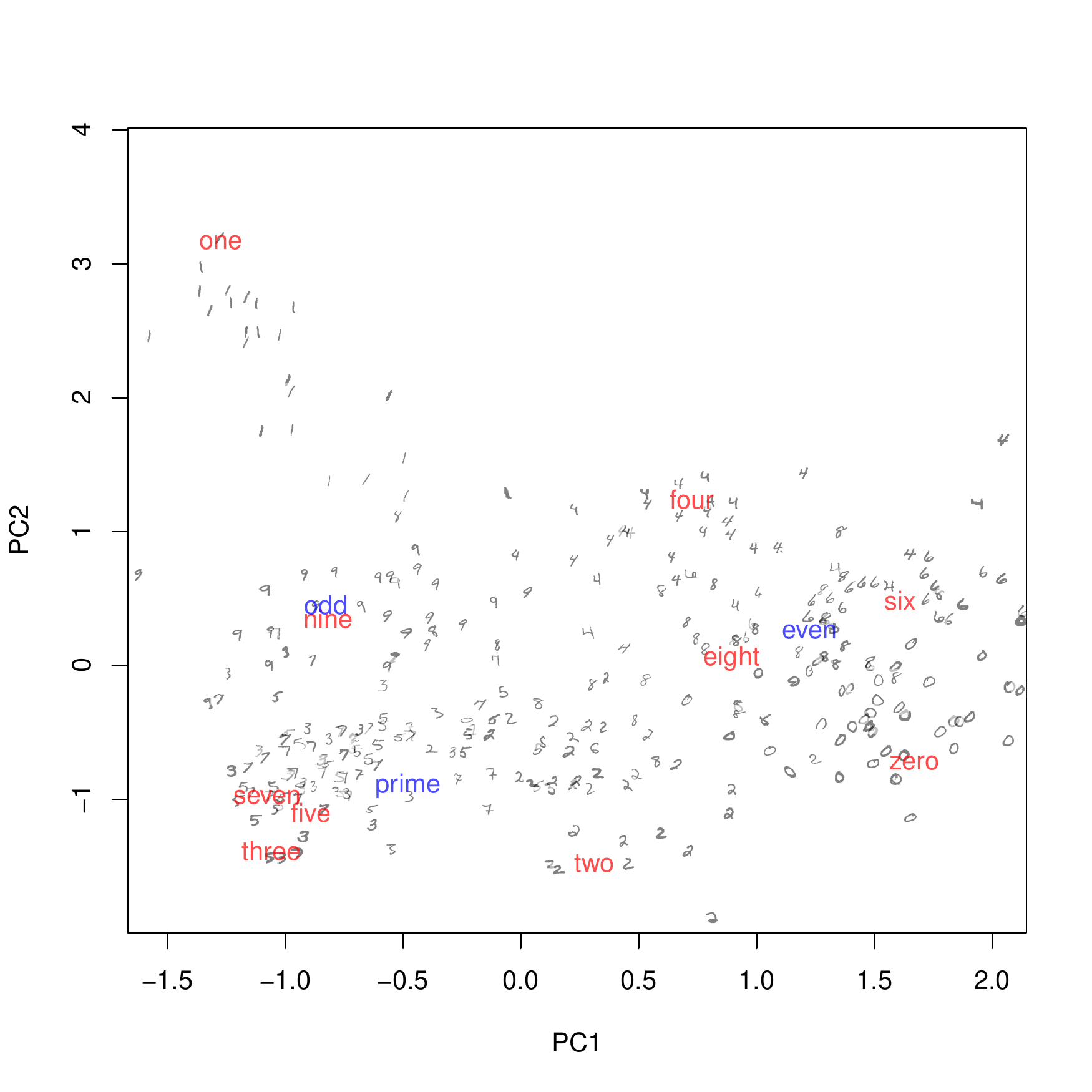}
 \end{center}
 \caption{Scatter plot of the first two components of CDMCA for the MNIST handwritten digits.
Shown are 300 digit images randomly sampled from the 60000 training images ($d=1$). Digit labels
($d=2$) and attribute labels ($d=3$) are also shown in the same ``common space'' of $K=2$.}
 \label{fig:mnist-map12}
\end{figure}  

\begin{figure}
 \begin{center}
 %  \subfigure[classification errors]{
 % \myfig{testmnist5a-plot1-cck-gamma.pdf}
 %  \label{fig:cck-gamma}
 %  }
 %  \subfigure[cumulative errors ($K=9$) ]{
 %  \myfig{testmnist5a-plot1-fitkcvk-gamma.pdf}
 %   \label{fig:fitkcvk-gamma}
 %  }

  \subfigure[classification errors ($K=9$) ]{
  \myfig{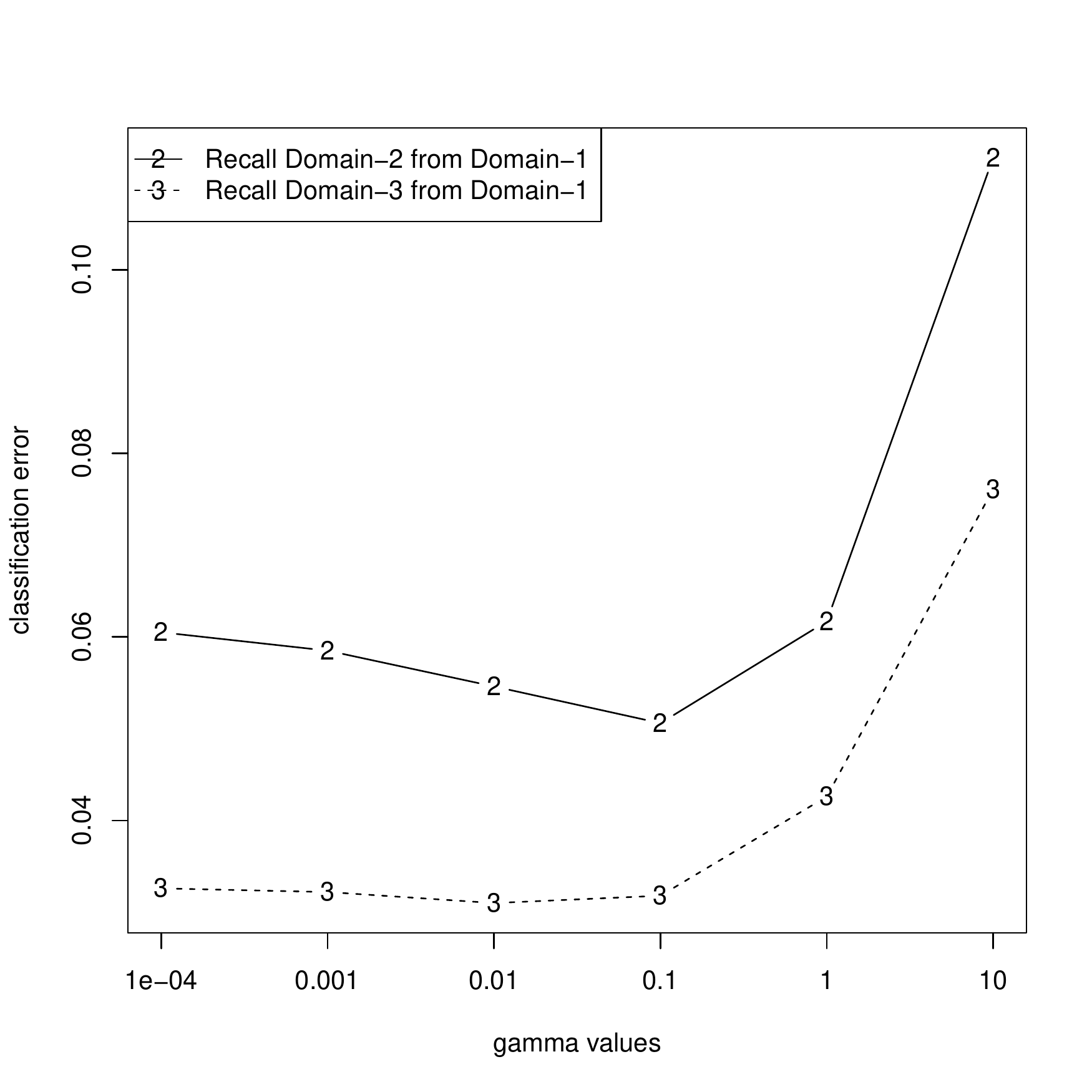}
   \label{fig:mnist-classification-error}
  }
  \subfigure[matching errors ($K=9$) ]{
  \myfig{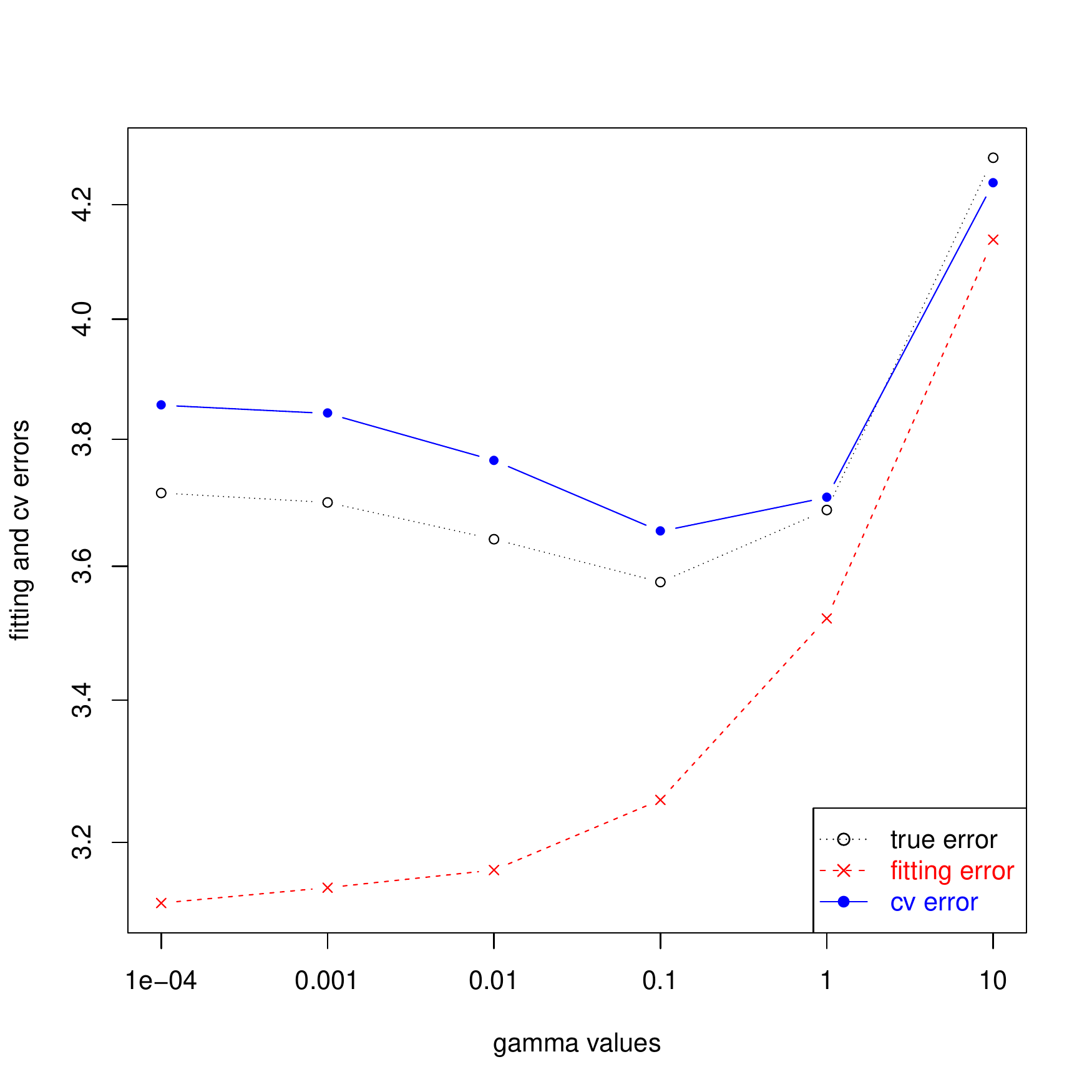}
   \label{fig:mnist-matching-error}
  }
 \end{center}
 \caption{CDMCA results for $\gamma_M=10^{-4}, 10^{-3}, 10^{-2}, 10^{-1}, 10^0, 10^{1}$.
 (a)~Classification error of predicting digit labels ($d=2$) from the 10000 test images,
 and that for predicting attribute labels ($d=3$). (b)~The true matching error of 10000 test
 images, and the fitting error and the cross-validation error computed from the 60000 training
 images.} \label{fig:mnist-error}
\end{figure}

%%%%%%%%%%%%%%%%%%%%%%%%%%%%%%%%%%%%%%%%%%%%%%%%%%%%%%%%%%%%%%%%%%%%%%%%%%%%%

% \begin{verbatim}
% (1) 9sim: prob0.edge = 0.02
% (2) 1sim: prob0.edge = 0.04
% (3) 5sim: prob0.edge = 0.08

% (4) 10sim: prob0.edge = 0.02, alpha0.pareto = 2.0
% (5) 2sim: prob0.edge = 0.04, alpha0.pareto = 2.0
% (6) 6sim: prob0.edge = 0.08, alpha0.pareto = 2.0

% (7) 11sim: prob0.node = sqrt(0.02) (Q0=19)
% (8) 3sim: prob0.node = 0.2
% (9) 7sim: prob0.node = sqrt(0.08)

% (10) 12sim: prob0.node = sqrt(0.02), alpha0.pareto = 2.0 (Q0=19)
% (11) 4sim: prob0.node = 0.2, alpha0.pareto = 2.0 (Q0=19)
% (12) 8sim: prob0.node = sqrt(0.08), alpha0.pareto = 2.0 (Q0=19)
% \end{verbatim}

\begin{figure}
 \begin{center}
  \subfigure[$\bmb{W}_A$]{
  \myfigc{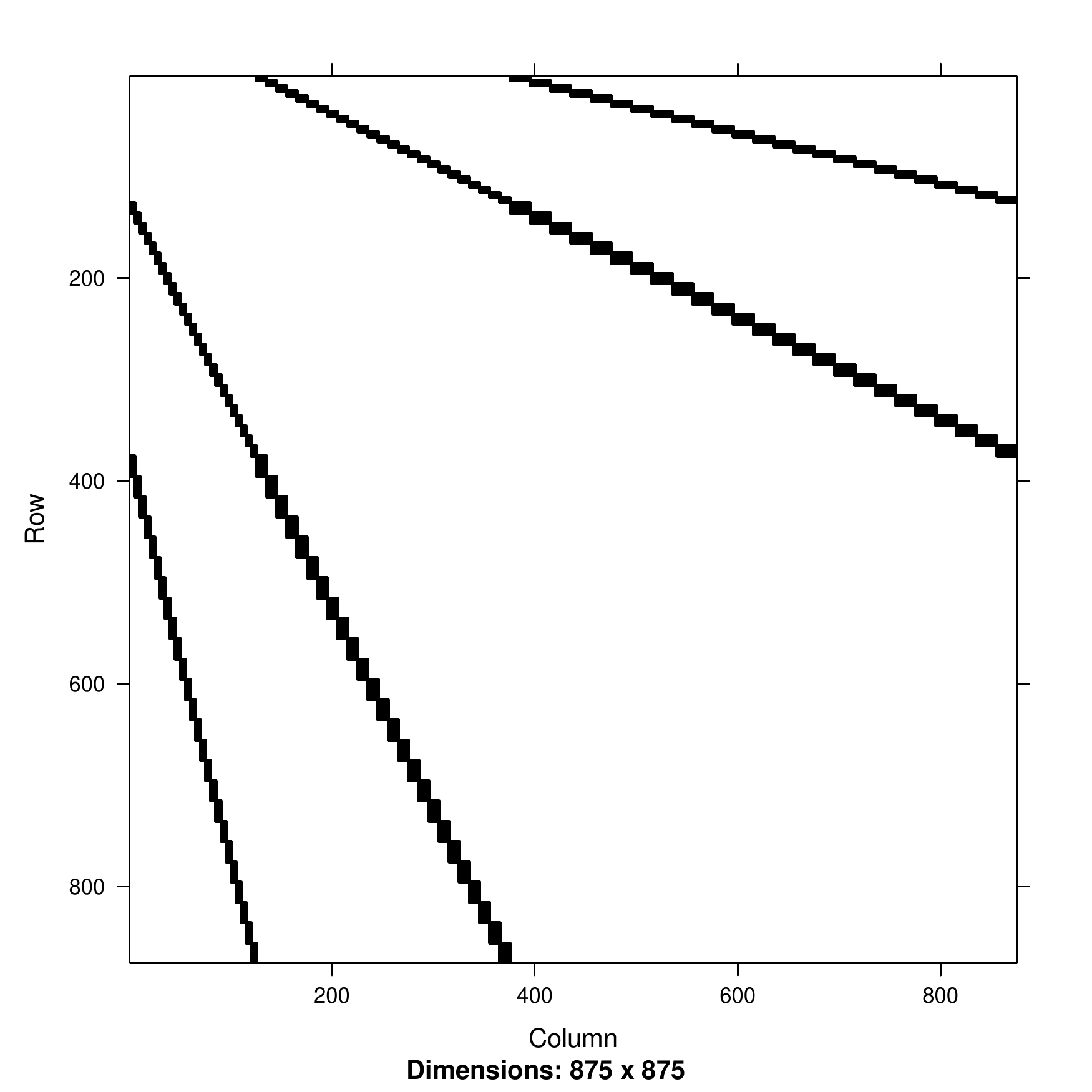}
  \label{fig:wa}
  }
  \subfigure[$\bm{W}$ of Experiment 1]{
  \myfigc{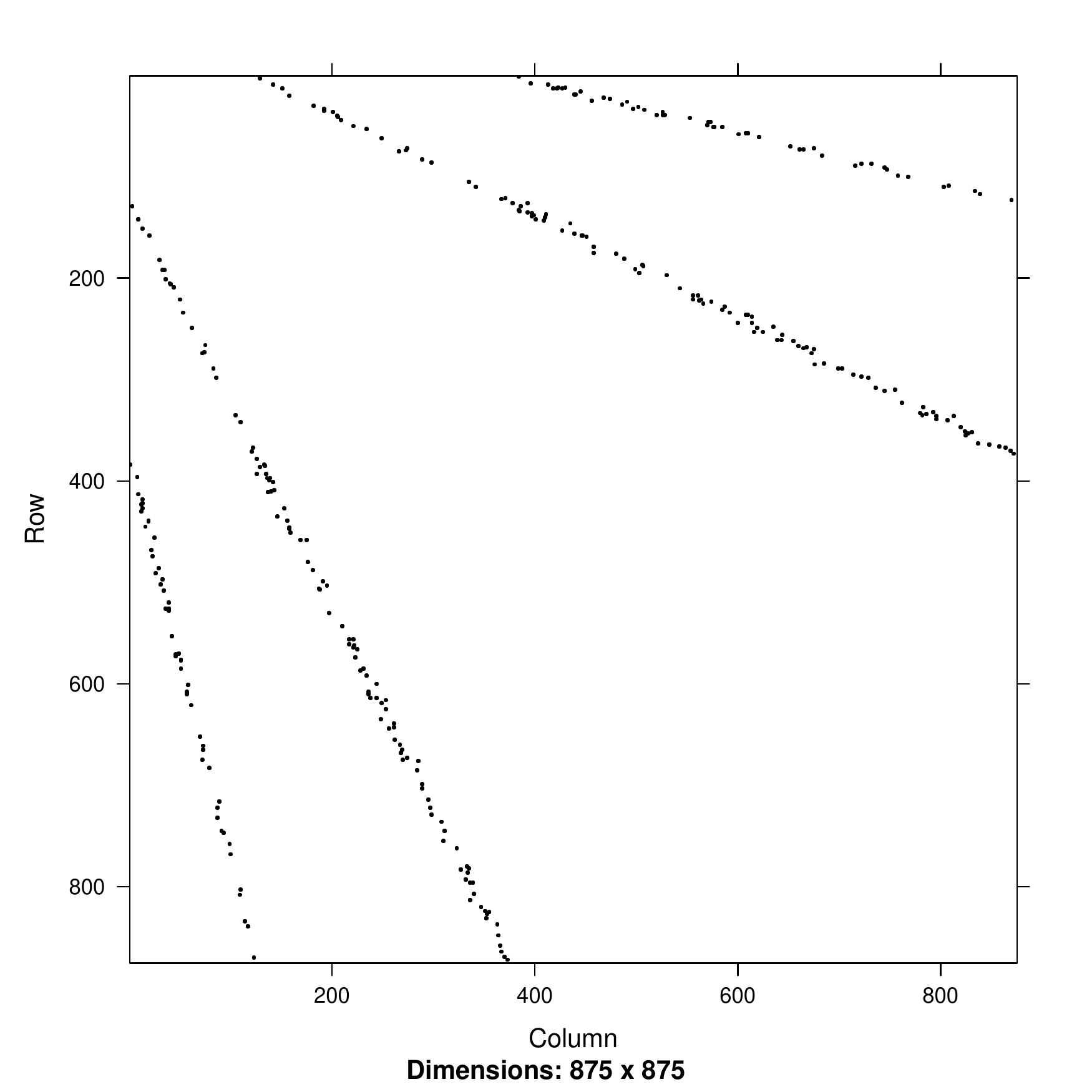}
  \label{fig:w1}
  }
  \subfigure[$\bm{W}$ of Experiment 9]{
  \myfigc{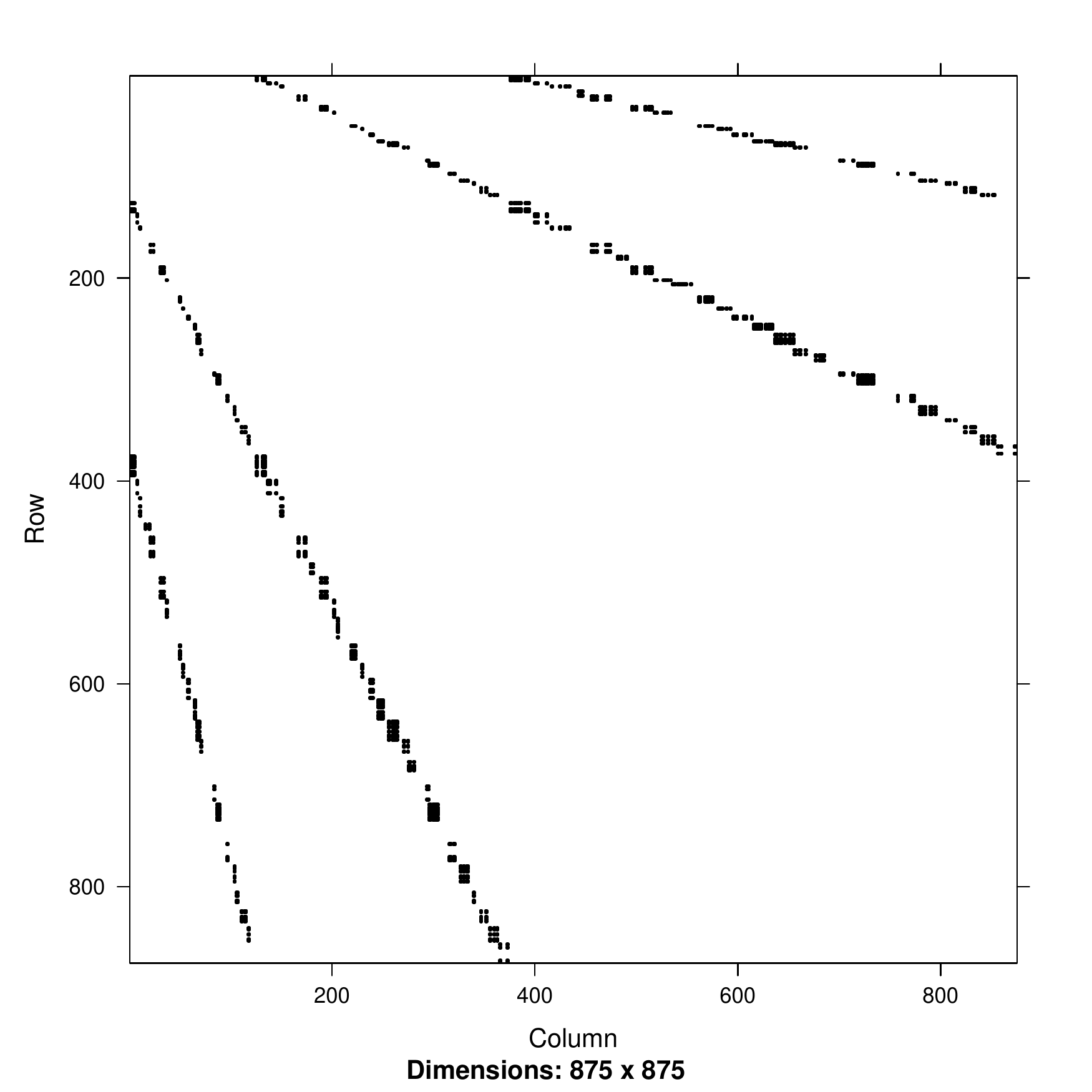}
  \label{fig:w9}
  }
 \end{center}
 \caption{Regular matching weights. (a)~A true matching weight of a regular structure.
  (b)~Link sampling from $\bmb{W}_A$. (c)~Node sampling from $\bmb{W}_A$.} \label{fig:weight-a}
\end{figure}  

\begin{figure}
 \begin{center}
  \subfigure[$\bmb{W}_B$]{
  \myfigc{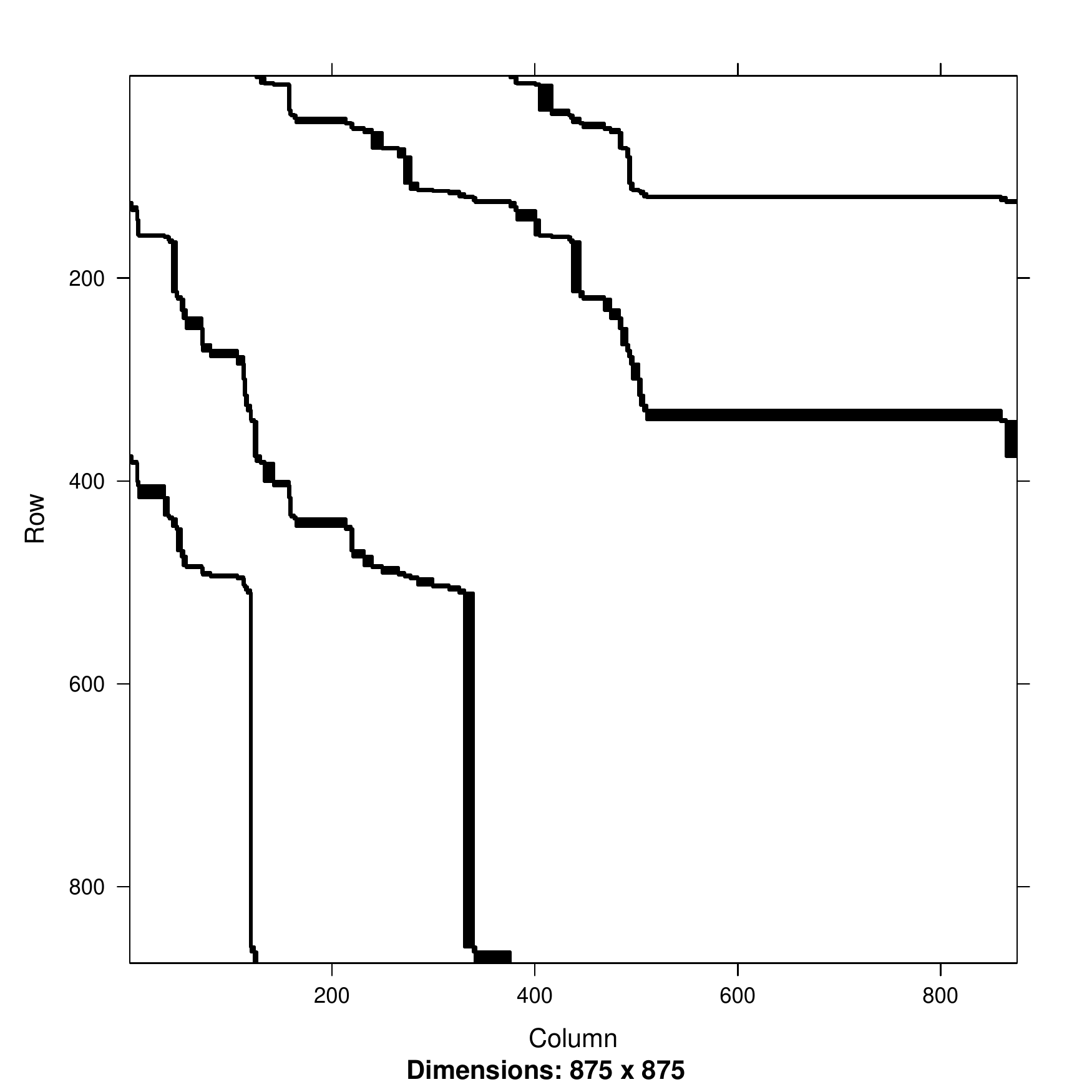}
  \label{fig:wb}
  }
  \subfigure[$\bm{W}$ of Experiment 5]{
  \myfigc{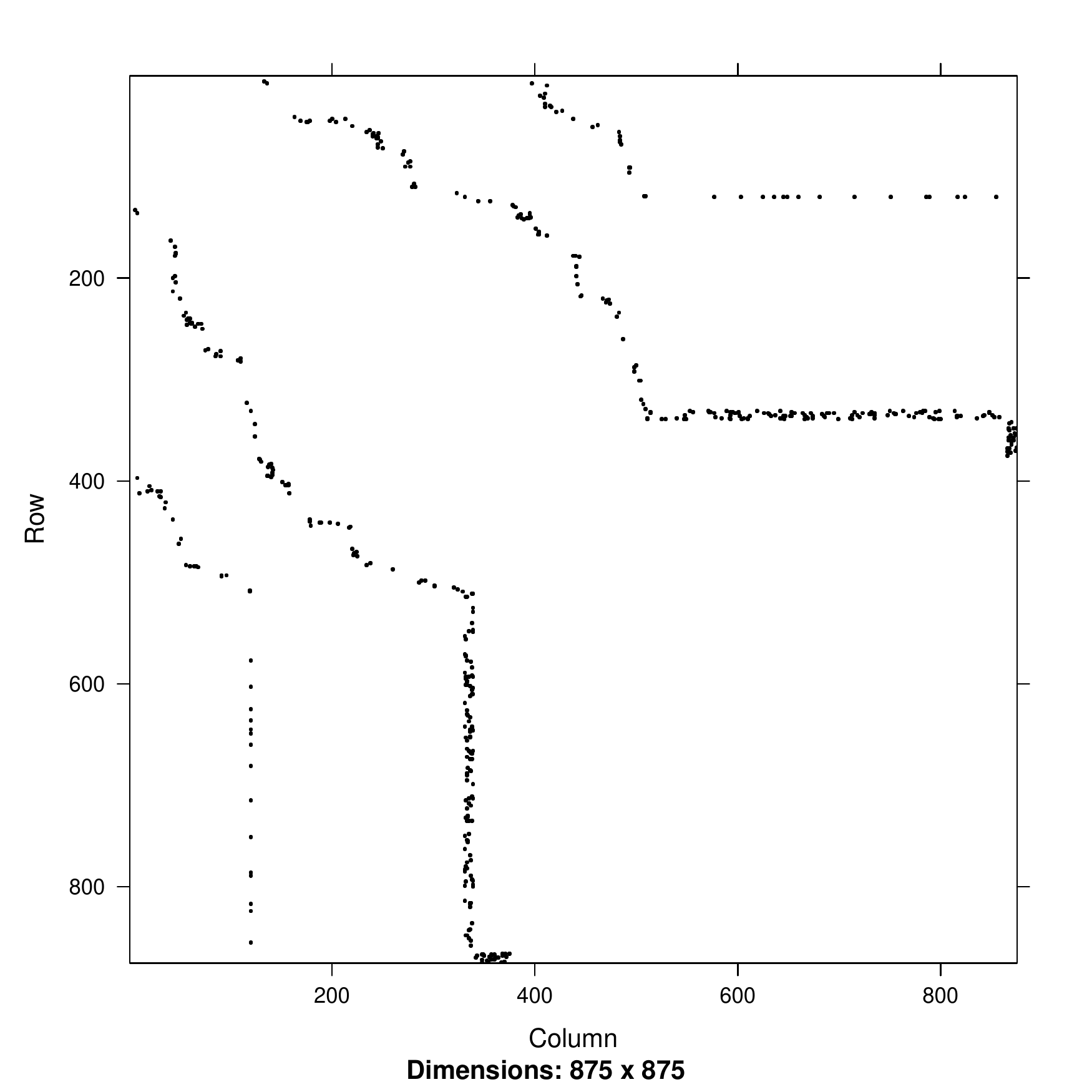}
  \label{fig:w5}
  }
  \subfigure[$\bm{W}$ of Experiment 11]{
  \myfigc{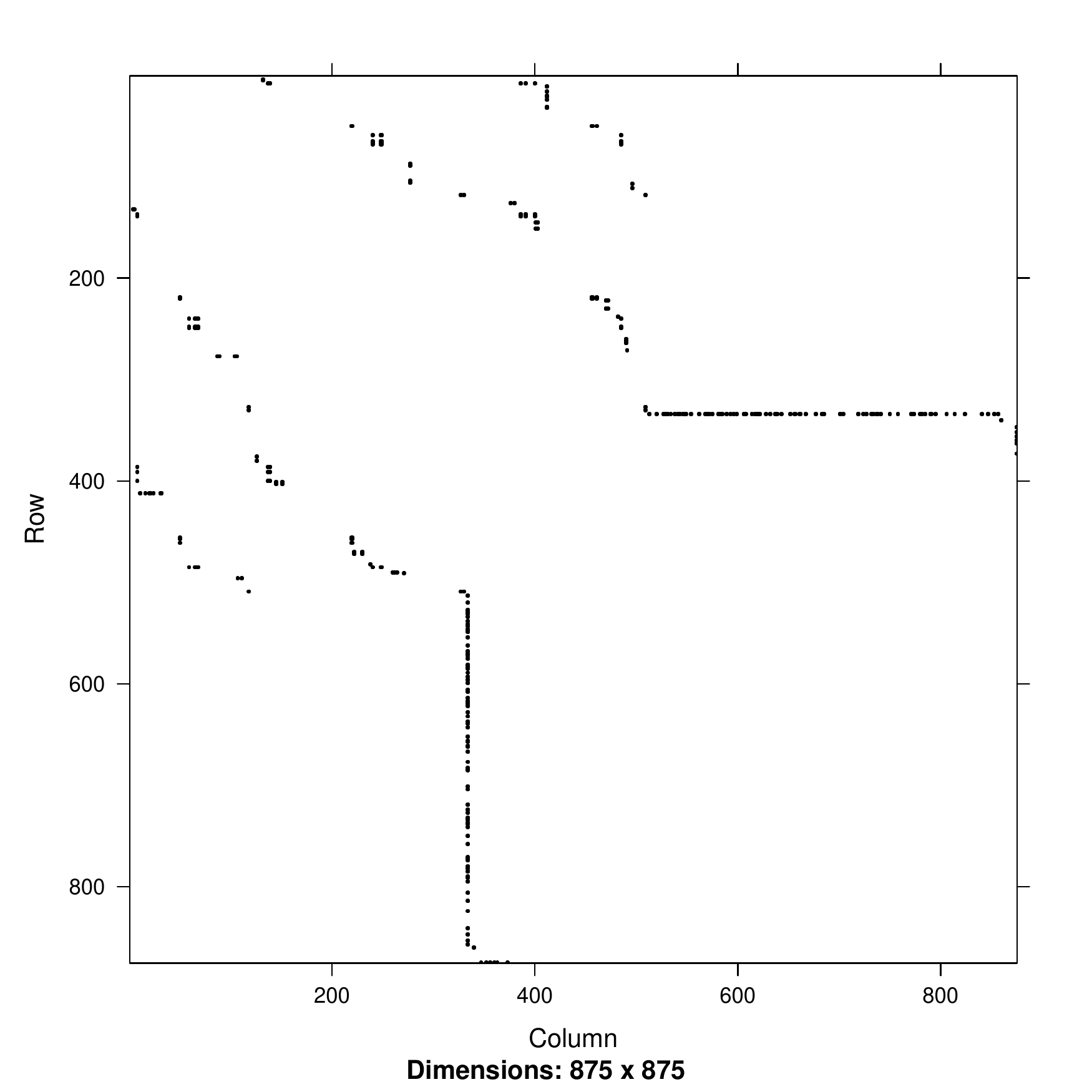}
  \label{fig:w11}
  }
 \end{center}
 \caption{Power-law matching weights.  (a)~A true matching weight of a power-law structure.
  (b)~Link sampling from $\bmb{W}_B$. (c)~Node sampling from $\bmb{W}_B$.} \label{fig:weight-b}
\end{figure}  

\begin{figure}
 \begin{center}
  \subfigure[$\gamma_M=0$]{
  \myfigc{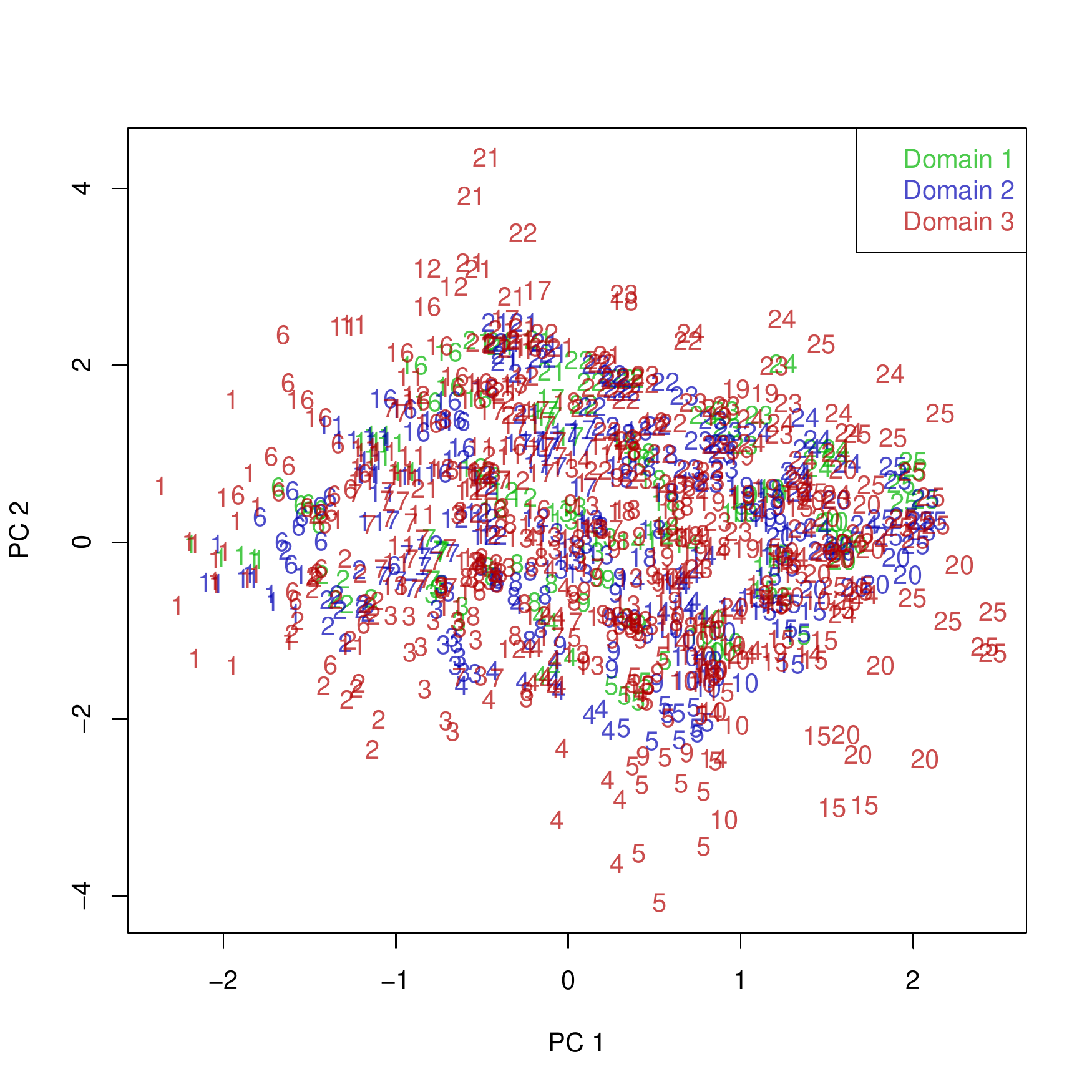}
  \label{fig:exp1-00pc}
  }
  \subfigure[$\gamma_M=0.1$]{
  \myfigc{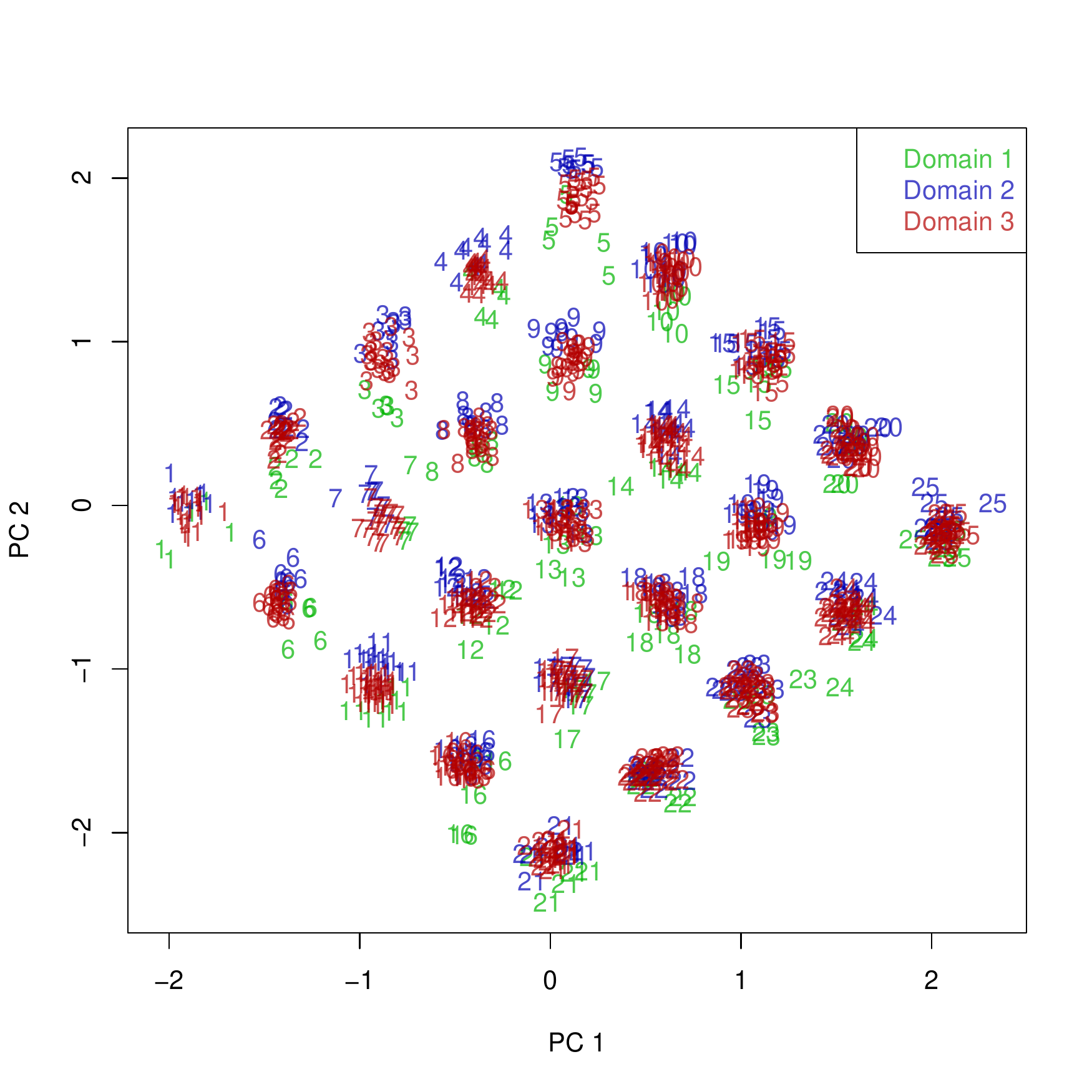}
  \label{fig:exp1-01pc}
  }
  \subfigure[matching errors (weighted)]{
  \myfigc{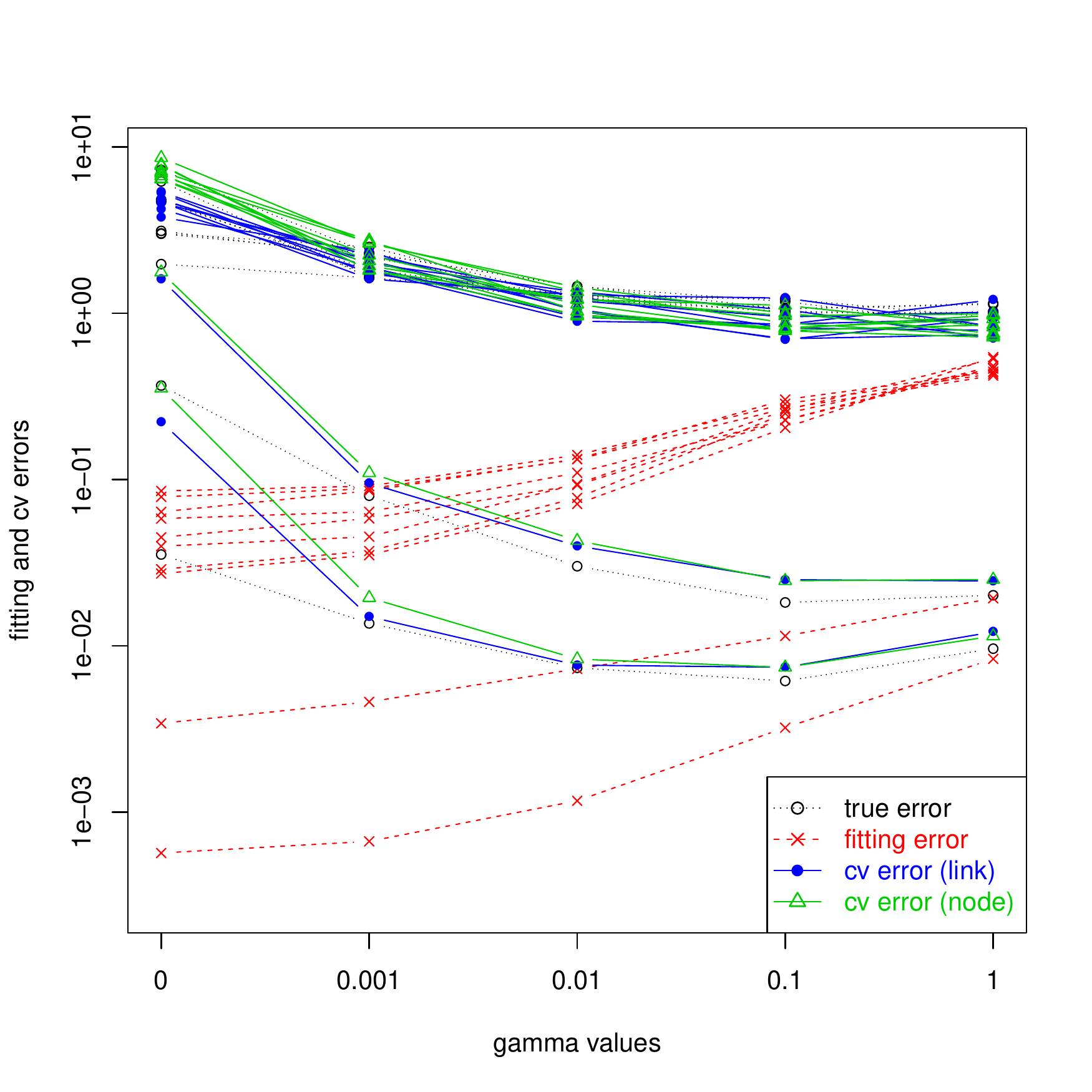}
  \label{fig:exp1-fitcv}
  }
 \end{center}
 \caption{Experiment 1. (a) and (b)~Scatter plots of the first two components of CDMCA with $\gamma_M=0$
 and $\gamma_M=0.1$ respectively. (c)~True matching error of $\bmb{W}_A$, and the other matching errors
 computed from $\bm{W}$.} \label{fig:exp1}
\end{figure}  

% \begin{verbatim}
% 9sim
% > (a=nonzerow(W0)); sum(a) # nonzero elements
% [1]    0 1250 2500    0    0 5000    0    0    0
% [1] 8750
% > (a=nonzerow(W1)); sum(a) # nonzero elements
% [1]  0 23 51  0  0 88  0  0  0
% [1] 162
% \end{verbatim}

\begin{figure}
 \begin{center}
  \subfigure[$\gamma_M=0$]{
  \myfigc{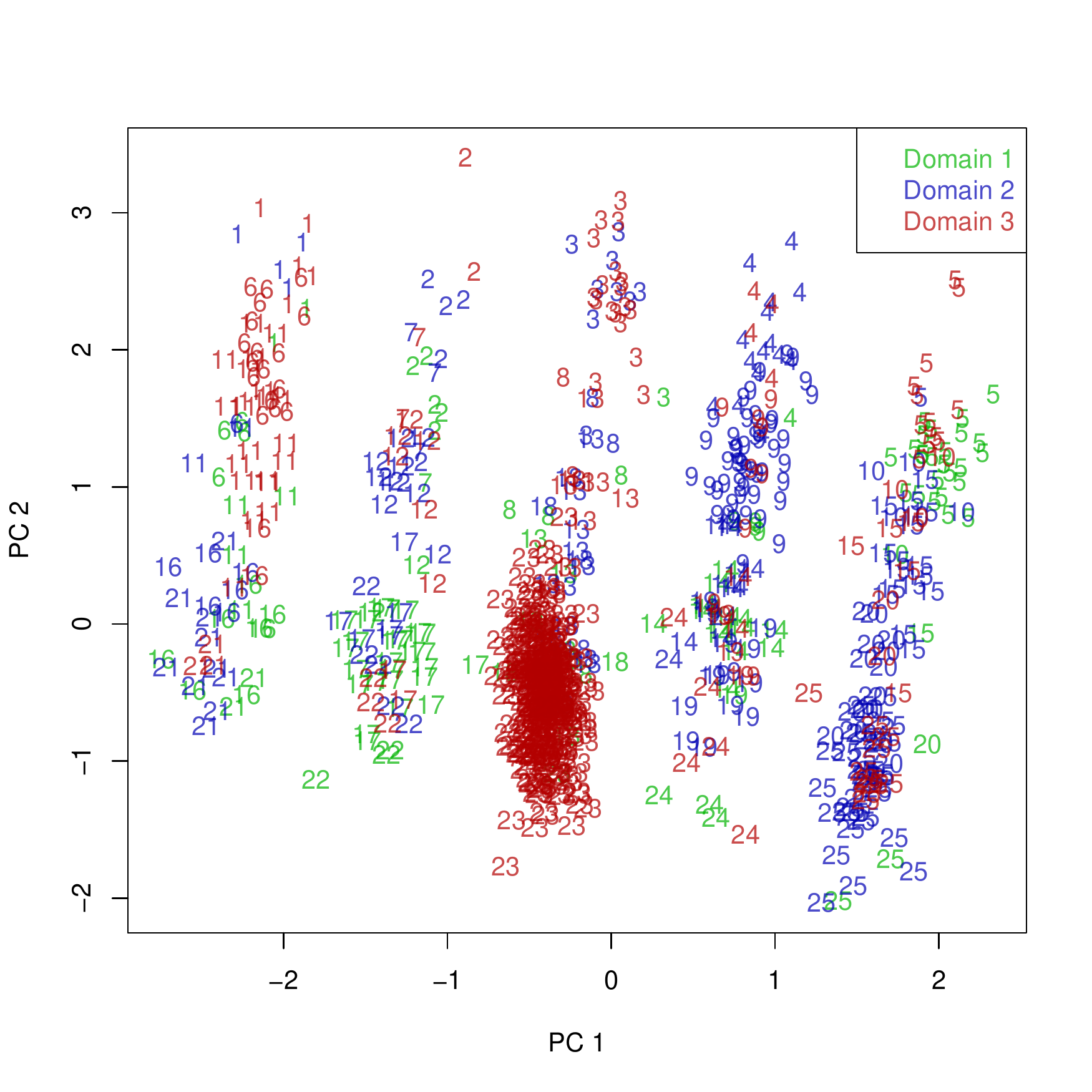}
  \label{fig:exp5-00pc}
  }
  \subfigure[$\gamma_M=0.1$]{
  \myfigc{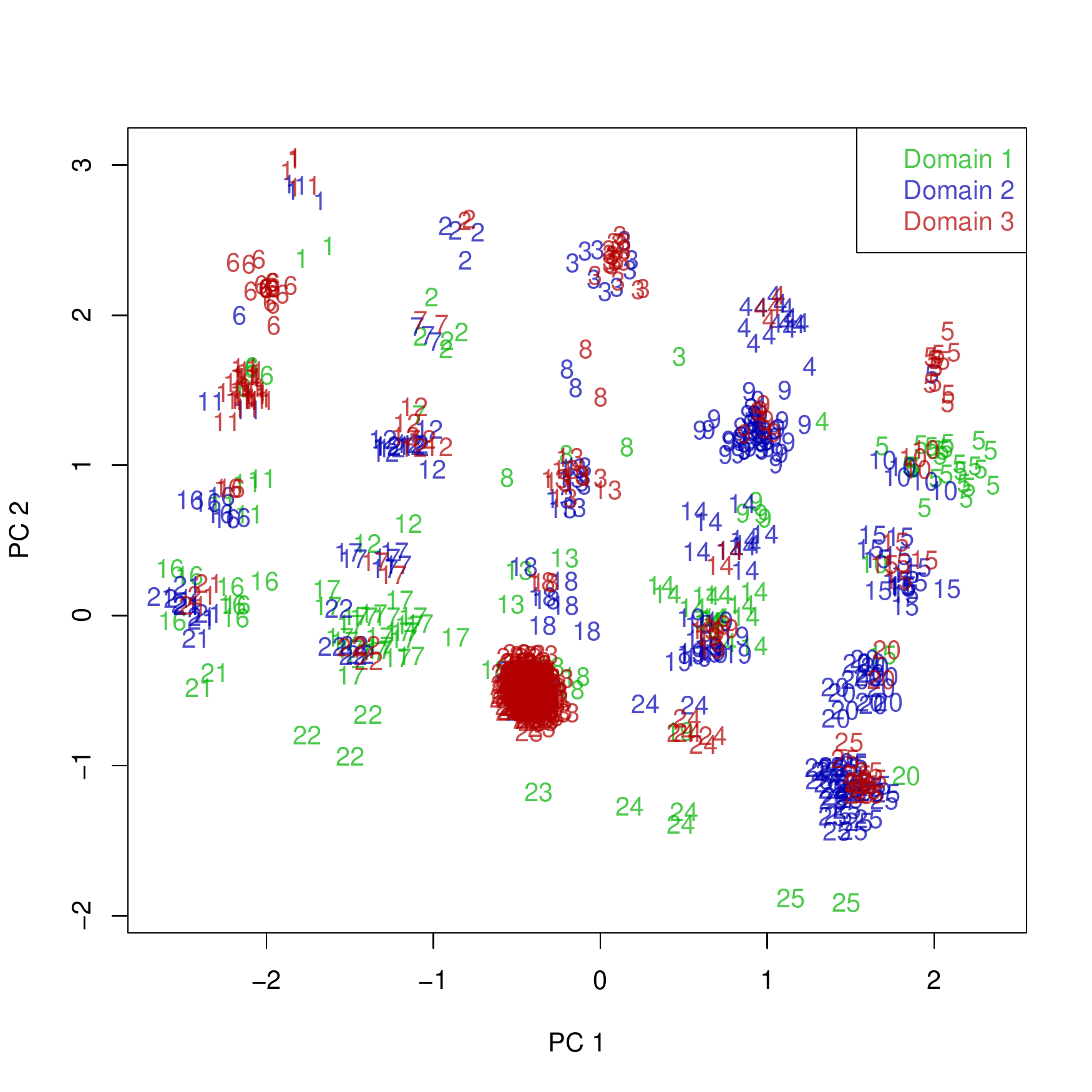}
  \label{fig:exp5-01pc}
  }
  \subfigure[matching errors (unweighted)]{
  \myfigc{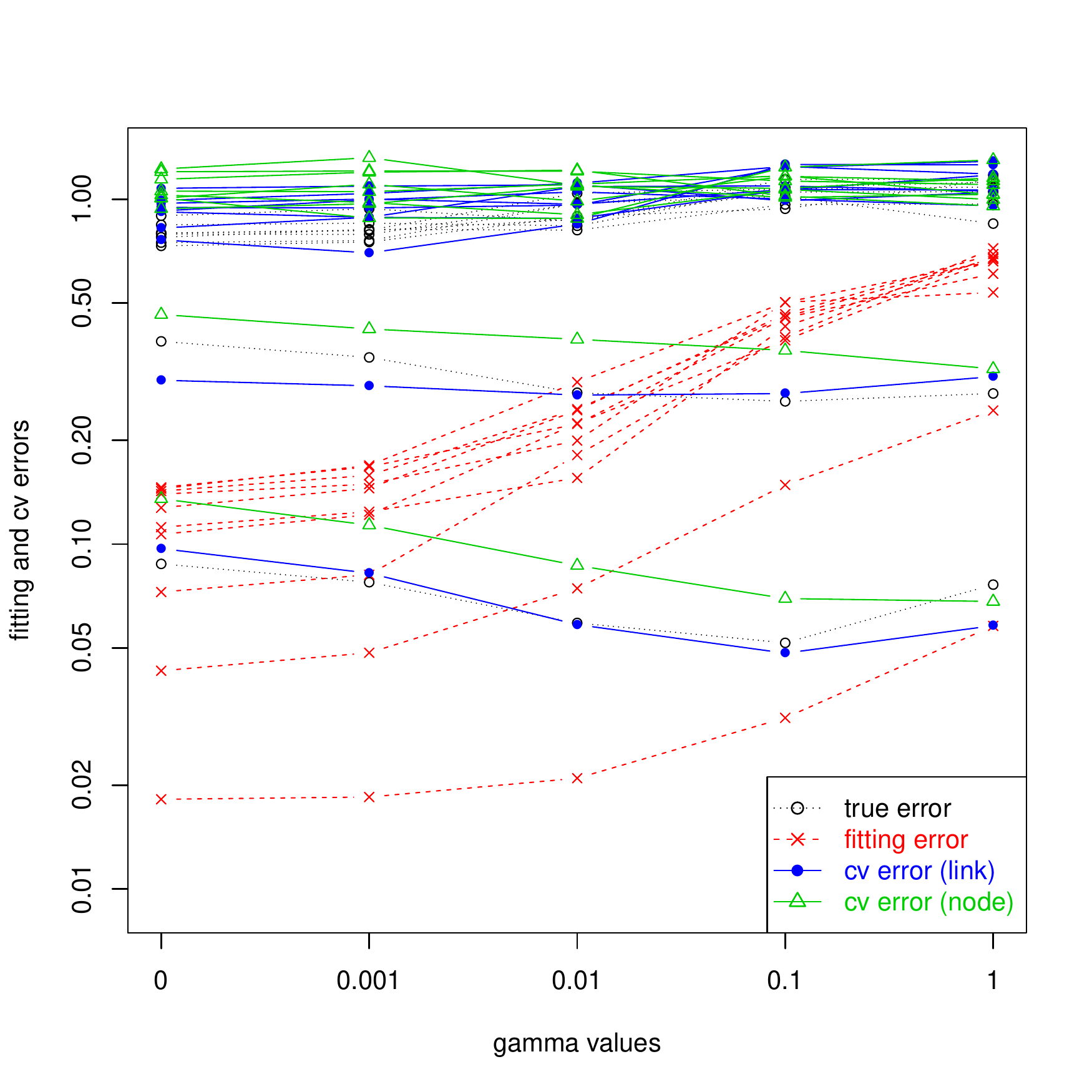}
  \label{fig:exp5-fitcv}
  }
 \end{center}
 \caption{Experiment 5. (a) and (b)~Scatter plots of the first two components of CDMCA with $\gamma_M=0$
 and $\gamma_M=0.1$ respectively. (c)~True matching error of $\bmb{W}_B$, and the other matching errors
 computed from $\bm{W}$.} \label{fig:exp5}
\end{figure}  

% \begin{verbatim}
% 2sim
% > (a=nonzerow(W0)); sum(a) # nonzero elements
% [1]    0  906 1096    0    0 4657    0    0    0
% [1] 6659
% > (a=nonzerow(W1)); sum(a) # nonzero elements
% [1]   0  39  39   0   0 180   0   0   0
% [1] 258
% \end{verbatim}

\begin{figure}
 \begin{center}
  \subfigure[fitting error]{
  \myfigc{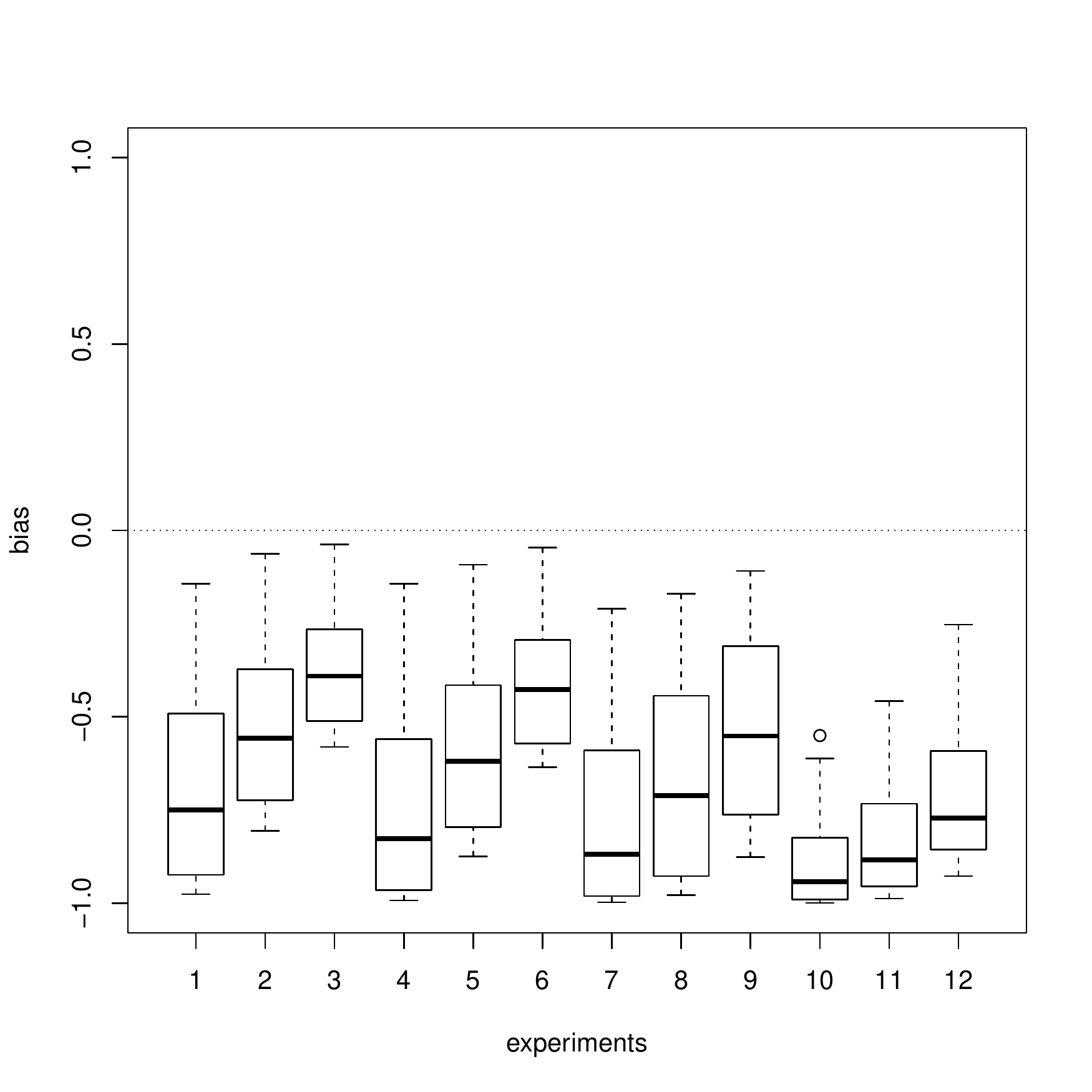}
  \label{fig:bias-fit-w}
  }
  \subfigure[cv error (link)]{
  \myfigc{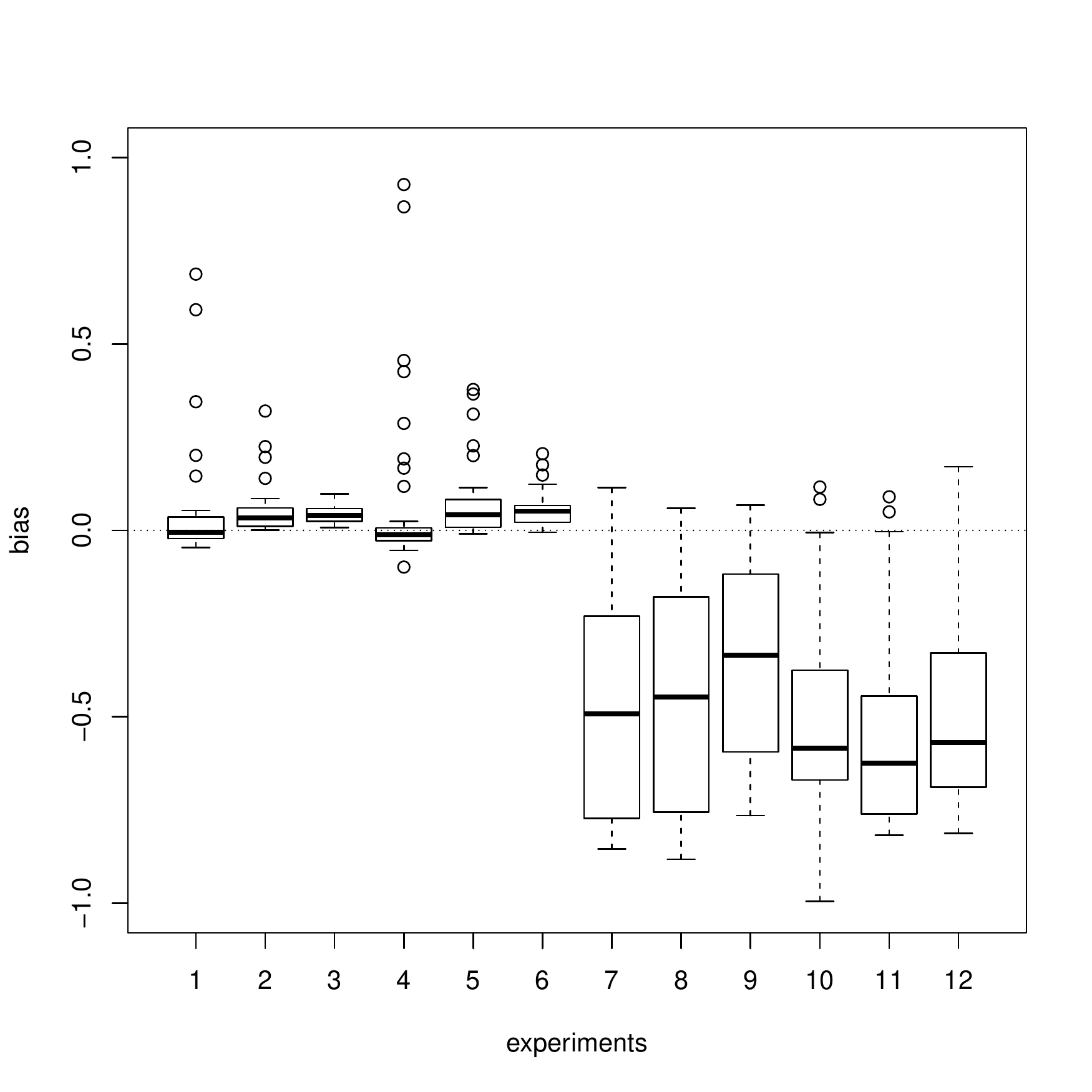}
  \label{fig:bias-cvlink-w}
  }
  \subfigure[cv error (node)]{
  \myfigc{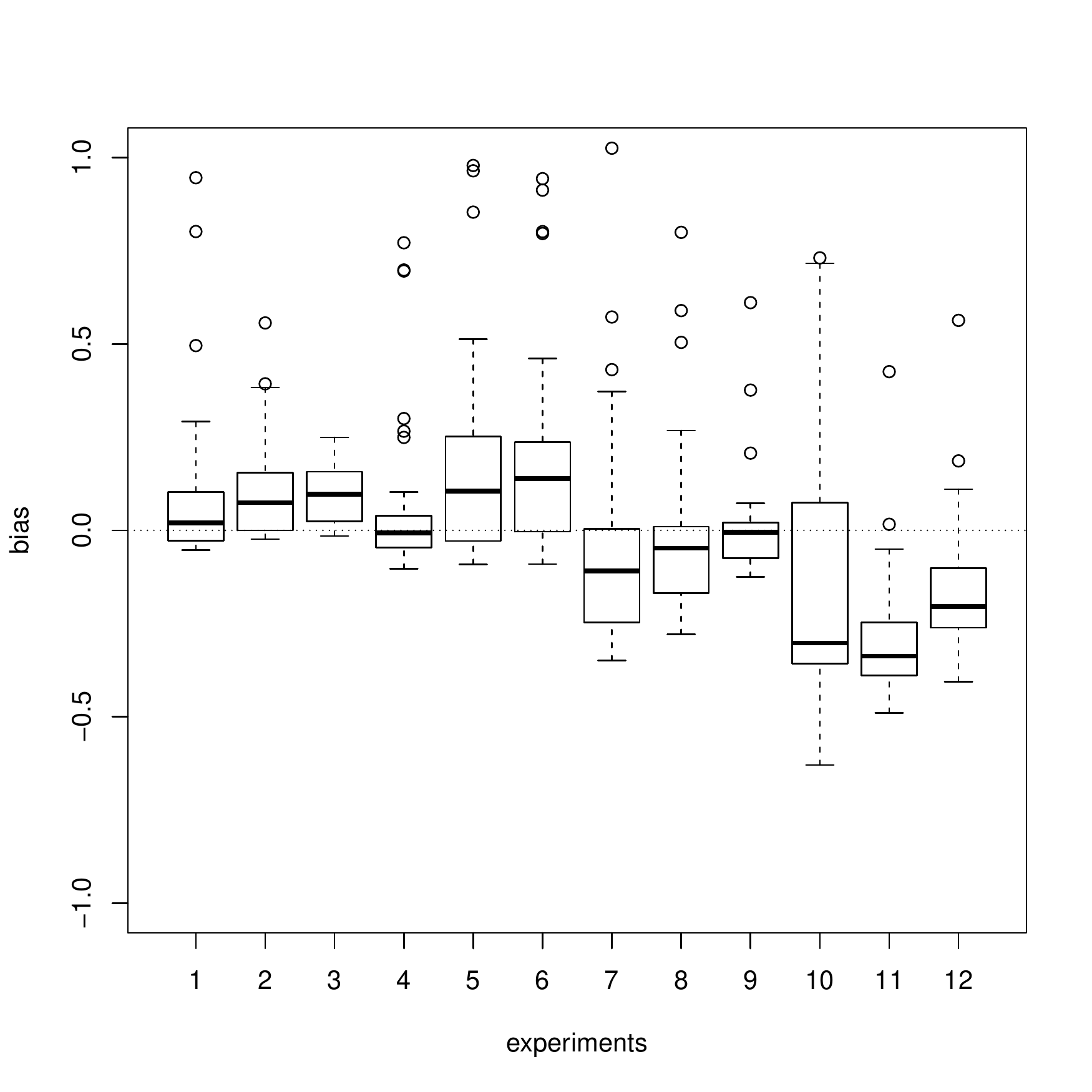}
  \label{fig:bias-cvnode-w}
  }
 \end{center}
 \caption{Bias of the matching errors of the components rescaled by the weighted variance
 (\ref{eq:ymy1}).} \label{fig:boxplots-w}
\end{figure}  

\begin{figure}
 \begin{center}
  \subfigure[fitting error]{
  \myfigc{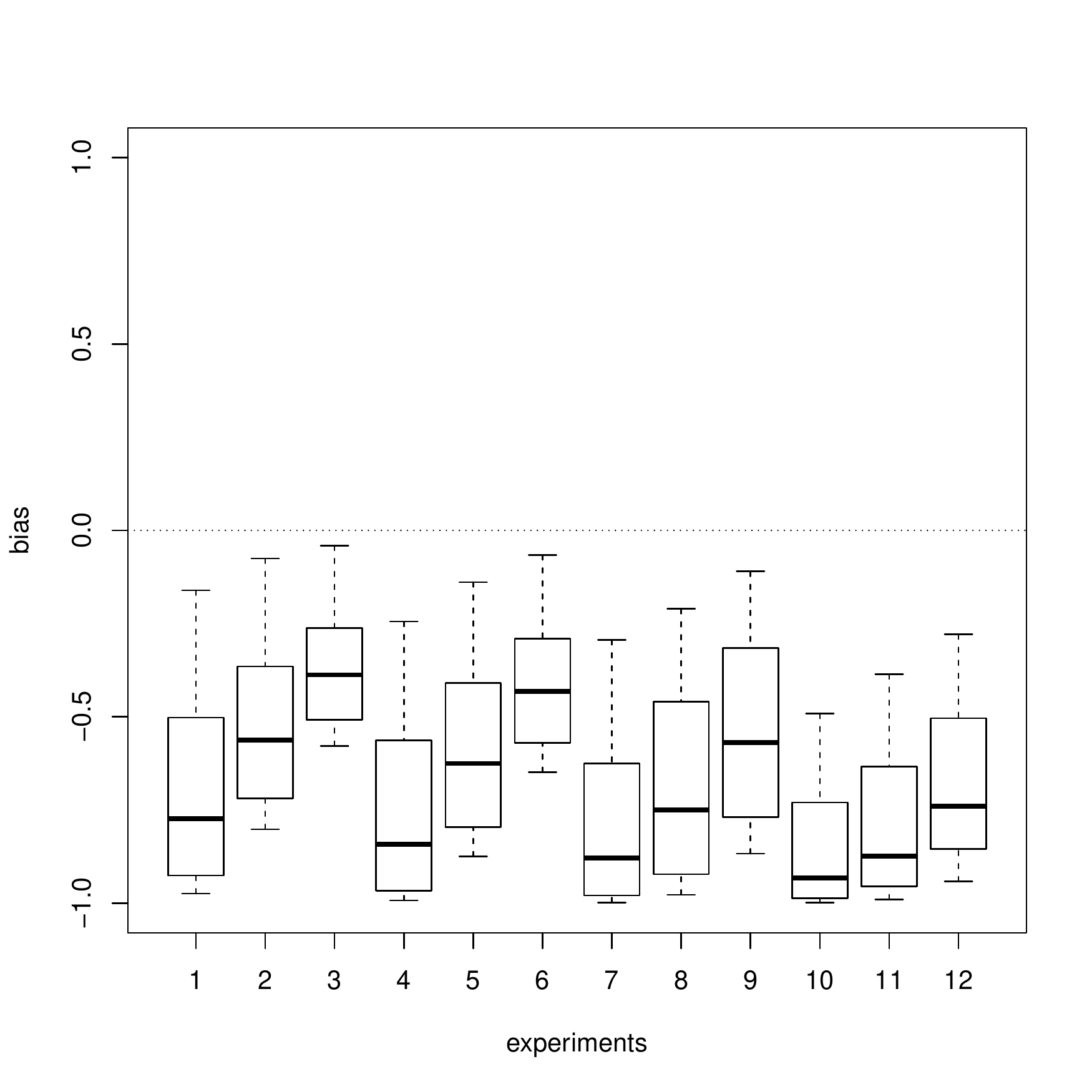}
  \label{fig:bias-fit-u}
  }
  \subfigure[cv error (link)]{
  \myfigc{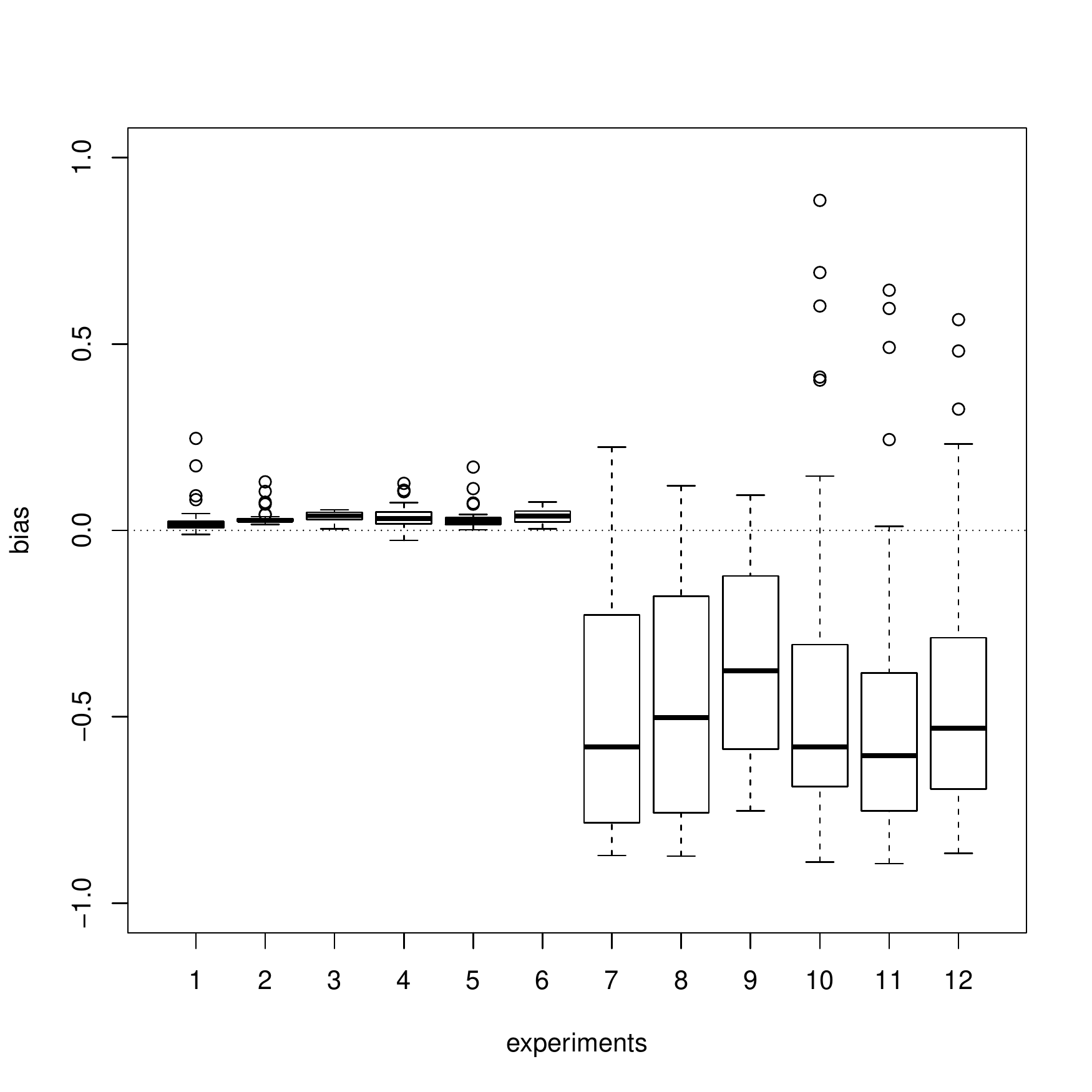}
  \label{fig:bias-cvlink-u}
  }
  \subfigure[cv error (node)]{
  \myfigc{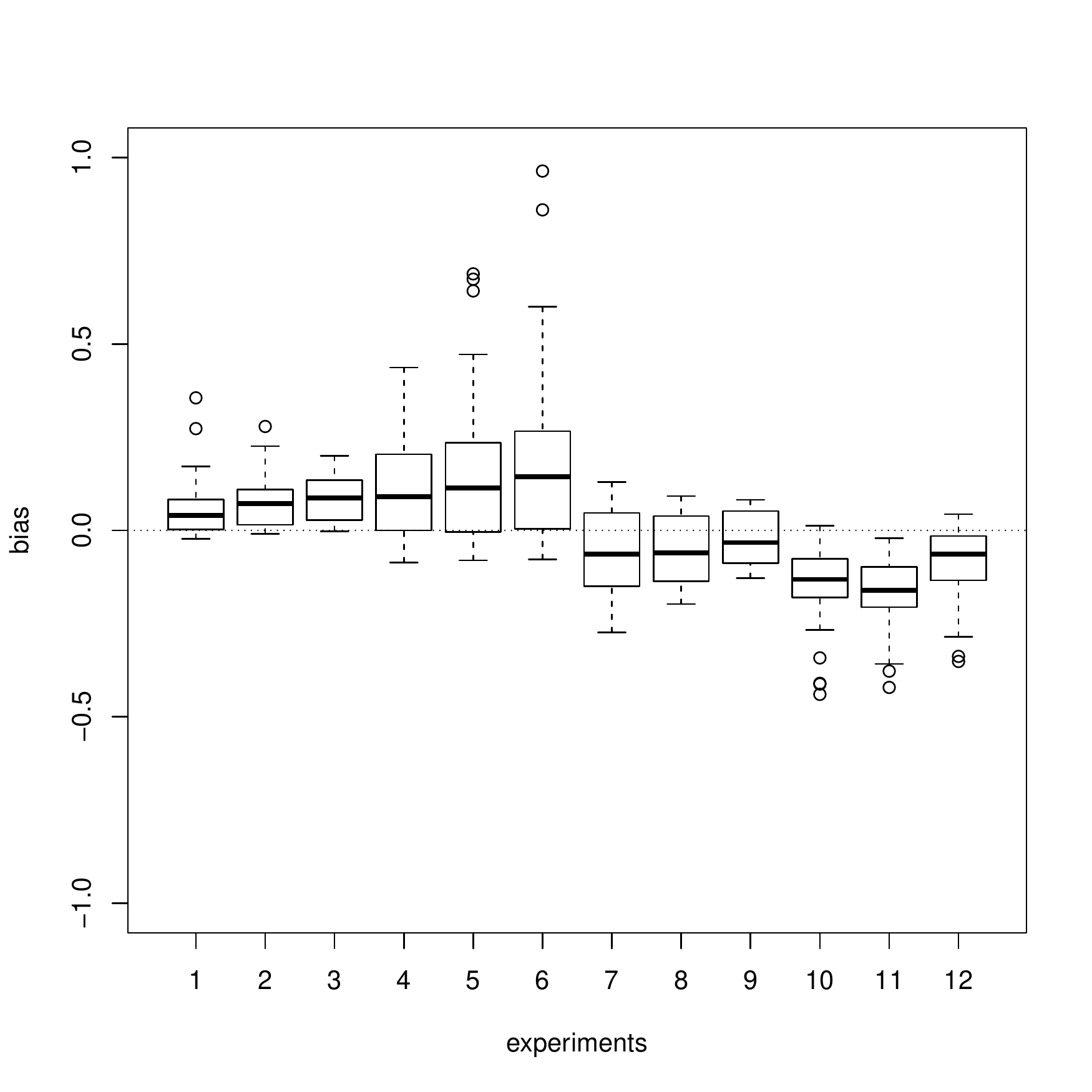}
  \label{fig:bias-cvnode-u}
  }
 \end{center}
 \caption{Bias of the matching errors of the components rescaled by the unweighted variance
 (\ref{eq:yy1}).} \label{fig:boxplots-u}
\end{figure}

\end{document}